\documentclass[nohyperref]{article}

\usepackage{microtype}
\usepackage{graphicx}
\usepackage{subfigure}
\usepackage{booktabs} 

\usepackage{appendix}
\usepackage{amsmath}
\usepackage{amssymb}
\usepackage{bbm}
\usepackage{bm}
\usepackage{amsthm}
\usepackage{grffile}
\usepackage{titletoc}
\newtheorem{defi}{Definition}

\newtheorem{prop}{Proposition}

\newtheorem{lem}{Lemma}
\newtheorem{rem}{Remark}

\newtheorem{thm}{Theorem}
\newtheorem{asm}{Assumption}

\usepackage{multirow}
\usepackage{balance}
\usepackage{float} 

\usepackage[dvipsnames]{xcolor}
\definecolor{ballblue}{HTML}{338EA7}
\definecolor{darkblue}{HTML}{183D5E}
\definecolor{lightseagreen}{HTML}{759D39}
\definecolor{lightred}{HTML}{DD7769}
\definecolor{org}{HTML}{F8A145}
\definecolor{blu}{HTML}{63ACE5}
\definecolor{c1}{HTML}{41B3A3}
\definecolor{c2}{HTML}{3500D3}

\usepackage[colorlinks,linkcolor=blue,citecolor=org,urlcolor=org, linktoc=all]{hyperref}
\def \op{{\color{lightseagreen} \textbf{(OP)}}}
\def \opz{{\color{lightseagreen} \textbf{(OP0)}}}

\usepackage[accepted]{icml2022}

\icmltitlerunning{AdAUC: End-to-end Adversarial AUC Optimization Against Long-tail Problems}

\begin{document}
	
	\twocolumn[
	\icmltitle{AdAUC: End-to-end Adversarial AUC Optimization Against Long-tail Problems}
	
	
	
	
\begin{icmlauthorlist}
\icmlauthor{Wenzheng Hou}{ict,cas}
\icmlauthor{Qianqian Xu}{ict}
\icmlauthor{Zhiyong Yang}{cas}
\icmlauthor{Shilong Bao}{iie,cass}
\icmlauthor{Yuan He}{ali}
\icmlauthor{Qingming Huang}{ict,cas,bdkm,pc}
\end{icmlauthorlist}

\icmlaffiliation{ict}{Key Laboratory of Intelligent Information Processing, Institute of Computing Technology, CAS, Beijing, China.}
\icmlaffiliation{cas}{School of Computer Science and Technology, University of Chinese Academy of Sciences, Beijing, China.}
\icmlaffiliation{iie}{State Key Laboratory of Information Security, Institute of Information
Engineering, CAS, Beijing, China.}
\icmlaffiliation{cass}{School of Cyber Security, University of Chinese Academy of Sciences, Beijing, China.}
\icmlaffiliation{ali}{Alibaba Group, Beijing, China}
\icmlaffiliation{bdkm}{Key Laboratory of Big Data Mining and Knowledge Management, Chinese Academy of Sciences, Beijing, China.}
\icmlaffiliation{pc}{Artificial Intelligence Research Center, Peng Cheng Laboratory, Shenzhen, China}

\icmlcorrespondingauthor{Qianqian Xu}{xuqianqian@ict.ac.cn}
\icmlcorrespondingauthor{Qingming Huang}{qmhuang@ucas.ac.cn}
	
	\icmlkeywords{Machine Learning, ICML}
	
	\vskip 0.3in
	]
	
	
	
    	\printAffiliationsAndNotice{}  
	
	\begin{abstract}
	It is well-known that deep learning models are vulnerable to adversarial examples.  Existing studies of adversarial training have made great progress against this challenge. As a typical trait, they often assume that the class distribution is overall balanced. However, long-tail datasets are ubiquitous in a wide spectrum of applications, where the amount of head class instances is larger than the tail classes. Under such a scenario, AUC is a much more reasonable metric than accuracy since it is insensitive toward class distribution. Motivated by this, we present an early trial to explore adversarial training methods to optimize AUC.
    The main challenge lies in that the positive and negative examples are tightly coupled in the objective function. As a direct result, one cannot generate adversarial examples without a full scan of the dataset. To address this issue, based on a concavity regularization scheme, we reformulate the AUC optimization problem as a saddle point problem, where the objective becomes an instance-wise function. This leads to an end-to-end training protocol. Furthermore, we provide a convergence guarantee of the proposed algorithm. Our analysis differs from the existing studies since the algorithm is asked to generate adversarial examples by calculating the gradient of a min-max problem. Finally, the extensive experimental results show the performance and robustness of our algorithm in three long-tail datasets.
	\end{abstract}
	
	\section{Introduction}
	\label{Introduction}
	
	\begin{figure}
	    \centering
	    \includegraphics[width=0.46\textwidth]{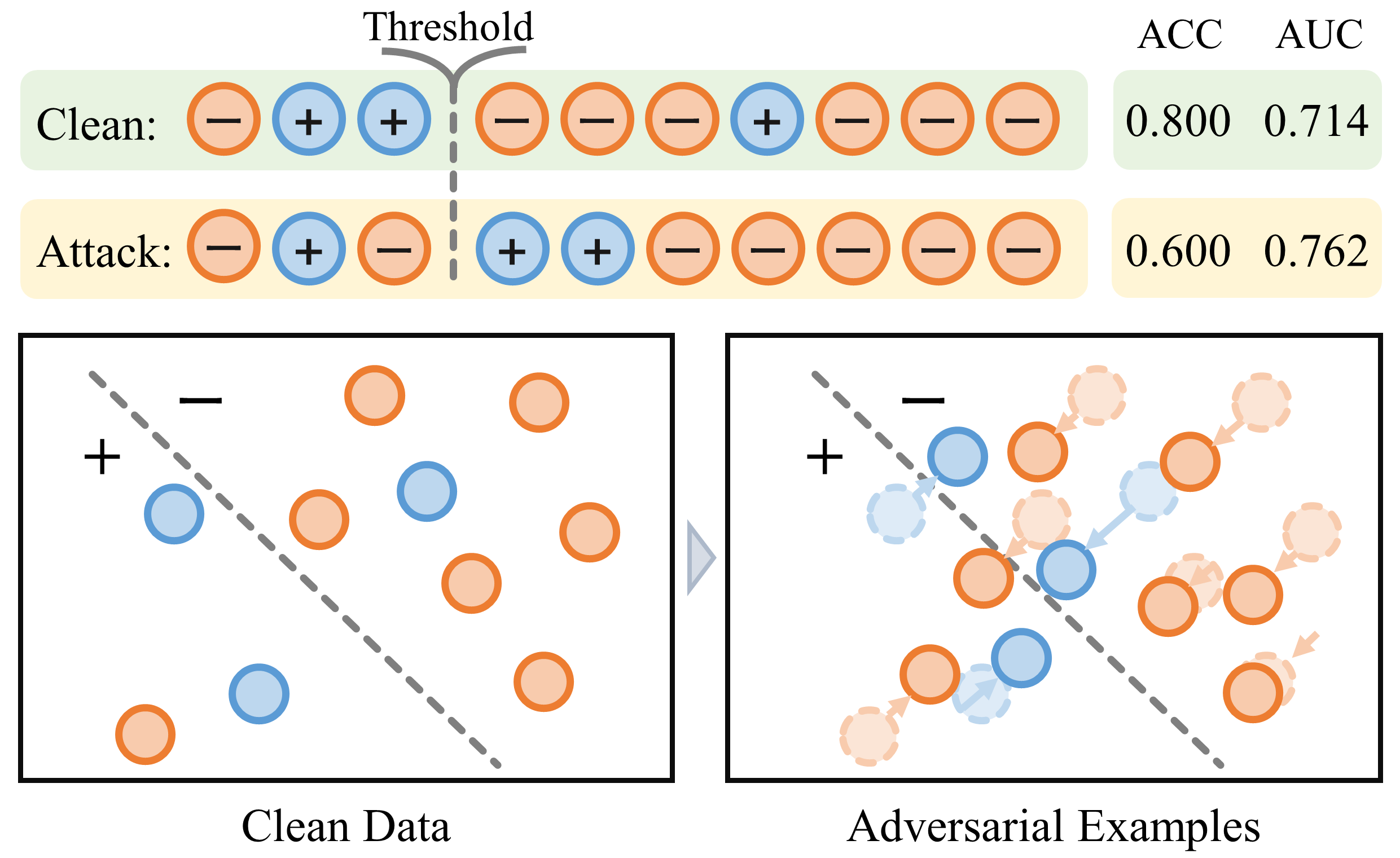}
	    \caption{Diagram of ACC and AUC change when the model is attacked. 
	    The \textbf{upper} rectangular boxes represent the score rank before and after the attack occurs; 
	    The \textbf{lower} plots represent the change of score in the embedding space when the model is attacked.}
	    \label{fig.motivation}
	\end{figure}
	
    Deep learning has recently achieved significant progress on various machine learning tasks, such as computer vision \cite{voulodimos2018deep} and natural language processing \cite{strubell2019energy, sorin2020deep}.
	However, recent work shows that deep learning models are vulnerable to adversarial attack \cite{szegedy2014intriguing, biggio2013evasion}. For example, images with human imperceptible perturbations (i.e., adversarial examples) can easily fool even the well-trained models.
	The existence of adversarial examples has raised big security threats to deep neural networks, which impels extensive efforts to improve the adversarial robustness \cite{madry2018towards, zhang2019theoretically}. 
	For example, one can resist the adversarial examples by means of adversarial training (AT). Specifically, AT could be formulated as a min-max problem, where the inner maximization problem is employed to generate adversarial examples, and the outer minimization problem is used to learn the model under the adversarial noise. In this way, AT could be easily applied to most of the modern architectures in deep learning, making it one of the most effective measures against adversarial attack.
	
	The prior art of adversarial training methods focuses on balanced benchmark datasets. On top of this, the learning objective is to increase the overall accuracy. However, real-world datasets usually exhibit a long-tail distribution that the proportion of the majority classes examples significantly dominates the others. For such long-tail problems, accuracy (ACC) is considered to be a less appropriate performance metric than another metric named AUC (Area Under the ROC Curve).
    Specifically, AUC is the probability of observing a positive instance with a higher score than a negative one. It is well-known to be insensitive to class distributions and costs \cite{fawcett2006introduction, hand2001simple}. Comparing AUC with ACC, we then ask:
	
	\textit{Can we improve the adversarial robustness of AUC by employing the traditional ACC-based methods?}

	Unfortunately, the answer might be negative. As shown in Fig.\ref{fig.motivation}, the adversarial examples generated by minimizing ACC (ACC drops from 0.8 to 0.6), may fail to attack AUC (AUC here increases from 0.714 to 0.762). In this sense, the model trained on such adversarial examples cannot improve the adversarial robustness of AUC. Therefore, more attention should be paid to AUC when studying the adversarial robustness against long-tail problems. 

	\textit{{Inspired by this fact, we present a very early trial to study adversarial training in AUC optimization with an end-to-end framework.}}
	
   Existing AT methods can be easily implemented in an end-to-end manner since the inner maximization problem for generating the adversarial examples can be solved instance-wisely. However, this is not the case for AUC optimization. Specifically, in the expression of AUC, every positive instance is coupled with all the negative instances and vice versa. For a binary class classification problem,  this means that we need to spend $O(n^- \cdot T)$ time for generating the adversarial example of a positive instance and $O(n^+ \cdot T)$ for a negative one,  where $n^+, n^-$ are the number of positive and negative instances, and $T$ is complexity for calculating the gradient for a single positive-negative instance pair. In this sense, we can hardly implement such a naive training method on top of even the simplest deep learning framework.
	
    To solve the challenge, this paper proposes an end-to-end adversarial AUC optimization framework with a convergence guarantee. Specifically, our contribution is as follows:

	First, based on a reformulation technique and a concavity regularizer, we show that the original problem is equivalent to a min-max problem where the objective function can be expressed in an instance-wise manner.
	
	Second, we propose an AT algorithm to optimize the min-max problem, where we alternately invoke a projected-gradient-descent-like protocol to generate the adversarial examples, and a stochastic gradient descent-ascent protocol to train the model parameters. Meanwhile, we also present a convergence analysis to show the correctness of our algorithm. The proof here is non-trivial since we have to simultaneously estimate the gradient of the min player and the max player.
	
    Finally, we conduct a series of empirical analyses of our proposed algorithm on long-tail datasets. The results demonstrate the effectiveness of our proposed method.

	\section{Related Work}
	\subsection{AUC Optimization}
	As a motivating study, \cite{cortes2003auc} investigates the inconsistency between AUC maximization and error rate minimization, which shows the necessity to study direct AUC optimization methods.
	After that, a series of algorithms are designed for off-line AUC optimization \cite{herschtal2004optimising, joachims2005support}. 
	To extend the scalability of AUC optimization, researchers start to explore the online and stochastic optimization extensions of the AUC maximization problem.
	\cite{zhao2011online} makes the first attempt for this direction based on the reservoir sampling technique.
	\cite{gao2013one} proposes a one-pass AUC optimization algorithm based on the squared surrogate loss.
	After that, \cite{ying2016stochastic} reformulates the minimization problem of the pairwise square loss into an equivalent stochastic saddle point problem, where the objective function could be expressed in an instance-wise manner.
	On top of the reformulation framework, \cite{natole2018stochastic} proposes an accelerated version with a faster convergence rate and \cite{liu2019stochastic} explores its extension in deep neural networks.
	Meanwhile, many researchers provide theoretical guarantees for AUC optimization algorithms from different aspects, such as generalization analysis \cite{agarwal2005generalization, clemenccon2008ranking, usunier2005data} and consistency analysis \cite{agarwal2014surrogate, gao2015consistency}.
	Beyond the optimization algorithms and theoretical supports for AUC, in practice, AUC optimization demonstrates its effectiveness in various class-imbalanced tasks, such as disease prediction \cite{westcott2019chronic, gola2020polygenic, ren2018robust}, rare event detection \cite{feizi2020hierarchical, robles2020threshold} and etc. 
	
	Compared with the existing study, we present a very early trial for the adversarial training problem.

	\subsection{Adversarial Training}
	
	For a long time, machine learning models have proved vulnerable to adversarial examples \cite{biggio2013evasion, szegedy2014intriguing, goodfellow2014generative}.
	Numerous defenses have been proposed to address the security concern raised by the issue \cite{athalye2018robustness, athalye2018obfuscated}.
	Among such studies, adversarial training is one of the most popular methods \cite{kurakin2016adversarial, madry2018towards, zhang2019theoretically}.
	The majority of studies in this direction follows the min-max formulation proposed in \cite{madry2018towards}, which so far has been improved in various way
	\cite{shafahi2020universal,cai2018curriculum,tramer2017ensemble, pang2019improving,wang2019improving, zhang2020geometry, maini2020adversarial, tramer2019adversarial}.
	Furthermore, due to the heavy computational burden of AT, accelerating the training procedure of AT becomes increasingly urgent. Recently, there has been a new wave to explore the acceleration of AT, which includes reusing the computations \cite{shafahi2019adversarial, zhang2019you}, adaptive adversarial steps \cite{wang2019convergence} and  one-step training \cite{wong2019fast}. Besides the practical improvements, there are also some recent advances in theoretical investigations from the perspective of optimization \cite{wang2019convergence, DBLP:journals/corr/abs-2201-01965}, generalization \cite{xing2021generalization, NEURIPS2019_16bda725}, and consistency \cite{bao2020calibrated}.

	In this paper, we will present an AT algorithm on top of the AUC optimization. As shown in the introduction, the complicated expression of AUC brings new elements into our model formulation and theoretical analysis. 
	
	\section{Preliminaries}
	In this section, we briefly introduce the AUC optimization problem and the adversarial training framework.
	
	\subsection{AUC Optimization Problem}
	\label{Preliminaries}
	Let $X$ be the feature set. Based on \cite{hanley1982meaning}, AUC of a scoring function $h_\theta: X \rightarrow [0,1]$ is equivalent to the probability that a positive instance is predicted with a higher score compared to a negative instance:
	\begin{equation} \nonumber
		\mathsf{AUC}(h_\theta) = \mathrm{Pr}\left(h_\theta(\boldsymbol{x}^+) \geq h_\theta(\boldsymbol{x}^-) | y^+=1, y^-=0\right),
	\end{equation}
	where $(\boldsymbol{x}^+, y^+)$ and $(\boldsymbol{x}^-, y^-)$ represent positive and negative examples, respectively, and $\theta$ is the model parameters.
	By employing a differentiable loss $\ell$ as the surrogate loss, 
	the unbiased estimation of $\mathsf{AUC}(h_\theta)$ could be expressed as: 
	\begin{equation} \nonumber
		\hat{\mathsf{AUC}}(h_\theta) = 1 - \sum_{i=1}^{n^+}\sum_{j=1}^{n^-} \frac{\ell\left(h_\theta(\boldsymbol{x}^+)-h_\theta(\boldsymbol{x}^-) \right)}{n^+ n^-} .
	\end{equation}
	where $n^+$ and $n^-$ denote the number of positive and negative examples, respectively.
	Then AUC maximization problem is equivalent to the following minimization problem:
	\begin{equation*}
	 \opz~~	\min\limits_{\boldsymbol{\theta}} \mathcal{L}(\boldsymbol{\theta}, \boldsymbol{x}, y) :=  \sum_{i=1}^{n^+}\sum_{j=1}^{n^-} \frac{\ell\left(h_\theta(\boldsymbol{x}^+)-h_\theta(\boldsymbol{x}^-) \right)}{n^+ n^-}.
	\end{equation*}
		
	\begin{table}[t]
		\label{table1}
		\centering
		\caption{Notations and their description}
		\begin{tabular}{ll}
			\hline
			Notations       &   Description \\ \hline
			$n$				&   Number of total examples          \\
			$n^+$, $n^-$			&   Number of positive (negative) examples          \\
			$p$             &   Proportion of positive examples        \\
			$\boldsymbol{x}^0$				&   Clean Examples \\
			$\boldsymbol{x}^k$				&   Adversarial examples generated in step $k$ \\
			$y$				&     The label of example        \\ 
			$\boldsymbol{\delta}$				&   Perturbation on samples \\
			$\mathcal{X}_i$				&  $\mathcal{X}_i=\{\boldsymbol{x}| \left\| \boldsymbol{x} - \boldsymbol{x}_i^0 \right\|_{\infty} \leq \epsilon\}$           \\
			$\boldsymbol{\theta}$       &    Parameters of model \\
			$a, b, \alpha$              &    Learnable parameters of loss function\\
			$\boldsymbol{w}$            &    $\boldsymbol{w}=(\boldsymbol{\theta}, a, b)$ \\
			$f(\boldsymbol{w}, \alpha, \boldsymbol{x})$   & The surrogate objective function     \\
			$L(\boldsymbol{w},\alpha)$         &   $\frac{1}{n} \sum_{i=1}^{n} f(\boldsymbol{w}, \alpha, \boldsymbol{x}_i^*)$ \\
			$\Phi(\boldsymbol{w})$     &    $\Phi(\boldsymbol{w})=\max_\alpha L(\boldsymbol{w}, \alpha)$   \\
			$\mathcal{B}$             &      The mini-batch                    \\
			$M$             &     The batch size                          \\
			$L$             &     $\max\{L_{ww},L_{wx},L_{\alpha\alpha},L_{\alpha{x}},L_{{x}w}, L_{x\alpha}\}$                            \\
			$T$           &     The total of training epochs                       \\
			$\hat{g}(\alpha), \hat{g}(\boldsymbol{w})$    &  Stochastic gradient \\
			$g(\alpha), g(\boldsymbol{w})$    &  Stochastic gradient \\
			\hline
		\end{tabular}
	\end{table}
	
	\subsection{Adversarial Training Framework}
	Adversarial training is one of the most effective defensive strategies against adversarial examples \cite{goodfellow2015explaining, madry2018towards}, the key idea of which is to directly optimize the model performance based on the perturbed examples. Generally speaking, the adversarial training framework can be formalized as 
	\begin{equation}
	\label{AT_framework}
		\begin{aligned}
			\min\limits_{\boldsymbol{\theta}} \frac{1}{n} \sum_{i=1}^{n} \max\limits_{\left\| \boldsymbol{\delta}_i \right\|_{\infty} \leq \epsilon} \ell(h_\theta(\boldsymbol{x}_i^0+\boldsymbol{\delta}_i), y_i) .
		\end{aligned}
	\end{equation}
	Here $\boldsymbol{\delta}_i$ is the perturbation on clean feature vector $\boldsymbol{x}_i^0$, and $\boldsymbol{x}_i = \boldsymbol{x}_i^0+\boldsymbol{\delta}_i$ is the resulting adversarial example for the instance $(\boldsymbol{x}_i^0, y_i), i=1,2,\cdots, n$. The inner maximization problem generates such adversarial examples by trying to hurt the model performance (by maximizing the loss $\ell(h_\theta(\boldsymbol{x}_i^0+\boldsymbol{\delta}_i), y)$). The constraint $\left\| \boldsymbol{\delta}_i \right\|_{\infty} \leq \epsilon$ makes sure that the adversarial perturbation is small enough to be imperceptible. In this sense, the adversarial example lives in $\mathcal{X}_i=\left\{  \boldsymbol{x} | \left\| \boldsymbol{x}-\boldsymbol{x}^0_i \right\|_\infty \leq  \epsilon \right\}$. Finally, the outer minimization problem is to find a robust model that can resist the adversarial perturbation. 
	
	For the inner maximization problem, K-PGD \cite{madry2018towards} is a widely used attack method that perturbs the clean examples $\boldsymbol{x}^0$ iteratively with a total of K steps. At the end of each iteration, the example will be projected to the $\epsilon$-ball of $\boldsymbol{x}^0$. Specifically, the adversarial examples generated in $k+1$ step are as follows:
	\begin{equation}
		\boldsymbol{x}^{k+1} = \mathrm{Proj} \big\{ \boldsymbol{x}^{k}+ \beta \cdot \operatorname{sign}(\nabla_{\boldsymbol{x}} \ell(h_\theta(\boldsymbol{x}^{k}), y^{k})) \big\},
	\end{equation}
	where $\mathrm{Proj}$ is the projection function, and $\beta$ is the step size. For the outer minimization problem, gradient descent is usually used to solve it. 
	A more detailed introduction of adversarial attack methods is shown in the Appendix \ref{adversarial_attacks}.

	\section{Methodology}
	\label{methodology}
	Before entering into the methodology, we summarize some useful notations in Tab.\ref{table1} to make our argument easier to follow.	
	\subsection{Reformulation of Optimization Problem}
	A naive idea to perform AUC adversarial training is to directly combine \opz~ with the standard AT framework \cite{madry2018towards}, resulting in the following problem:
	\begin{equation} \nonumber
		\min\limits_{\boldsymbol{\theta}} \max\limits_{\boldsymbol{\delta}_1, \boldsymbol{\delta}_2, \cdots, \boldsymbol{\delta}_n} \sum_{i=1}^{n^+} \sum_{j=1}^{n^-} \frac{\ell(h_\theta(\boldsymbol{x}_i^+ +\boldsymbol{\delta}_i)-h_\theta(\boldsymbol{x}_j^- +\boldsymbol{\delta}_j))}{n^+ n^-},
	\end{equation}
	According to the definition of AUC optimization objective function $\mathcal{L}$ in \opz~, we know that each pair of positive examples is inter-dependent with all negative examples, and vice versa for the negative examples.
	Thus, the inner maximization problem for $\boldsymbol{\delta}_1, \boldsymbol{\delta}_2, \cdots, \boldsymbol{\delta}_n $ cannot be decoupled into a series of instance-wise maximization problems. 
	In other words, the following inequality holds in general:
	\begin{gather*}
		\min\limits_{\boldsymbol{\theta}} \max\limits_{\boldsymbol{\delta}_1, \boldsymbol{\delta}_2, \cdots, \boldsymbol{\delta}_n}  \sum_{i=1}^{n^+} \sum_{j=1}^{n^-} \frac{\ell(h_\theta(\boldsymbol{x}_i^+ +\boldsymbol{\delta}_i)-h_\theta(\boldsymbol{x}_j^- +\boldsymbol{\delta}_j))}{n^+ n^-} \\
		\neq \\
		\min\limits_{\boldsymbol{\theta}} \sum_{i=1}^{n^+} \sum_{j=1}^{n^-} \max\limits_{\boldsymbol{\delta}_i,\boldsymbol{\delta}_j} \frac{ \ell(h_\theta(\boldsymbol{x}_i^+ +\boldsymbol{\delta}_i)-h_\theta(\boldsymbol{x}_j^- +\boldsymbol{\delta}_j))}{n^+ n^-}.
	\end{gather*}
    In this sense, the generation of adversarial examples cannot be carried out in a  mini-batch fashion. Instead, one update $\bm{\delta}$ requires a full scan of $O(n^+n^-)$. This brings a heavy computational burden towards its application. Therefore, we need to reformulate the optimization problem. 
	
	Fortunately, if we adopt the square loss $\ell(t) = (1-t)^2$ as the surrogate loss function, then \cite{ying2016stochastic, liu2019stochastic} proved that \opz~ could be converted in a min-max problem, as shown in the following proposition:
	\begin{prop}
	\label{prop:reform} 
	The empirical risk of AUC in \opz~ is equivalent to
	\begin{equation}
		\label{equ_new_optimazation_problem}
		\begin{aligned}
			\mathcal{L}(\boldsymbol{\theta},\boldsymbol{x}, y) = \min\limits_{a, b} \max\limits_{\alpha} \frac{1}{n} \sum^{n}_{i=1} g(\boldsymbol{\theta}, a, b, \alpha, (\boldsymbol{x}_i, y_i)), 
		\end{aligned}
	\end{equation}
	where
	\begin{equation}
	    \label{new_loss}
		\begin{aligned}
			g&(\boldsymbol{\theta}, a, b, \alpha, (\boldsymbol{x}_i, y_i)) \\
			=& (1-p)(h_\theta(\boldsymbol{x}_i)-a)^2\mathbb{I}_{[y_i=1]} + 
			p(h_\theta(\boldsymbol{x}_i)-b)^2\mathbb{I}_{[y_i=0]} \\
			 &+ 2(1+\alpha)\left(ph_\theta(\boldsymbol{x}_i)\mathbb{I}_{[y_i=0]} - (1-p)h_\theta(\boldsymbol{x}_i)\mathbb{I}_{[y_i=1]}\right)\\
			 &- p(1-p)\alpha^2.
		\end{aligned}
	\end{equation}
    where $a, b, \alpha \in \mathbb{R}$ are learnable parameters, and $p = \mathrm{Pr}(y = 1)$. 
	\end{prop}

\begin{rem}\label{rem:restrict}
	According to \cite{ying2016stochastic}, $a,b,\alpha$ has the following closed-form solution:   $a=\hat{\mathbb{E}}[h_{\theta}(\boldsymbol{x})|y=1]$, $b=\hat{\mathbb{E}}[h_{\theta}(\boldsymbol{x})|y=0]$ and $\alpha=\hat{\mathbb{E}}[h_{\theta}(\boldsymbol{x})|y=0]-\hat{\mathbb{E}}[h_{\theta}(\boldsymbol{x})|y=1]$, where $\hat{\mathbb{E}}$ is a shorthand for sample mean. If the score $h_\theta$ is normalized to the set $[0,1]$, we can restrict $a, b$ and $\alpha$ to the following bounded domains: 
    $$\Omega_{a, b} = \left\{a,b \in \mathbb{R} | 0 \leq a,b \leq 1 \right\}, \Omega_{\alpha} = \left\{\alpha \in \mathbb{R} | |\alpha| \leq 1 \right\}.$$
    And we can easily verify that $g$ is $\mu$-strongly concave w.r.t. $\alpha$ in $\Omega_{\alpha}$, i.e., for any $\alpha_1, \alpha_2 \in \Omega_{\alpha}$, it holds that 
    \begin{equation}\nonumber
		\begin{aligned}
			g(\boldsymbol{\theta},& a, b, \alpha_1, (\boldsymbol{x}, y)) \leq g(\boldsymbol{\theta},a, b, \alpha_2, (\boldsymbol{x},y) ) + \\
			&\langle \nabla_\alpha g(\boldsymbol{\theta}, a, b, \alpha_2, (\boldsymbol{x}, y)), \alpha_1-\alpha_2 \rangle - \frac{\mu}{2}\left\| \alpha_1 - \alpha_2 \right\|_2^2.
		\end{aligned}
	\end{equation}
    And we can also verified that $g$ is locally strongly convex in $\Omega_{a, b}$ w.r.t. $a$ and $b$.
\end{rem}

	Hence, if we in turn construct an AT problem based on Prop.\ref{prop:reform}, we can obtain the following optimization problem with ease:
	\begin{equation} \nonumber
		\begin{aligned}
			\min\limits_{\boldsymbol{\theta}} \max\limits_{\boldsymbol{\delta}} \min\limits_{a, b \in \Omega_{a, b} } \max_{\alpha \in \Omega_{\alpha}} \frac{1}{n} \sum^{n}_{i=1}
			g(\boldsymbol{\theta}, a, b, \alpha, (\boldsymbol{x}_i+\boldsymbol{\delta}_i, y_i)).
		\end{aligned}
	\end{equation}
	
	The good news here is that in the new loss function $g$, positive samples and negative samples are independent of each other. However, the bad news is that the min-max-min-max problem is still hardly tractable. Through a careful investigation, if we can swap the order of $\min_{a,b}$ and $\max_{\bm{\delta}}$, we can then obtain a min-max problem which can be solved in an end-to-end fashion. To realize the idea, we could resort to the von Neumann's Minimax theorem \cite{neumann1928theorie, sion1958general}:
\begin{thm}
	Let $X \subset \mathbb{R}^n$ and $Y \subset \mathbb{R}^m$ be compact convex sets. If $f: X \times Y \rightarrow \mathbb{R}$ is a continuous function that is concave-convex, i.e. 
	\begin{equation} \nonumber
	    \begin{aligned}
	        &f(\cdot, \boldsymbol{y}): X \rightarrow \mathbb{R} \text{ is concave for fixed $\boldsymbol{y}$,}\\
	        &f(\boldsymbol{x}, \cdot): Y \rightarrow \mathbb{R} \text{ is convex for fixed $\boldsymbol{x}$.}
	    \end{aligned}
	\end{equation}

	Then we have that $\max\limits_{\boldsymbol{x} \in X} \min\limits_{\boldsymbol{y} \in Y} f(\boldsymbol{x}, \boldsymbol{y}) = \min\limits_{\boldsymbol{y} \in Y} \max\limits_{\boldsymbol{x} \in X} f(\boldsymbol{x}, \boldsymbol{y})$.
	
\end{thm}
Moreover, we resort to the definition of weak-concavity \cite{bohm2021variable, liu2021first}:
\begin{defi}
	$f(\bm{x}):\mathbb{R}^d \rightarrow \mathbb{R}$ is said to be a $\gamma$-weakly concave ($\gamma > 0$) function w.r.t. $\bm{x}$, if
	\begin{align*}
		f(\bm{x}) -\frac{\gamma}{2} ||\bm{x}||_2^2 
	\end{align*} 
	is a concave function w.r.t. $\bm{x}$.
\end{defi}
In the following proposition, we find a surrogate objective function $f(\boldsymbol{w}, \alpha, \boldsymbol{x}_i+\boldsymbol{\delta}_i)$ such that the resulting optimization problem could be reformulated as a min-max problem:

\begin{prop}\label{prop:doublereform}
Define: 
\begin{align*}
	\small
	&r(a, b, \bm{x}) =  \max_{\alpha} \frac{1}{n} \sum^{n}_{i=1}g(\boldsymbol{\theta}, a, b, \alpha, (\boldsymbol{x}_i+\boldsymbol{\delta}_i, y_i)),\\ 
	&f(\boldsymbol{w}, \alpha, \boldsymbol{x}_i+\boldsymbol{\delta}_i) = g(\boldsymbol{\theta}, a, b, \alpha, (\boldsymbol{x}_i+\boldsymbol{\delta}_i, y_i))\\ 
	&~~~~~~~~~~~~~~~~~~~~~~~~~~~~~~~~- \gamma\left\| \boldsymbol{x}_i + \boldsymbol{\delta}_i \right\|^2_2
\end{align*}

If $r(a, b, \bm{x})$ is $\gamma_*$-weakly concave w.r.t. $\bm{\delta}_1,\bm{\delta}_2,\cdots, \bm{\delta}_n$, then for all $\gamma > \gamma_\star$, we have the following problem:
\begin{align*}
\small
	\min\limits_{\boldsymbol{\theta}} \max\limits_{\boldsymbol{\delta}} \min\limits_{a,b \in \Omega_{a,b} } \max_{\alpha \in \Omega_{\alpha}} \frac{1}{n} \sum^{n}_{i=1} f(\bm{w}, \alpha, (\boldsymbol{x}_i+\boldsymbol{\delta}_i, y_i))
\end{align*}
is equivalent to:
\begin{equation}\nonumber
	\label{minmaxmax}
		\begin{aligned}
		\op ~~	&\min\limits_{\boldsymbol{w}} \max\limits_{\alpha} \max\limits_{\boldsymbol{\delta}} 
			\frac{1}{n} \sum^{n}_{i=1}\left[f(\boldsymbol{w}, \alpha, \boldsymbol{x}_i+\boldsymbol{\delta}_i)\right] \\
			& = \min\limits_{\boldsymbol{w}} \max\limits_{\alpha} \frac{1}{n} \sum^{n}_{i=1} \max\limits_{\boldsymbol{\delta}_i} 
			\left[f(\boldsymbol{w}, \alpha, \boldsymbol{x}_i+\boldsymbol{\delta}_i)\right],
		\end{aligned}
	\end{equation}
where $\bm{w} = (\bm{\theta},a,b)$.
Moreover, $f(\boldsymbol{w}, \alpha, \boldsymbol{x}_i+\boldsymbol{\delta}_i)$ is strongly concave w.r.t. $\bm{\delta}_i$.
\end{prop}



	
\begin{rem}
	According to \cite{bohm2021variable, liu2021first}, if $\max_{\alpha} \frac{1}{n} \sum^{n}_{i=1} \left[f(\boldsymbol{w}, \alpha, \boldsymbol{x}_i+\boldsymbol{\delta}_i)\right]$ is smooth and $L$ gradient Lipschitz, then it is also $L$-weakly concave. Hence, the weakly concavity assumption is much weaker than the strongly-concave assumption \cite{wang2019convergence}. 
\end{rem}
In this sense, we could turn to optimize \op ~in the next subsection.
\subsection{Training Strategy}
	In this subsection, we continue to propose an adversarial AUC optimization framework to solve \op. 
	Specifically, we design solutions for the \textit{inner maximization problem} and the \textit{outer min-max problem}, respectively.
	
	\textbf{Inner Maximization Problem: Adversarial Attack.}
	In this paper, we choose K-PGD \cite{madry2018towards} to generate adversarial examples. To better control the quality of adversarial examples, we introduce the First-Order Stationary Condition (FOSC) \cite{wang2019convergence} about the inner maximization problem, which is as follows:
	\begin{equation}
	    \label{FOSC}
		c(\boldsymbol{x}^k) = \max\limits_{\boldsymbol{x} \in \mathcal{X}} \langle \boldsymbol{x}-\boldsymbol{x}^k, \nabla_{\boldsymbol{x}} f(\boldsymbol{w}, \alpha, \boldsymbol{x}^k) \rangle
	\end{equation}
	
	When $c(\boldsymbol{x}^k)=0$, the optimization problem reaches the convergence state. Specifically, such a condition can be achieved when \textbf{a)} $\nabla_{\boldsymbol{x}} f(\boldsymbol{w}, \alpha, \boldsymbol{x}^k)=0$, or \textbf{b)} $\boldsymbol{x}^k-\boldsymbol{x}^0= \epsilon \cdot \operatorname{sign}\left( \nabla_{\boldsymbol{x}} f(\boldsymbol{w}, \alpha, \boldsymbol{x}^k) \right)$.
	Here \textbf{a)} implies $\boldsymbol{x}^k$ is a stationary point in the inner maximization problem, \textbf{b)} shows that local maximum point of $f(\boldsymbol{w}, \alpha, \boldsymbol{x}^k)$ reaches the boundary of $\mathcal{X}$.	The proof process is shown in Lem.\ref{Lemma1}.
	
	\begin{algorithm}[t]
		\caption{Adversarial Training for AUC Optimization}
		\label{algorithm1}
	\begin{algorithmic}
	    \STATE {\bfseries Input:}Neural network $h_\theta$; initial parameters $\boldsymbol{w}_0 = \{\boldsymbol{\theta}^0, a^0, b^0\}$ and $\alpha^0$; step size $\eta_{w}$, $\eta_{\alpha}$; mini-batch $\mathcal{B}$ and its size $M$; max FOSC value $c_{max}$; training epochs $T$; control epoch $T'$; PGD step $K$; PGD step size $\beta$; maximum perturbation boundary $\epsilon$.
	    \FOR{$t=0$ {\bfseries to} $T$}
	    \STATE $c_t = \max(0, c_{max}-t\cdot c_{max}/{T'})$
	    \FOR {Each batch $\boldsymbol{x}_{\mathcal{B}}^0$}
	    \STATE $M_c = \mathbbm{1}_{\mathcal{B}}$; $k=0$
	    \WHILE{$\sum M_c>0 \ \  \& \ \ k< K$}
	    \STATE $\boldsymbol{x}^{k+1}_{\mathcal{B}} = \boldsymbol{x}^{k}_{\mathcal{B}} + M_c \cdot \beta \cdot sign(\nabla_{\boldsymbol{x}}\ell(h_\theta(\boldsymbol{x}^{k}_{\mathcal{B}}), y))$
		\STATE $\boldsymbol{x}^{k+1}_{\mathcal{B}} = clip(\boldsymbol{x}^{k+1}_{\mathcal{B}}, \boldsymbol{x}^{k+1}_{\mathcal{B}}-\epsilon, \boldsymbol{x}^{k+1}_{\mathcal{B}}+\epsilon)$
					
					
		\STATE $M_c = \mathbbm{1}_{\mathcal{B}}(c(x_{1 \dots M}^{k+1}) \leq c_t)$
	    \STATE $k = k+1$
	    \ENDWHILE
	    \STATE $\alpha^{t+1} = \alpha^t + \eta_{\alpha} \hat{g}(\alpha)$
		\STATE $\boldsymbol{w}^{t+1} = \boldsymbol{w}^t - \eta_{w} \hat{g}(\boldsymbol{w})$  \ \ \# $\hat{g}$: stochastic gradient
	    \ENDFOR
	    \ENDFOR
	    \STATE {\bfseries return} $\boldsymbol{w}^T, \alpha^T$
	\end{algorithmic}
	\end{algorithm}
	
	\textbf{Outer min-max Problem.}
	For the outer min-max problem, we apply \textit{Stochastic Gradient Descent Ascent} (SGDA) to solve the problem. At each iteration, SGDA performs stochastic gradient descent over the parameter $\boldsymbol{w}$ with the stepsize $\eta_w$, and stochastic gradient ascent over the parameter $\alpha$ with the stepsize $\eta_\alpha$. 
	
	The total training strategy is presented in Alg.\ref{algorithm1}. This is an extension of the algorithm proposed in \cite{wang2019convergence}, where the outer level minimization problem now becomes a min-max problem. The value of FOSC can imply the adversarial strength of adversarial examples, whereas a small FOSC value implies a high adversarial strength.
	Due to this fact, through the FOSC value of current epoch $c_t$, we can dynamically control the strength of adversarial examples.
    Specifically, in the initial stage of training, the value of $c_t$ is large, which means the generated adversarial examples are not so hard. 
    In the later stage of training, the value of $c_t$ is 0, which allows model to be trained on much stronger adversarial examples. Consequently, such an algorithm will allow the model to learn from in a progressive manner. Here, we use $M_c$ to mask the examples that satisfy the condition in the \texttt{Line 9} in Alg.\ref{algorithm1}.
    By doing so, we can ensure that the FOSC values of the adversarial examples generated by Alg. \ref{algorithm1} are all less than $c_{max}$.
    When the adversarial examples are obtained, we calculate the stochastic gradient of the parameters
    $\boldsymbol{w}$ and $\alpha$. Then we perform stochastic gradient descent-ascent on $\boldsymbol{w}$ and $\alpha$ respectively.

	\subsection{Convergence Analysis}
	Next, we provide a convergence analysis of our proposed adversarial AUC optimization framework.
	
	We first give the definition and description of some notations.
	In detail, let $\boldsymbol{x}^*_i(\boldsymbol{w}, \alpha)= \arg \max_{\boldsymbol{x}_i \in \mathcal{X}_i} f(\boldsymbol{w}, \alpha, \boldsymbol{x}_i)$ where $\mathcal{X}_i = \left\{ \boldsymbol{x} | \left\|  \boldsymbol{x}-\boldsymbol{x}_i^0 \right\|_\infty \leq \epsilon \right\}$. 
	And
	\begin{equation} \nonumber
	\small
		\begin{aligned}
			L(\boldsymbol{w}, \alpha) &= \frac{1}{n}\sum_{i=1}^{n} \max\limits_{\boldsymbol{x}_i \in \mathcal{X}_i} 
			f(\boldsymbol{w}, \alpha, \boldsymbol{x}_i) = \frac{1}{n} \sum_{i=1}^{n} f(\boldsymbol{w}, \alpha, \boldsymbol{x}_i^*).
		\end{aligned}
	\end{equation}
	Then $\hat{\boldsymbol{x}}_i(\boldsymbol{w}, \alpha)$ is a $\delta$-approximate solution to $\boldsymbol{x}_i^*(\boldsymbol{w}, \alpha)$, if it satisfies that
	\begin{equation}
		\max\limits_{\boldsymbol{x} \in \mathcal{X}_i} \langle \boldsymbol{x} - \hat{\boldsymbol{x}}_i(\boldsymbol{w}, \alpha), \nabla_{\boldsymbol{x}} f(\boldsymbol{w}, \alpha, \hat{\boldsymbol{x}}_i(\boldsymbol{w}, \alpha)) \rangle \leq \delta.
	\end{equation}
	
	Furthermore, let $\nabla L(\alpha)$ denote the gradient of $L(\boldsymbol{w}, \alpha)$ w.t.r. $\alpha$.
	And let $g(\alpha)=\frac{1}{M} \sum_{i \in \mathcal{B}} \nabla_\alpha f(\boldsymbol{w}, \alpha, \boldsymbol{x}_i^*)$ be the stochastic gradient of $L(\boldsymbol{w}, \alpha)$ w.r.t. $\alpha$, where $\mathcal{B}$ is mini-batch and $M=| \mathcal{B} |$. 
	Meanwhile, let $\nabla_\alpha f(\boldsymbol{w}, \alpha, \hat{\boldsymbol{x}}(\boldsymbol{w}, \alpha))$ be the gradient of $f(\boldsymbol{w}, \alpha, \hat{\boldsymbol{x}}(\boldsymbol{w}, \alpha))$ w.r.t. $\alpha$, and let $\hat{g}(\alpha)=\frac{1}{M} \sum_{i \in \mathcal{B}} f(\boldsymbol{w}, \alpha, \hat{\boldsymbol{x}}_i)$ be the approximate stochastic gradient of $L(\boldsymbol{w}, \alpha)$ w.r.t. $\alpha$.
	And for $\boldsymbol{w}$, we have the same definition as $\alpha$.
	In addition, let 
	\begin{align*}
		\Phi(\boldsymbol{w}) = \max\limits_\alpha L(\boldsymbol{w}, \alpha),~~~ \nabla\Phi(\boldsymbol{w}) = \nabla_{\boldsymbol{w}} L(\boldsymbol{w}, \alpha^*(\boldsymbol{w})).
	\end{align*}
		
	Then, before giving the convergence analysis, we list some assumptions needed to the analysis.
	
	\begin{asm}
	\label{asm1}
	The function $f(\boldsymbol{w}, \alpha, \boldsymbol{x})$ satisfies the gradient Lipschitz conditions as follows:
	\begin{equation} \nonumber
		\begin{aligned}
		    &\sup\limits_{\alpha, \boldsymbol{x}} \left\| \nabla_{\boldsymbol{w}} f(\boldsymbol{w}, \alpha, \boldsymbol{x}) - \nabla_{\boldsymbol{w}} f(\boldsymbol{w}', \alpha, \boldsymbol{x}) \right\|_2 \leq L_{w w}\left\| \boldsymbol{w}-\boldsymbol{w}'\right\|_2 \\
			&\sup\limits_{\alpha, \boldsymbol{w}} \left\| \nabla_{\boldsymbol{w}} f(\boldsymbol{w}, \alpha, \boldsymbol{x}) -\nabla_{\boldsymbol{w}} f(\boldsymbol{w}, \alpha, \boldsymbol{x}') \right\|_2 \leq L_{w x}\left\| \boldsymbol{x}-\boldsymbol{x}'\right\|_2 \\
			&\sup\limits_{\alpha, \boldsymbol{x}} \left\| \nabla_{\boldsymbol{x}} f(\boldsymbol{w}, \alpha, \boldsymbol{x}) -\nabla_{\boldsymbol{x}} f(\boldsymbol{w}', \alpha, \boldsymbol{x}) \right\|_2 \leq L_{x w}\left\| \boldsymbol{w}-\boldsymbol{w}'\right\|_2 \\
			&\sup\limits_{\boldsymbol{w}, \boldsymbol{x}} \left\| \nabla_\alpha f(\boldsymbol{w}, \alpha, \boldsymbol{x}) -\nabla_\alpha f(\boldsymbol{w}, \alpha', \boldsymbol{x}) \right\|_2 \leq L_{\alpha \alpha}\left\| \alpha-\alpha'\right\|_2 \\
			&\sup\limits_{\boldsymbol{w}, \alpha} \left\| \nabla_\alpha f(\boldsymbol{w}, \alpha, \boldsymbol{x}) -\nabla_\alpha f(\boldsymbol{w}, \alpha, \boldsymbol{x}') \right\|_2 \leq L_{\alpha x}\left\| \boldsymbol{x}-\boldsymbol{x}'\right\|_2 \\
			&\sup\limits_{\boldsymbol{w}, \boldsymbol{x}} \left\| \nabla_{\boldsymbol{x}} f(\boldsymbol{w}, \alpha, \boldsymbol{x}) -\nabla_{\boldsymbol{x}} f(\boldsymbol{w}, \alpha', \boldsymbol{x}) \right\|_2 \leq L_{x \alpha}\left\| \alpha-\alpha'\right\|_2 \\
		\end{aligned}
	\end{equation}
	where $L_{\alpha \alpha}, L_{\alpha x}, L_{x \alpha}, L_{w w}, L_{w x}, L_{x w}$ are positive constants.
	\end{asm}
	
	\begin{rem}
	The first three gradient Lipschitz conditions in Asm.\ref{asm1} are made in \cite{sinha2018certifying}, and the last three gradient Lipschitz conditions are made in \cite{liu2019stochastic}. 
	Meanwhile, for the overparameterized deep neural network, the loss function is semi-smooth \cite{allen2019convergence, du2019gradient}, which helps to justify Asm.\ref{asm1}.
	\end{rem}
	
	\begin{asm}
	\label{asm2}
	    $\left\| \nabla_{\boldsymbol{w}}f(\boldsymbol{w}, a, \boldsymbol{x})\right\|_2$ is upper bounded by $l_{w}$.
	\end{asm}
	
	\begin{rem}
	    Asm.\ref{asm2} is widely used in minimax optimization problems \cite{sinha2018certifying}.
	\end{rem}
	
	\begin{asm}
	\label{asm3}
	    $f(\boldsymbol{w}, \alpha, \boldsymbol{x})$ is locally $\mu$-strongly concave in $\mathcal{X}_i$ for all $i \in [n]$, i.e. for any $\boldsymbol{x}_1, \boldsymbol{x}_2 \in \mathcal{X}_i$, it holds that
	\begin{equation}\nonumber
		\begin{aligned}
			f(\boldsymbol{w}, &\alpha, \boldsymbol{x}_1) \leq f(\boldsymbol{w}, \alpha, \boldsymbol{x}_2) + \\
			&\langle \nabla_{\boldsymbol{x}} f(\boldsymbol{w}, \alpha, \boldsymbol{x}_2), \boldsymbol{x}_1-\boldsymbol{x}_2 \rangle - \frac{\mu}{2}\left\| \boldsymbol{x}_1 - \boldsymbol{x}_2 \right\|_2^2
		\end{aligned}
	\end{equation}
	\end{asm}
	
	\begin{rem}
		The strongly concave assumption is equivalent to weakly concave assumption of $r$, which is much easier to be achieved.
	\end{rem}

    \begin{asm}
    \label{asm4}
    The stochastic gradient $g(\alpha)$ satisfies
	\begin{equation} \nonumber
		\begin{aligned}
			&\mathbb{E}[g(\alpha) - \nabla L(\alpha)] = 0, 
			&\mathbb{E}\left[\left \|g(\alpha) - \nabla L(\alpha)\right \|^2_2\right] \leq \sigma^2 \\
			&\mathbb{E}[g(\boldsymbol{w}) \!- \!\nabla L(\boldsymbol{w})] = 0, 
			&\mathbb{E}\left[\left \|g(\boldsymbol{w}) - \nabla L(\boldsymbol{w})\right \|^2_2\right] \leq \sigma^2 \\
		\end{aligned}
	\end{equation}
    \end{asm}
	
	\begin{rem}
		Asm.\ref{asm4} is a common assumption used to analyze the optimization algorithm based on stochastic gradient \cite{lin2020gradient, liu2019stochastic}.
	\end{rem}
	
\begin{thm}
    \label{thm:minimax}
	Under the above assumptions and let the stepsizes be chosen as $\eta_w = \Theta	\left(1/\kappa^2 L \right)$, $\eta_\alpha=\Theta(1/L)$ and $\kappa \leq \frac{7}{6}$, then we have the following inequality:
	\begin{equation}\nonumber
		\begin{aligned}
			 \frac{1}{T+1}&\left( \sum_{t=0}^{T} \mathbb{E}\left[ \left\| \nabla \Phi(\boldsymbol{w}_t) \right\|_2^2 \right] \right)  \\
			 & \leq \frac{360\kappa^2L\Delta_\Phi + 13\kappa L^2D^2}{T+1} + \frac{26\kappa \sigma^2}{M} + h_\delta + h_\Delta
		\end{aligned}
	\end{equation}
	where $h_\delta = \frac{1024}{253}\left( \frac{3\kappa L \delta}{1024} + 6\kappa^4 L \delta + \frac{L\delta}{8} + L^2 \sqrt{\frac{\delta}{\mu}} \right)$ and $h_\Delta= \frac{384\kappa^4 L^2 \Delta}{253}$ and $\Delta_\Phi = \Phi(\boldsymbol{w}^0) - \min_{\boldsymbol{w}} \Phi(\boldsymbol{w})$, $D=|\Omega_{\alpha}|$, $L$ denotes the maximum in Lipschitz constant, and the condition number $\kappa=L/\mu$.
\end{thm}
	
	\begin{rem}
		On the right side of the inequality, the first term is an $O(1/T)$ magnitude, the last three terms behave as the residuals.  The second term $\frac{26\kappa \sigma^2}{M}$ is related to the batch size $M$. As $M$ increases, its value gradually tends to 0.
		The third term $h_\delta$ is due to the use of an approximate solution $\hat{\boldsymbol{x}}$ to the inner maximization problem instead of the optimal solution $\boldsymbol{x}^*$, and the fourth term $h_\Delta$ is due to $\mathcal{X}$ being a bounded set. Since all of the residuals are small, we can find an $\epsilon$-stationary point within a finite number of epochs.
	\end{rem}

	\section{Experiments}
	In this section, we evaluate the performance of our AdAUC algorithms in three long-tail datasets.
	
		\begin{table*}[t]
		\centering
		\caption{Test AUC and robustness of models trained with various methods. 
		We mark the best performance with \textbf{Bold} and the second best performance with \underline{underscore}.
		\label{table_all}}
		\begin{tabular}{c|c|c|ccccccc}
			\toprule
			\multirow{2}{*}{Dataset}  & \multirow{2}{*}{Method}                      & \multirow{2}{*}{Training} & \multicolumn{7}{c}{Evaluated Against} \\ \cline{4-10} 
			&                       &             & Clean  & FSGM & PGD-5  & PGD-10 & PGD-20 & C$\&$W & AA \\ \hline
			\multicolumn{1}{c|}{\multirow{6}{*}{CIFAR-10-LT}} & \multicolumn{1}{c|}{\multirow{3}{*}{CE}} & NT       & 0.7264  & 0.4038  & 0.0753 & 0.0206 & 0.0044 & 0.0009 & 0.0082 \\
			\multicolumn{1}{c|}{} & \multicolumn{1}{c|}{}                 &   AT$_1$         &  0.6659 & 0.5487  & 0.3335 & 0.2743 & 0.2344 & 0.2330 & 0.2678 \\ 
			\multicolumn{1}{c|}{} & \multicolumn{1}{c|}{}                  &    AT$_2$    & 0.6833  & 0.6296  & 0.4870 & 0.4417 & \underline{0.4319} & \underline{0.4310} & \underline{0.4384} \\ 
			\cline{2-10} 
			\multicolumn{1}{c|}{} & \multicolumn{1}{c|}{\multirow{3}{*}{AdAUC}} &  NT          &  \textbf{0.7885} & 0.6606  & 0.2671 & 0.1892  & 0.0064 & 0.0573 & 0.0740 \\
			\multicolumn{1}{c|}{} & \multicolumn{1}{c|}{}                  &    AT$_1$     &  0.7347 & \underline{0.6646}  & \underline{0.5236} & \underline{0.4625} & 0.4224 & 0.3927 & 0.4362 \\ 
			\multicolumn{1}{c|}{} & \multicolumn{1}{c|}{}                  &    AT$_2$     & \underline{0.7528}  & \textbf{0.6952}  & \textbf{0.5591}  & \textbf{0.5309} & \textbf{0.5283} & \textbf{0.5283} & \textbf{0.5291} \\ 
			\midrule
			\multicolumn{1}{c|}{\multirow{6}{*}{CIFAR-100-LT}} & \multicolumn{1}{c|}{\multirow{3}{*}{CE}} & NT       &  \underline{0.6382} &  0.1207 & 0.0271 & 0.0159 & 0.0110 & 0.0102 & 0.0123 \\
			\multicolumn{1}{c|}{} & \multicolumn{1}{c|}{}                 &   AT$_1$         &  0.6193 & 0.5183  & 0.3195 & 0.2750 & 0.2668 & 0.2630 & 0.2703 \\ 
			\multicolumn{1}{c|}{} & \multicolumn{1}{c|}{}                  &    AT$_2$    &  0.6198 & 0.5192  & 0.3183 & 0.2767 & 0.2681 & 0.2647 & 0.2712 \\ 
			\cline{2-10} 
			\multicolumn{1}{c|}{} & \multicolumn{1}{c|}{\multirow{3}{*}{AdAUC}} &  NT          & \textbf{0.6462}  &  0.5161 & 0.3046 & 0.1818 & 0.1214 & 0.0035 & 0.1313 \\
			\multicolumn{1}{c|}{} & \multicolumn{1}{c|}{}                  &    AT$_1$     &  0.6302 & \underline{0.5301}  & \underline{0.3815} & \underline{0.3306} & \underline{0.2989} & \underline{0.2760} & \underline{0.3102} \\ 
			\multicolumn{1}{c|}{} & \multicolumn{1}{c|}{}                  &    AT$_2$     & 0.6313 & \textbf{0.5798}  & \textbf{0.4644} & \textbf{0.4234} & \textbf{0.4065} & \textbf{0.3968} & \textbf{0.4122} \\ 
			\midrule
			\multicolumn{1}{c|}{\multirow{6}{*}{MNIST-LT}} & \multicolumn{1}{c|}{\multirow{3}{*}{CE}} & NT       &  0.9736 & 0.7057  & 0.0116 & 0.0010 & 0.0002 & 0.0000 & 0.0003 \\
			\multicolumn{1}{c|}{} & \multicolumn{1}{c|}{}                 &   AT$_1$         &  0.9488 & 0.9302  & 0.8733 & 0.8626 & 0.8615 & 0.8611 & 0.8618 \\
			\multicolumn{1}{c|}{} & \multicolumn{1}{c|}{}                  &    AT$_2$    & 0.9547  & 0.9392  & 0.8912 & 0.8824 & 0.8816 & 0.8813 & 0.8818 \\ 
			\cline{2-10} 
			\multicolumn{1}{c|}{} & \multicolumn{1}{c|}{\multirow{3}{*}{AdAUC}} &  NT          &  \textbf{0.9904} & 0.9309  & 0.5677 & 0.4419 & 0.3913 & 0.3645 &  0.4026\\
			\multicolumn{1}{c|}{} & \multicolumn{1}{c|}{}                  &    AT$_1$    &  0.9772 & \underline{0.9695}  & \underline{0.9422} & \textbf{0.9395} & \textbf{0.9382} & \textbf{0.9381} &  \textbf{0.9383} \\ 
			\multicolumn{1}{c|}{} & \multicolumn{1}{c|}{}                  &    AT$_2$     & \underline{0.9852}  & \textbf{0.9774}  & \textbf{0.9436} & \underline{0.9347} & \underline{0.9323} & \underline{0.9310} & \underline{0.9311} \\ 
			\bottomrule
		\end{tabular}
	\end{table*}
	\vspace{-0.3cm}
	
	\subsection{Competitors and Experiment Setting.}    
	We compare the performance of our proposed algorithm and the AT methods with the classical \textbf{CE} classification loss function when the datasets have the long tail distribution.
	
	We adopt the WideResNet-28 \cite{zagoruyko2016wide} as the model architecture. 
	The other detailed settings are shown in App.\ref{Experiment_Setting}.
	In Tab.\ref{table_all}, \textbf{NT} means Natural Training without adversarial operations,  \textbf{AT$_1$} means adversarial training without FOSC, and \textbf{AT$_2$} means algorithm in Alg.\ref{algorithm1}. 
	
	To validate the robustness of our algorithm, we adopt FSGM \cite{goodfellow2014generative}, iterative attack PGD \cite{madry2018towards},  C$\&$W \cite{carlini2017towards} and ensemble attack AA \cite{croce2020reliable} as attack methods.

	\subsection{Dataset Description}
	
	\textbf{Binary CIFAR-10-LT Dataset.}
	We construct a long-tail CIFAR-10 dataset, where the sample size across different classes decays exponentially and ensure the ratio of sample sizes of the least frequent to the most frequent class is set to 0.01. Then, we label the first 5 classes as the negative class and the last 5 classes as positive, which leads that the ratio of  positive class size to negative class size \textbf{$\rho \approx 1:9$}.

	\begin{figure}[h!]
        \centering 
        \subfigure[AdAUC NT on Clean Data]{
        \includegraphics[width=0.45\columnwidth]{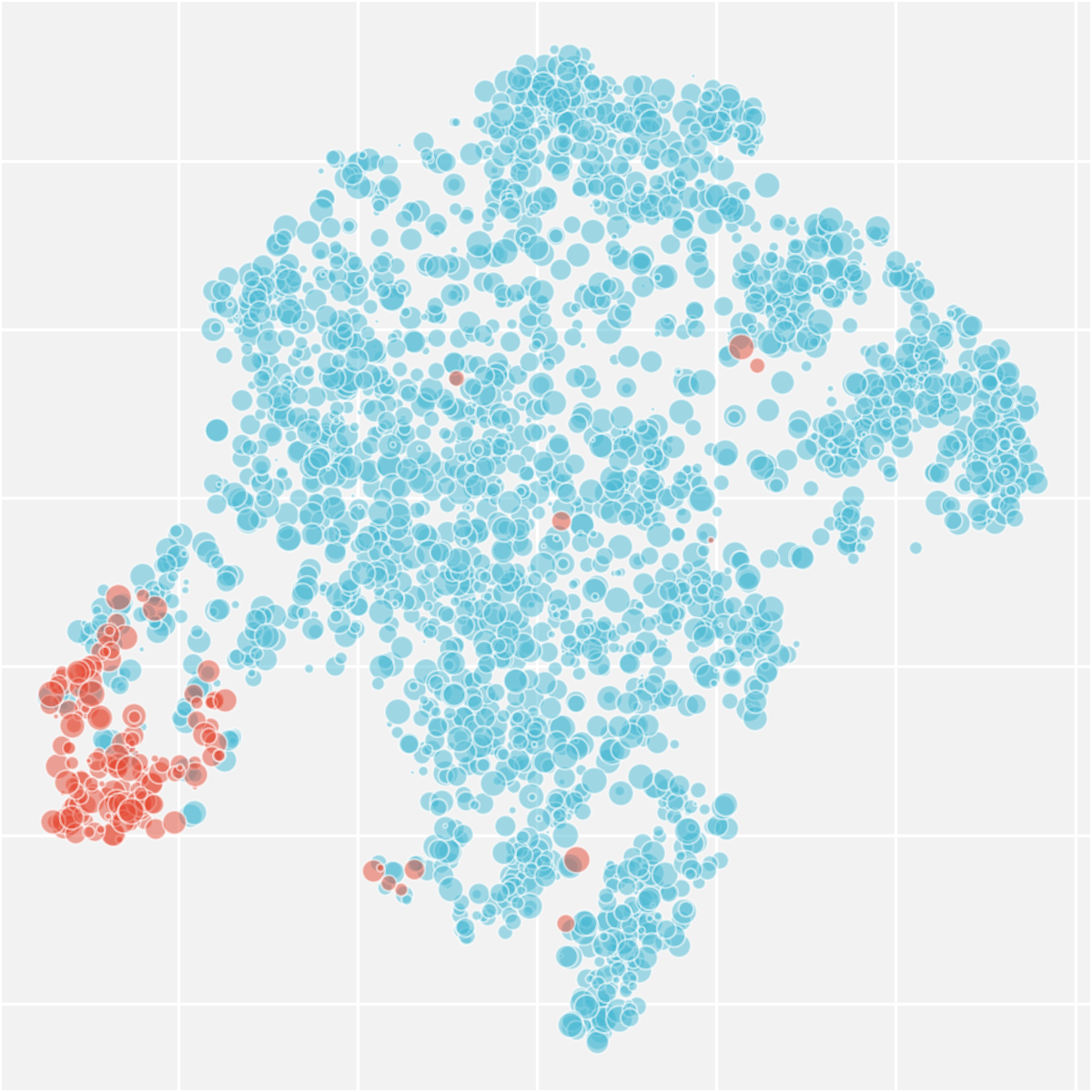}}
        \subfigure[AdAUC NT on the Data generated by PGD-10]{
        \includegraphics[width=0.45\columnwidth]{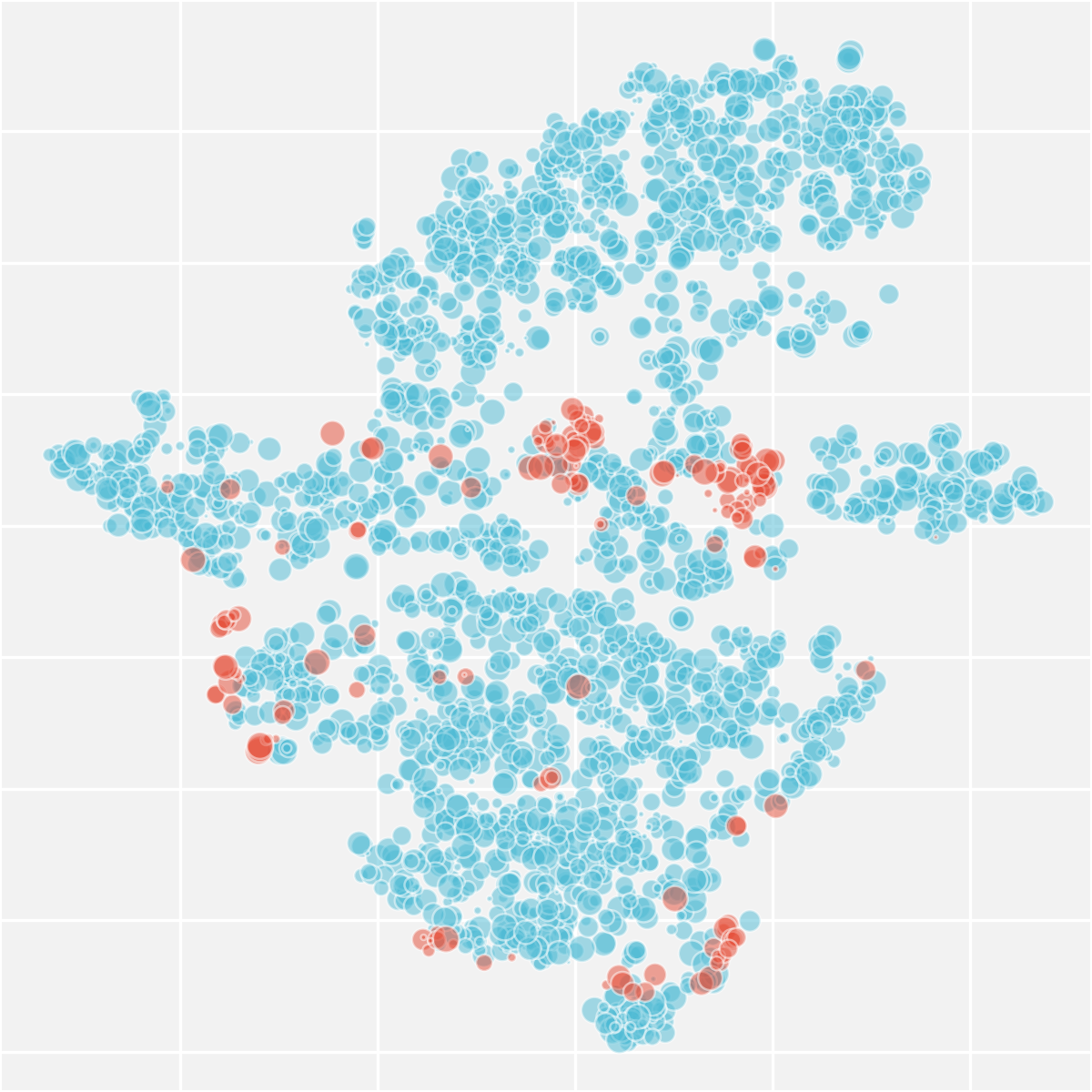}}
        \subfigure[AdAUC AT$_2$ on Clean Data]{
        \includegraphics[width=0.45\columnwidth]{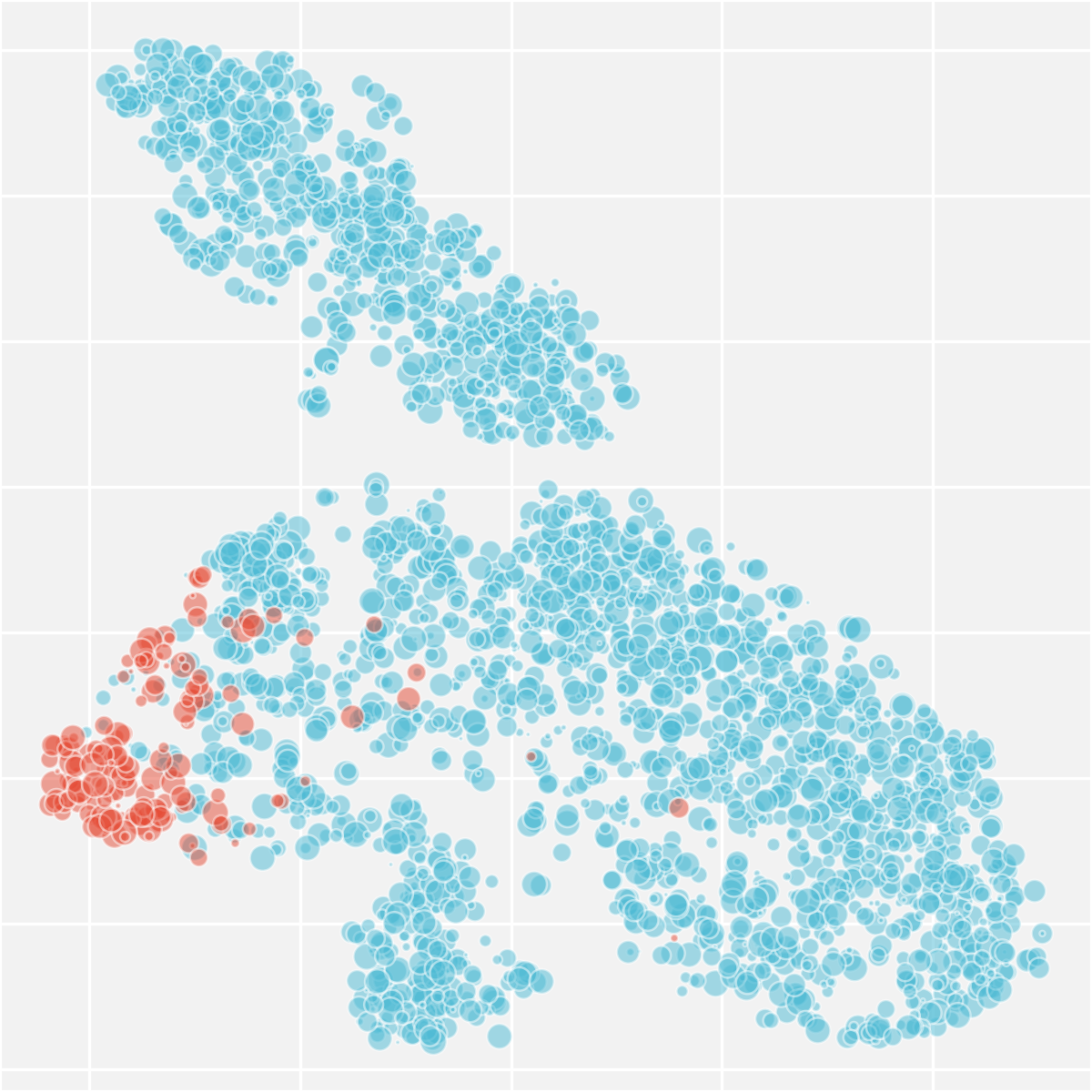}}
        \subfigure[AdAUC AT$_2$ on the Data generated by PGD-10]{
        \includegraphics[width=0.45\columnwidth]{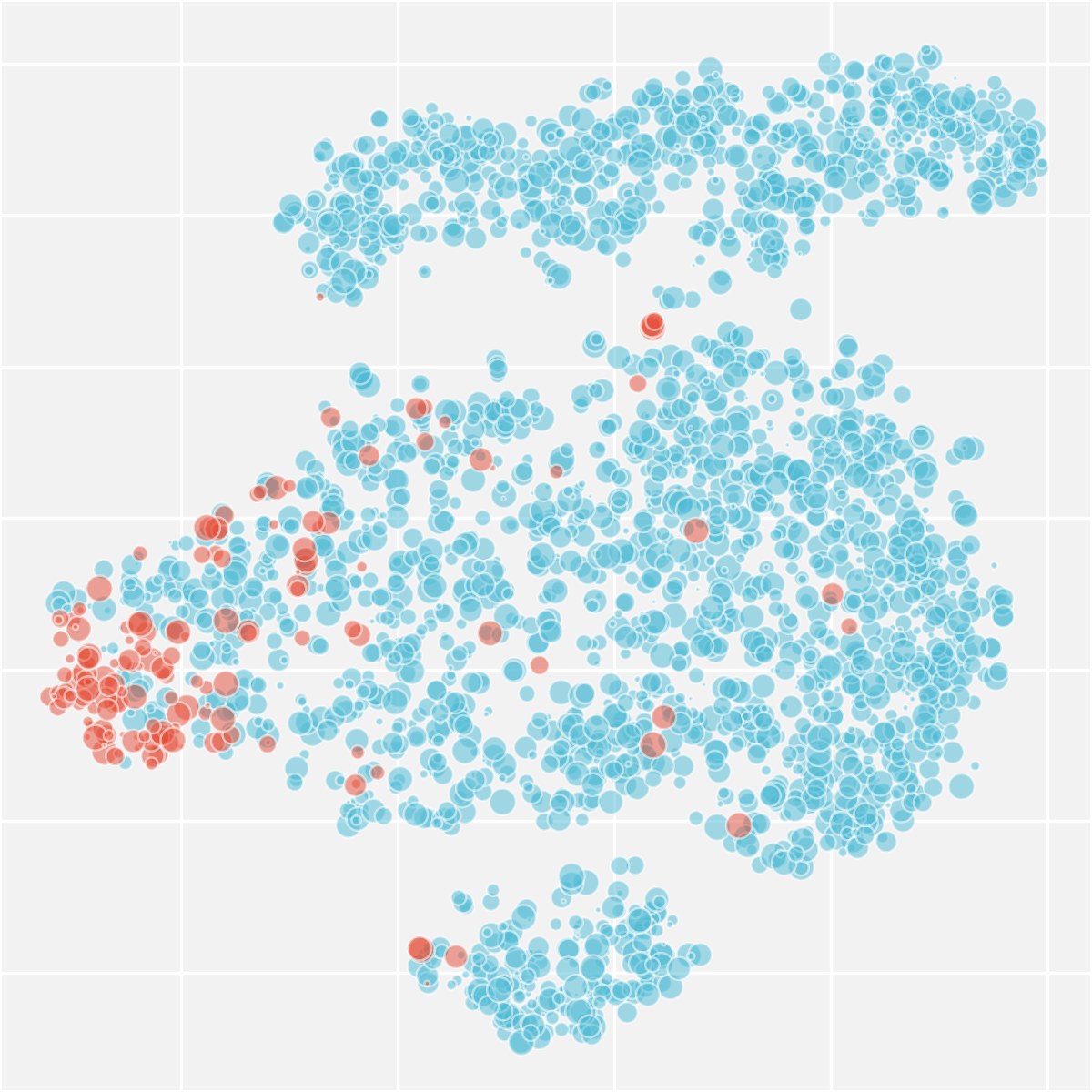}}
        
        \caption{The t-SNE projection of our methods on MNIST-LT dataset. 
        \textbf{Red points} represent the positive examples, and \textbf{blue points} represent the negative examples.}
        \label{Fig.tsne.MNIST.all}
    \end{figure}
    
    \begin{figure*}[ht]
        \centering 
        \subfigure[CIFAR-10-LT]{
        \includegraphics[width=0.28\textwidth]{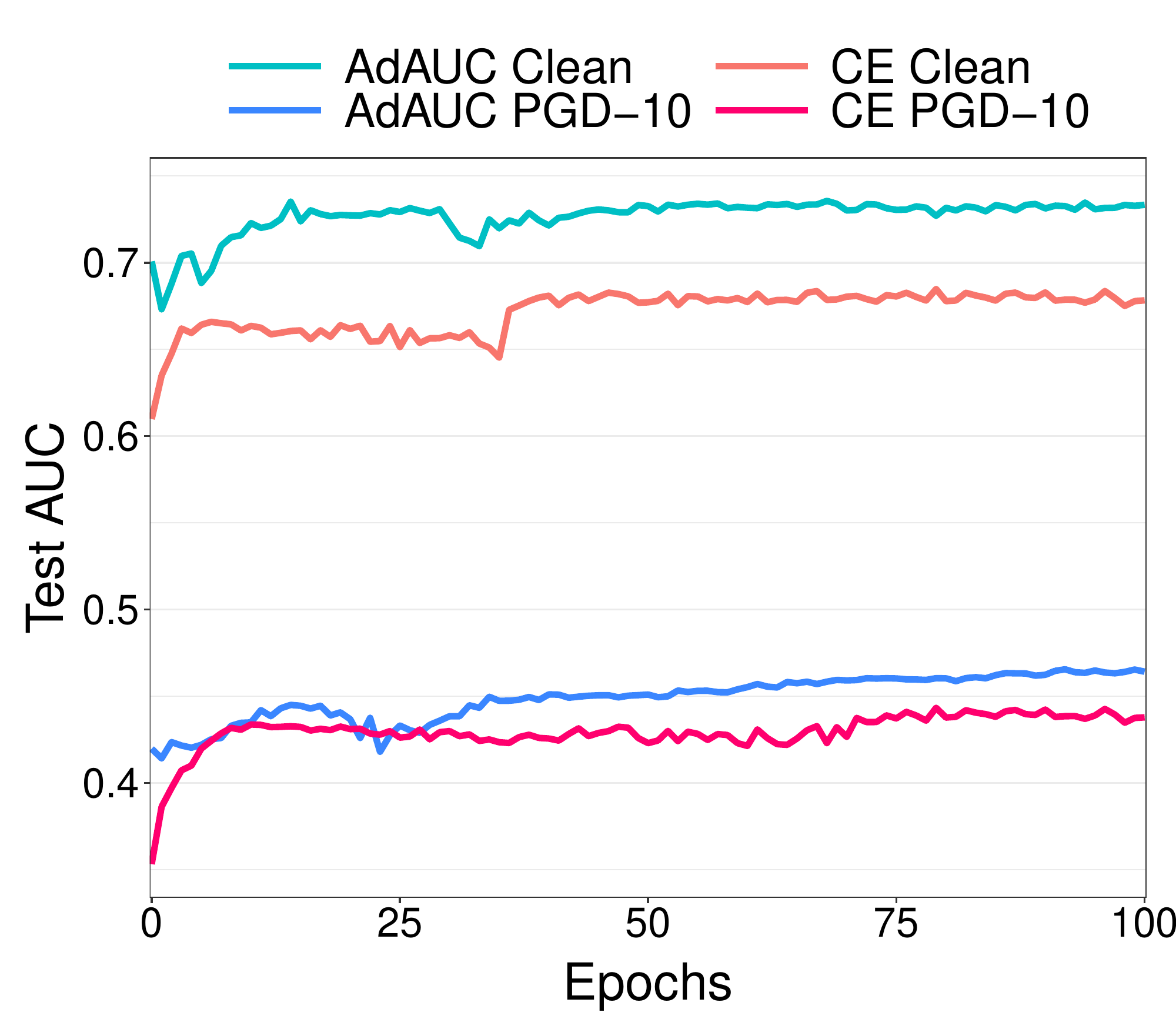}}
        \subfigure[CIFAR-100-LT]{
        \includegraphics[width=0.28\textwidth]{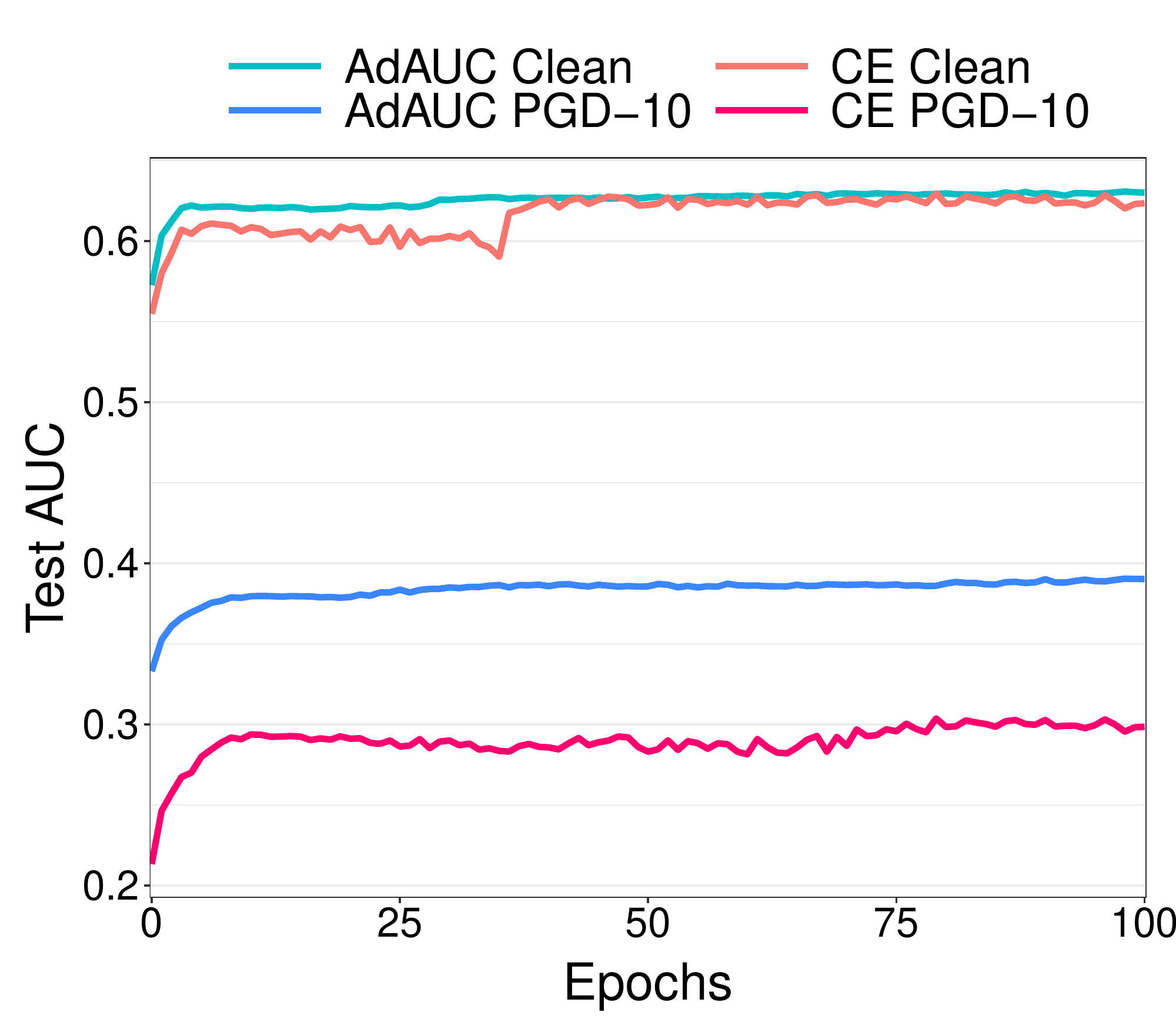}}
        \subfigure[MNIST-LT]{
        \includegraphics[width=0.28\textwidth]{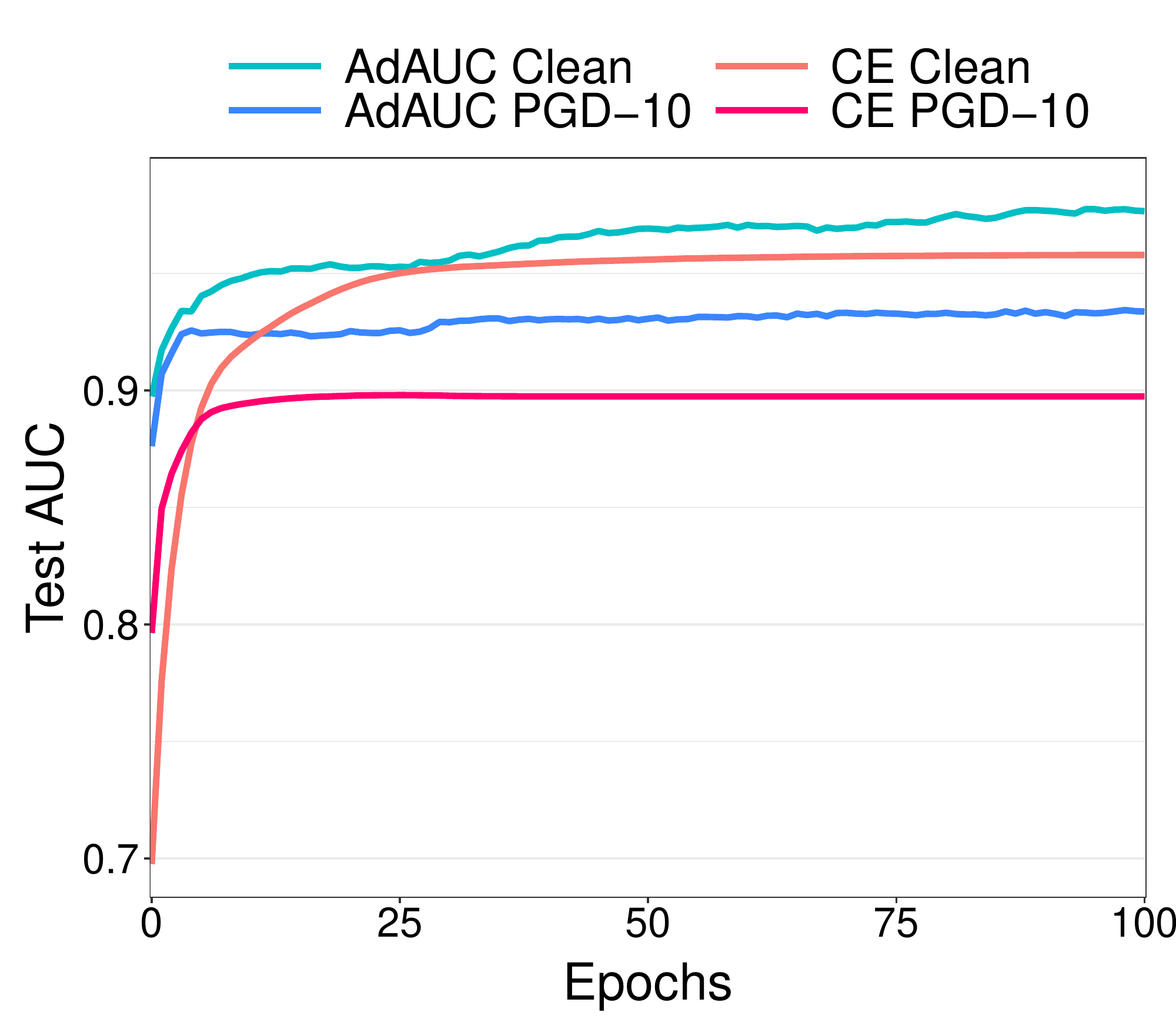}}
        
        \caption{The convergence of AUC on testing data of CIFAR-10-LT, CIFAR-100-LT and MNIST-LT.}
        \label{Fig.Convergence}
    \end{figure*}
	
	\begin{figure*}[ht]
        \centering 
        \subfigure[CE NT]{
        \includegraphics[width=0.24\textwidth]{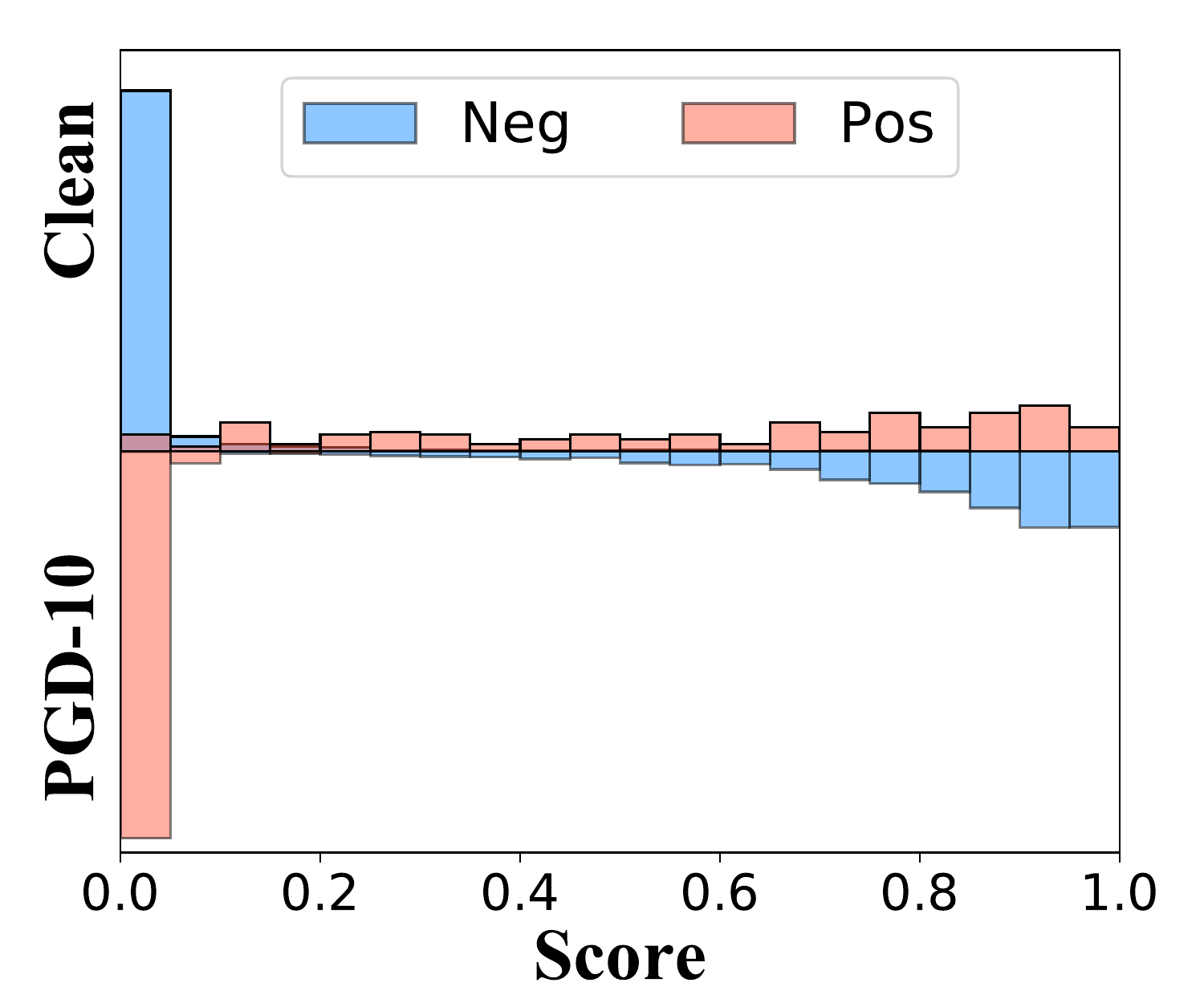}}
        \subfigure[CE AT$_2$]{
        \includegraphics[width=0.24\textwidth]{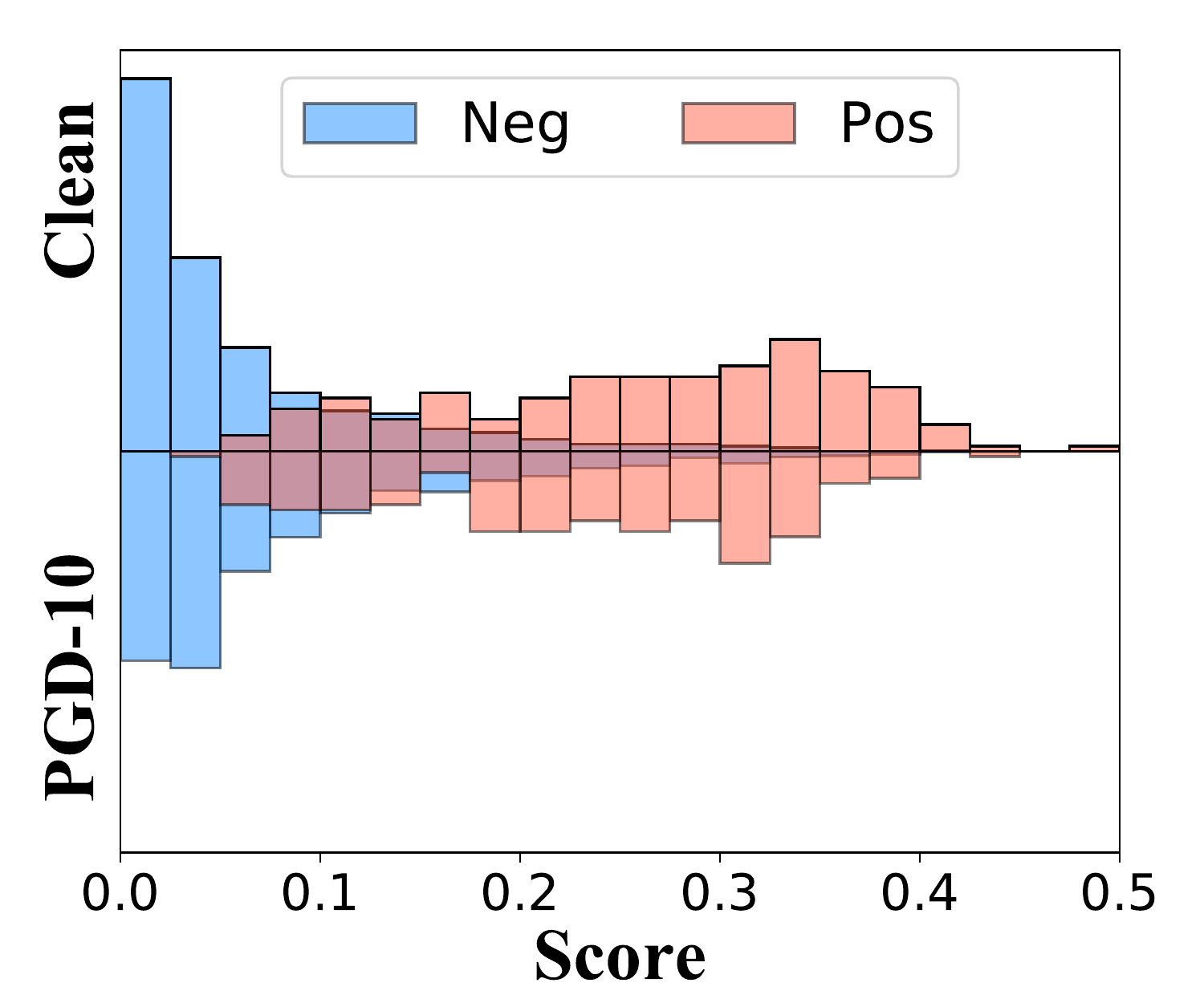}}
        \subfigure[AdAUC NT]{
        \includegraphics[width=0.24\textwidth]{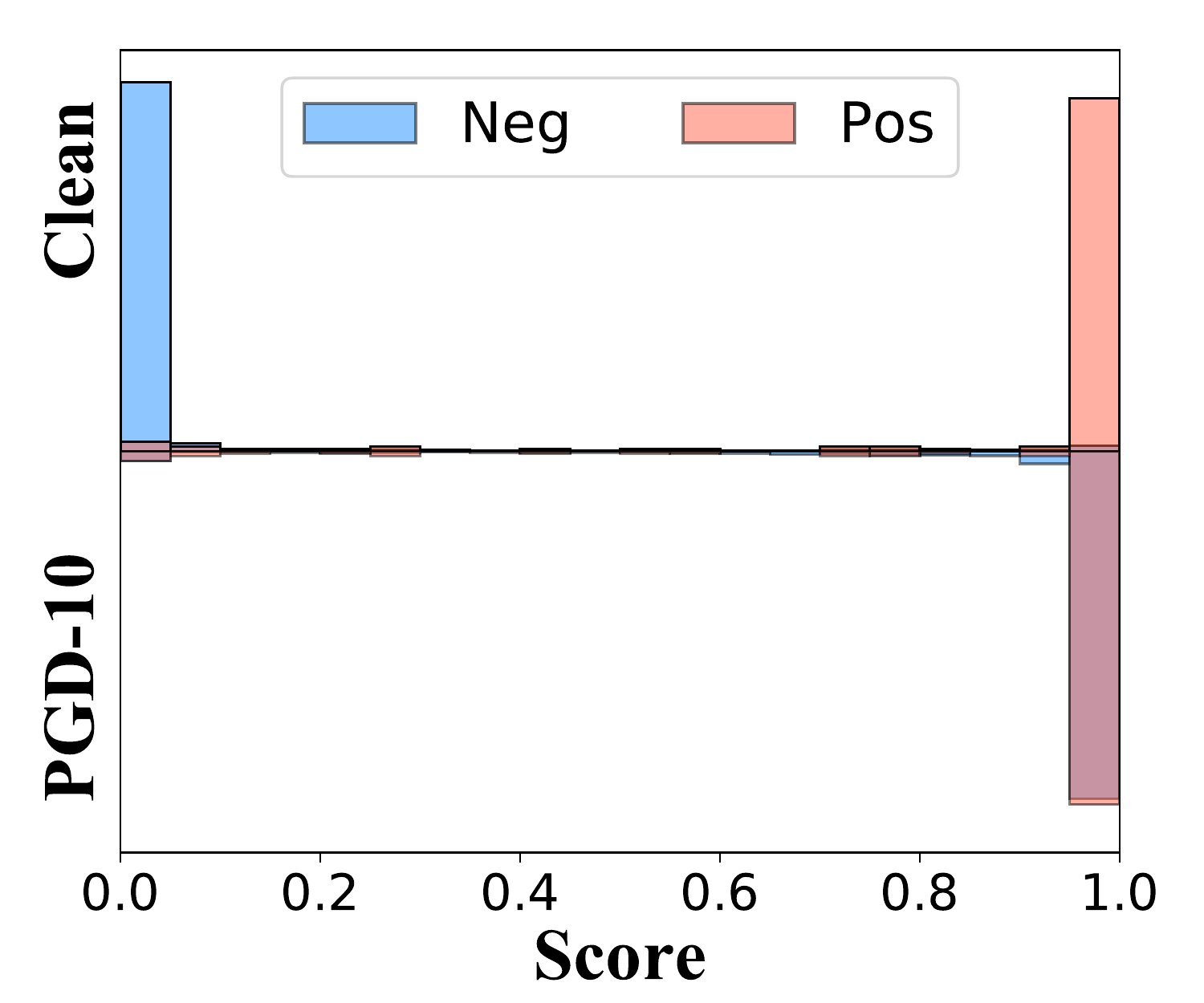}}
        \subfigure[AdAUC AT$_2$]{
        \includegraphics[width=0.24\textwidth]{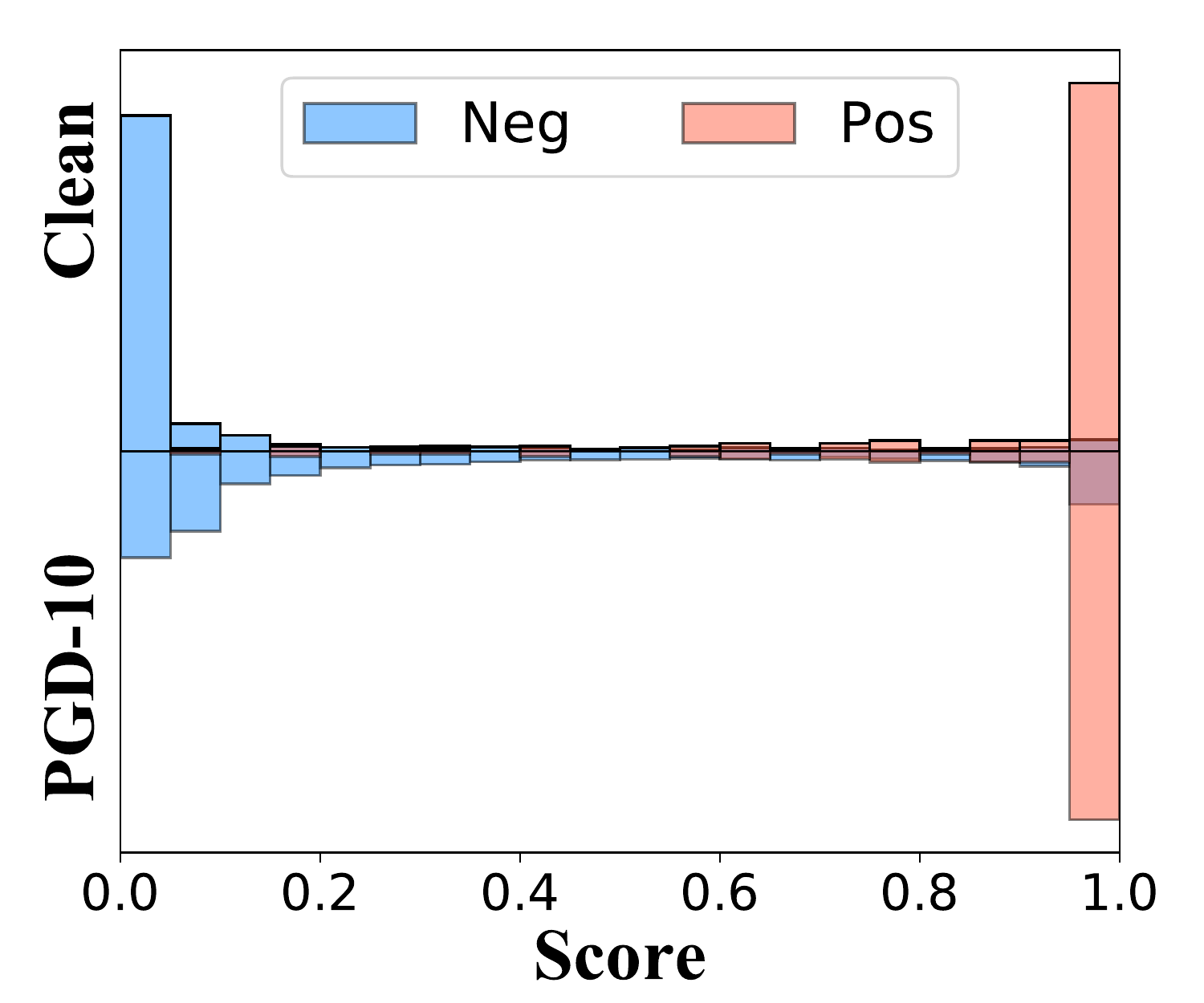}}
        
        \caption{Score distribution of positive and negative examples on MNIST-LT dataset. In each subfigure, the \textbf{above part} represents the score distribution evaluated against clean data, and the \textbf{below part} represents the score distribution evaluated against PGD-10.}
        \label{Fig.Distribution.MNIST.all}
    \end{figure*}
    
	\textbf{Binary CIFAR-100-LT Dataset.}
	We construct a long-tail CIFAR-100 dataset in the same way as CIFAR-10-LT, where  we label the first 50 classes as the negative class and the last 50 classes as positive, which leads that the ratio of  positive class size to negative class size \textbf{$\rho \approx 1:9$}.
	
	\textbf{Binary MNIST-LT Dataset.}
	We construct a long-tail MNIST dataset from original MNIST dataset \cite{lecun1998gradient} in the same way as CIFAR-10-LT, where the ratio of  positive class size to negative class size \textbf{$\rho \approx 1:9$}.

	\subsection{Overall Performance}
	The performance and robustness of all the involved methods on three datasets are shown in Tab.\ref{table_all}. 
	Consequently, we have the following observations: 
	\textbf{1)} On all the datasets, our methods achieve the best or competitive performance evaluated against all adversarial attack methods as shown in Tab.\ref{table_all}.
	\textbf{2)} Even the AUC optimization with NT has certain robustness (the AUC will not drop to 0 when evaluated against adversarial examples). 
	This is because the decision surface obtained by AUC optimization has a greater tolerance for minority classes than CE. 
	Specifically, the decision surface is far away from the positive examples. To validate this argument, we show the score distribution in Fig.\ref{Fig.Distribution.MNIST.all} w.r.t MNIST-LT. For AdAUC NT, adversarial examples increase the score of the negative examples, while it has less impact on the positive examples.
	However, the perturbation becomes much more violent. As shown in Fig.\ref{Fig.Distribution.MNIST.all}-(a), the adversarial examples simultaneously increase the score of the negative examples and decreases the score of the positive examples. The more results of other datasets are shown in the App.\ref{score_distribution}, which show a similar trend. 
	Moreover, we show the feature visualization of NT and AT$_2$ of our methods for clean data and adversarial examples. It implies that our AdAUC algorithm could separate the positive and negative instances well in the embedding space in Fig.\ref{Fig.tsne.MNIST.all}.

	\subsection{Convergence Analysis}    
	We report the convergence of test AUC of CE-based methods and our proposed methods in Fig. \ref{Fig.Convergence}.
	We can observe that our proposed method performs better both on clean data and adversarial examples. 
	However, due to the high complexity of the outer min-max problem, it could be seen that our method converges slightly slower than CE methods, which is consistent with the analysis of Thm.\ref{thm:minimax}.
	
	\section{Conclusion}
	In this paper, we initiate the study on adversarial AUC optimization against long-tail problem. The complexity of AUC loss function makes the corresponding adversarial training hardly scalable. To address this issue, we first construct a reformulation of the AT problem of AUC optimization. By further applying a concavity promoting regularizer, we can reformulate the original problem as a min-max problem where the objective function can be expressed instance-wisely. On top of the reformulation, we construct an end-to-end training algorithm with provable guarantee. Finally, we conduct a series of empirical studies on three long-tail benchmark datasets, the results of which demonstrate the effectiveness of our proposed method.

    \section*{Acknowledgements}
    This work was supported in part by the National Key R$\&$D Program of China under Grant 2018AAA0102000, in part by National Natural Science Foundation of China: U21B2038, 61931008, 6212200758 and 61976202, in part by the Fundamental Research Funds for the Central Universities, in part by Youth Innovation Promotion Association CAS, in part by the Strategic Priority Research Program of Chinese Academy of Sciences, Grant No. XDB28000000, and in part by the National Postdoctoral Program for Innovative Talents under Grant BX2021298.

	\bibliography{main}

\begin{thebibliography}{62}
\providecommand{\natexlab}[1]{#1}
\providecommand{\url}[1]{\texttt{#1}}
\expandafter\ifx\csname urlstyle\endcsname\relax
  \providecommand{\doi}[1]{doi: #1}\else
  \providecommand{\doi}{doi: \begingroup \urlstyle{rm}\Url}\fi

\bibitem[Agarwal(2014)]{agarwal2014surrogate}
Agarwal, S.
\newblock Surrogate regret bounds for bipartite ranking via strongly proper
  losses.
\newblock \emph{The Journal of Machine Learning Research}, 15\penalty0
  (1):\penalty0 1653--1674, 2014.

\bibitem[Agarwal et~al.(2005)Agarwal, Graepel, Herbrich, Har-Peled, Roth, and
  Jordan]{agarwal2005generalization}
Agarwal, S., Graepel, T., Herbrich, R., Har-Peled, S., Roth, D., and Jordan,
  M.~I.
\newblock Generalization bounds for the area under the roc curve.
\newblock \emph{Journal of Machine Learning Research}, 6\penalty0 (4), 2005.

\bibitem[Allen-Zhu et~al.(2019)Allen-Zhu, Li, and Song]{allen2019convergence}
Allen-Zhu, Z., Li, Y., and Song, Z.
\newblock A convergence theory for deep learning via over-parameterization.
\newblock In \emph{International Conference on Machine Learning}, pp.\
  242--252, 2019.

\bibitem[Athalye \& Carlini(2018)Athalye and Carlini]{athalye2018robustness}
Athalye, A. and Carlini, N.
\newblock On the robustness of the cvpr 2018 white-box adversarial example
  defenses.
\newblock \emph{arXiv preprint arXiv:1804.03286}, 2018.

\bibitem[Athalye et~al.(2018)Athalye, Carlini, and
  Wagner]{athalye2018obfuscated}
Athalye, A., Carlini, N., and Wagner, D.
\newblock Obfuscated gradients give a false sense of security: Circumventing
  defenses to adversarial examples.
\newblock In \emph{International conference on machine learning}, pp.\
  274--283. PMLR, 2018.

\bibitem[Bai et~al.(2022)Bai, Gautam, and
  Sojoudi]{DBLP:journals/corr/abs-2201-01965}
Bai, Y., Gautam, T., and Sojoudi, S.
\newblock Efficient global optimization of two-layer relu networks:
  Quadratic-time algorithms and adversarial training.
\newblock \emph{CoRR}, abs/2201.01965, 2022.

\bibitem[Bao et~al.(2020)Bao, Scott, and Sugiyama]{bao2020calibrated}
Bao, H., Scott, C., and Sugiyama, M.
\newblock Calibrated surrogate losses for adversarially robust classification.
\newblock In \emph{Conference on Learning Theory}, pp.\  408--451. PMLR, 2020.

\bibitem[Biggio et~al.(2013)Biggio, Corona, Maiorca, Nelson, {\v{S}}rndi{\'c},
  Laskov, Giacinto, and Roli]{biggio2013evasion}
Biggio, B., Corona, I., Maiorca, D., Nelson, B., {\v{S}}rndi{\'c}, N., Laskov,
  P., Giacinto, G., and Roli, F.
\newblock Evasion attacks against machine learning at test time.
\newblock In \emph{Joint European Conference on Machine Learning and Knowledge
  Discovery in Databases}, pp.\  387--402, 2013.

\bibitem[B{\"o}hm \& Wright(2021)B{\"o}hm and Wright]{bohm2021variable}
B{\"o}hm, A. and Wright, S.~J.
\newblock Variable smoothing for weakly convex composite functions.
\newblock \emph{Journal of optimization theory and applications}, 188\penalty0
  (3):\penalty0 628--649, 2021.

\bibitem[Cai et~al.(2018)Cai, Liu, and Song]{cai2018curriculum}
Cai, Q.-Z., Liu, C., and Song, D.
\newblock Curriculum adversarial training.
\newblock In \emph{Proceedings of the 27th International Joint Conference on
  Artificial Intelligence}, pp.\  3740--3747, 2018.

\bibitem[Carlini \& Wagner(2017)Carlini and Wagner]{carlini2017towards}
Carlini, N. and Wagner, D.
\newblock Towards evaluating the robustness of neural networks.
\newblock In \emph{IEEE Symposium on Security and Privacy}, pp.\  39--57, 2017.

\bibitem[Cl{\'e}men{\c{c}}on et~al.(2008)Cl{\'e}men{\c{c}}on, Lugosi, and
  Vayatis]{clemenccon2008ranking}
Cl{\'e}men{\c{c}}on, S., Lugosi, G., and Vayatis, N.
\newblock Ranking and empirical minimization of u-statistics.
\newblock \emph{The Annals of Statistics}, 36\penalty0 (2):\penalty0 844--874,
  2008.

\bibitem[Cortes \& Mohri(2003)Cortes and Mohri]{cortes2003auc}
Cortes, C. and Mohri, M.
\newblock Auc optimization vs. error rate minimization.
\newblock \emph{Advances in neural information processing systems},
  16:\penalty0 313--320, 2003.

\bibitem[Croce \& Hein(2020)Croce and Hein]{croce2020reliable}
Croce, F. and Hein, M.
\newblock Reliable evaluation of adversarial robustness with an ensemble of
  diverse parameter-free attacks.
\newblock In \emph{International Conference on Machine Learning}, pp.\
  2206--2216, 2020.

\bibitem[Du et~al.(2019)Du, Lee, Li, Wang, and Zhai]{du2019gradient}
Du, S., Lee, J., Li, H., Wang, L., and Zhai, X.
\newblock Gradient descent finds global minima of deep neural networks.
\newblock In \emph{International Conference on Machine Learning}, pp.\
  1675--1685, 2019.

\bibitem[Fawcett(2006)]{fawcett2006introduction}
Fawcett, T.
\newblock An introduction to roc analysis.
\newblock \emph{Pattern Recognition Letters}, 27\penalty0 (8):\penalty0
  861--874, 2006.

\bibitem[Feizi(2020)]{feizi2020hierarchical}
Feizi, A.
\newblock Hierarchical detection of abnormal behaviors in video surveillance
  through modeling normal behaviors based on auc maximization.
\newblock \emph{Soft Computing}, 24\penalty0 (14):\penalty0 10401--10413, 2020.

\bibitem[Gao \& Zhou(2015)Gao and Zhou]{gao2015consistency}
Gao, W. and Zhou, Z.-H.
\newblock On the consistency of auc pairwise optimization.
\newblock In \emph{Twenty-Fourth International Joint Conference on Artificial
  Intelligence}, 2015.

\bibitem[Gao et~al.(2013)Gao, Jin, Zhu, and Zhou]{gao2013one}
Gao, W., Jin, R., Zhu, S., and Zhou, Z.-H.
\newblock One-pass auc optimization.
\newblock In \emph{International conference on machine learning}, pp.\
  906--914, 2013.

\bibitem[Gola et~al.(2020)Gola, Erdmann, M{\"u}ller-Myhsok, Schunkert, and
  K{\"o}nig]{gola2020polygenic}
Gola, D., Erdmann, J., M{\"u}ller-Myhsok, B., Schunkert, H., and K{\"o}nig,
  I.~R.
\newblock Polygenic risk scores outperform machine learning methods in
  predicting coronary artery disease status.
\newblock \emph{Genetic epidemiology}, 44\penalty0 (2):\penalty0 125--138,
  2020.

\bibitem[Goodfellow et~al.(2014)Goodfellow, Pouget-Abadie, Mirza, Xu,
  Warde-Farley, Ozair, Courville, and Bengio]{goodfellow2014generative}
Goodfellow, I.~J., Pouget-Abadie, J., Mirza, M., Xu, B., Warde-Farley, D.,
  Ozair, S., Courville, A., and Bengio, Y.
\newblock Generative adversarial nets.
\newblock In \emph{Proceedings of the 27th International Conference on Neural
  Information Processing Systems}, pp.\  2672--2680, 2014.

\bibitem[Goodfellow et~al.(2015)Goodfellow, Shlens, and
  Szegedy]{goodfellow2015explaining}
Goodfellow, I.~J., Shlens, J., and Szegedy, C.
\newblock Explaining and harnessing adversarial examples.
\newblock In \emph{International Conference on Learning Representations}, 2015.

\bibitem[Hand \& Till(2001)Hand and Till]{hand2001simple}
Hand, D.~J. and Till, R.~J.
\newblock A simple generalisation of the area under the roc curve for multiple
  class classification problems.
\newblock \emph{Machine Learning}, 45\penalty0 (2):\penalty0 171--186, 2001.

\bibitem[Hanley \& McNeil(1982)Hanley and McNeil]{hanley1982meaning}
Hanley, J.~A. and McNeil, B.~J.
\newblock The meaning and use of the area under a receiver operating
  characteristic (roc) curve.
\newblock \emph{Radiology}, 143\penalty0 (1):\penalty0 29--36, 1982.

\bibitem[Herschtal \& Raskutti(2004)Herschtal and
  Raskutti]{herschtal2004optimising}
Herschtal, A. and Raskutti, B.
\newblock Optimising area under the roc curve using gradient descent.
\newblock In \emph{Proceedings of the twenty-first international conference on
  Machine learning}, pp.\ ~49, 2004.

\bibitem[Joachims(2005)]{joachims2005support}
Joachims, T.
\newblock A support vector method for multivariate performance measures.
\newblock In \emph{Proceedings of the 22nd international conference on Machine
  learning}, pp.\  377--384, 2005.

\bibitem[Kurakin et~al.(2017)Kurakin, Goodfellow, and
  Bengio]{kurakin2016adversarial}
Kurakin, A., Goodfellow, I., and Bengio, S.
\newblock Adversarial machine learning at scale.
\newblock In \emph{International Conference on Learning Representations}, 2017.

\bibitem[LeCun et~al.(1998)LeCun, Bottou, Bengio, and
  Haffner]{lecun1998gradient}
LeCun, Y., Bottou, L., Bengio, Y., and Haffner, P.
\newblock Gradient-based learning applied to document recognition.
\newblock \emph{Proceedings of the IEEE}, 86\penalty0 (11):\penalty0
  2278--2324, 1998.

\bibitem[Lin et~al.(2020)Lin, Jin, and Jordan]{lin2020gradient}
Lin, T., Jin, C., and Jordan, M.
\newblock On gradient descent ascent for nonconvex-concave minimax problems.
\newblock In \emph{International Conference on Machine Learning}, pp.\
  6083--6093, 2020.

\bibitem[Liu et~al.(2019)Liu, Yuan, Ying, and Yang]{liu2019stochastic}
Liu, M., Yuan, Z., Ying, Y., and Yang, T.
\newblock Stochastic auc maximization with deep neural networks.
\newblock In \emph{International Conference on Learning Representations}, 2019.

\bibitem[Liu et~al.(2021)Liu, Rafique, Lin, and Yang]{liu2021first}
Liu, M., Rafique, H., Lin, Q., and Yang, T.
\newblock First-order convergence theory for weakly-convex-weakly-concave
  min-max problems.
\newblock \emph{Journal of Machine Learning Research}, 22\penalty0
  (169):\penalty0 1--34, 2021.

\bibitem[Madry et~al.(2018)Madry, Makelov, Schmidt, Tsipras, and
  Vladu]{madry2018towards}
Madry, A., Makelov, A., Schmidt, L., Tsipras, D., and Vladu, A.
\newblock Towards deep learning models resistant to adversarial attacks.
\newblock In \emph{International Conference on Learning Representations}, 2018.

\bibitem[Maini et~al.(2020)Maini, Wong, and Kolter]{maini2020adversarial}
Maini, P., Wong, E., and Kolter, Z.
\newblock Adversarial robustness against the union of multiple perturbation
  models.
\newblock In \emph{International Conference on Machine Learning}, pp.\
  6640--6650. PMLR, 2020.

\bibitem[Natole et~al.(2018)Natole, Ying, and Lyu]{natole2018stochastic}
Natole, M., Ying, Y., and Lyu, S.
\newblock Stochastic proximal algorithms for auc maximization.
\newblock In \emph{International Conference on Machine Learning}, pp.\
  3710--3719. PMLR, 2018.

\bibitem[Nesterov(1998)]{nesterov1998introductory}
Nesterov, Y.
\newblock Introductory lectures on convex programming volume i: Basic course.
\newblock \emph{Lecture notes}, 3\penalty0 (4):\penalty0 5, 1998.

\bibitem[Neumann(1928)]{neumann1928theorie}
Neumann, J.~v.
\newblock Zur theorie der gesellschaftsspiele.
\newblock \emph{Mathematische Annalen}, 100\penalty0 (1):\penalty0 295--320,
  1928.

\bibitem[Pang et~al.(2019)Pang, Xu, Du, Chen, and Zhu]{pang2019improving}
Pang, T., Xu, K., Du, C., Chen, N., and Zhu, J.
\newblock Improving adversarial robustness via promoting ensemble diversity.
\newblock In \emph{International Conference on Machine Learning}, pp.\
  4970--4979. PMLR, 2019.

\bibitem[Ren et~al.(2018)Ren, Yang, Zhao, Chen, Xue, Miao, Huang, and
  Liu]{ren2018robust}
Ren, K., Yang, H., Zhao, Y., Chen, W., Xue, M., Miao, H., Huang, S., and Liu,
  J.
\newblock A robust auc maximization framework with simultaneous outlier
  detection and feature selection for positive-unlabeled classification.
\newblock \emph{IEEE transactions on neural networks and learning systems},
  30\penalty0 (10):\penalty0 3072--3083, 2018.

\bibitem[Robles et~al.(2020)Robles, Zaidouni, Mavromoustaki, and
  Refael]{robles2020threshold}
Robles, E., Zaidouni, F., Mavromoustaki, A., and Refael, P.
\newblock Threshold optimization in multiple binary classifiers for extreme
  rare events using predicted positive data.
\newblock In \emph{AAAI Spring Symposium: Combining Machine Learning with
  Knowledge Engineering (1)}, 2020.

\bibitem[Shafahi et~al.(2019)Shafahi, Najibi, Ghiasi, Xu, Dickerson, Studer,
  Davis, Taylor, and Goldstein]{shafahi2019adversarial}
Shafahi, A., Najibi, M., Ghiasi, A., Xu, Z., Dickerson, J., Studer, C., Davis,
  L.~S., Taylor, G., and Goldstein, T.
\newblock Adversarial training for free!
\newblock In \emph{Proceedings of the 33rd International Conference on Neural
  Information Processing Systems}, pp.\  3358--3369, 2019.

\bibitem[Shafahi et~al.(2020)Shafahi, Najibi, Xu, Dickerson, Davis, and
  Goldstein]{shafahi2020universal}
Shafahi, A., Najibi, M., Xu, Z., Dickerson, J., Davis, L.~S., and Goldstein, T.
\newblock Universal adversarial training.
\newblock In \emph{Proceedings of the AAAI Conference on Artificial
  Intelligence}, pp.\  5636--5643, 2020.

\bibitem[Sinha et~al.(2018)Sinha, Namkoong, and Duchi]{sinha2018certifying}
Sinha, A., Namkoong, H., and Duchi, J.
\newblock Certifying some distributional robustness with principled adversarial
  training.
\newblock In \emph{International Conference on Learning Representations}, 2018.

\bibitem[Sion(1958)]{sion1958general}
Sion, M.
\newblock On general minimax theorems.
\newblock \emph{Pacific Journal of Mathematics}, 8\penalty0 (1):\penalty0
  171--176, 1958.

\bibitem[Sorin et~al.(2020)Sorin, Barash, Konen, and Klang]{sorin2020deep}
Sorin, V., Barash, Y., Konen, E., and Klang, E.
\newblock Deep learning for natural language processing in
  radiology—fundamentals and a systematic review.
\newblock \emph{Journal of the American College of Radiology}, 17\penalty0
  (5):\penalty0 639--648, 2020.

\bibitem[Strubell et~al.(2019)Strubell, Ganesh, and
  McCallum]{strubell2019energy}
Strubell, E., Ganesh, A., and McCallum, A.
\newblock Energy and policy considerations for deep learning in nlp.
\newblock In \emph{Proceedings of the 57th Annual Meeting of the Association
  for Computational Linguistics}, pp.\  3645--3650, 2019.

\bibitem[Szegedy et~al.(2014)Szegedy, Zaremba, Sutskever, Bruna, Erhan,
  Goodfellow, and Fergus]{szegedy2014intriguing}
Szegedy, C., Zaremba, W., Sutskever, I., Bruna, J., Erhan, D., Goodfellow, I.,
  and Fergus, R.
\newblock Intriguing properties of neural networks.
\newblock In \emph{International Conference on Learning Representations}, 2014.

\bibitem[Tram{\`e}r \& Boneh(2019)Tram{\`e}r and Boneh]{tramer2019adversarial}
Tram{\`e}r, F. and Boneh, D.
\newblock Adversarial training and robustness for multiple perturbations.
\newblock In \emph{Proceedings of the 33rd International Conference on Neural
  Information Processing Systems}, pp.\  5866--5876, 2019.

\bibitem[Tram{\`e}r et~al.(2017)Tram{\`e}r, Kurakin, Papernot, Boneh, and
  McDaniel]{tramer2017ensemble}
Tram{\`e}r, F., Kurakin, A., Papernot, N., Boneh, D., and McDaniel, P.
\newblock Ensemble adversarial training: Attacks and defenses.
\newblock \emph{stat}, 1050:\penalty0 30, 2017.

\bibitem[Tu et~al.(2019)Tu, Zhang, and Tao]{NEURIPS2019_16bda725}
Tu, Z., Zhang, J., and Tao, D.
\newblock Theoretical analysis of adversarial learning: A minimax approach.
\newblock In Wallach, H., Larochelle, H., Beygelzimer, A., d\textquotesingle
  Alch\'{e}-Buc, F., Fox, E., and Garnett, R. (eds.), \emph{Advances in Neural
  Information Processing Systems}, volume~32. Curran Associates, Inc., 2019.

\bibitem[Usunier et~al.(2005)Usunier, Amini, and Gallinari]{usunier2005data}
Usunier, N., Amini, M.-R., and Gallinari, P.
\newblock A data-dependent generalisation error bound for the auc.
\newblock In \emph{Proceedings of the ICML 2005 Workshop on ROC Analysis in
  Machine Learning}. Citeseer, 2005.

\bibitem[Voulodimos et~al.(2018)Voulodimos, Doulamis, Doulamis, and
  Protopapadakis]{voulodimos2018deep}
Voulodimos, A., Doulamis, N., Doulamis, A., and Protopapadakis, E.
\newblock Deep learning for computer vision: A brief review.
\newblock \emph{Computational intelligence and neuroscience}, 2018, 2018.

\bibitem[Wang et~al.(2019{\natexlab{a}})Wang, Ma, Bailey, Yi, Zhou, and
  Gu]{wang2019convergence}
Wang, Y., Ma, X., Bailey, J., Yi, J., Zhou, B., and Gu, Q.
\newblock On the convergence and robustness of adversarial training.
\newblock In \emph{International Conference on Machine Learning}, pp.\
  6586--6595, 2019{\natexlab{a}}.

\bibitem[Wang et~al.(2019{\natexlab{b}})Wang, Zou, Yi, Bailey, Ma, and
  Gu]{wang2019improving}
Wang, Y., Zou, D., Yi, J., Bailey, J., Ma, X., and Gu, Q.
\newblock Improving adversarial robustness requires revisiting misclassified
  examples.
\newblock In \emph{International Conference on Learning Representations},
  2019{\natexlab{b}}.

\bibitem[Westcott et~al.(2019)Westcott, Capaldi, McCormack, Ward, Fenster, and
  Parraga]{westcott2019chronic}
Westcott, A., Capaldi, D.~P., McCormack, D.~G., Ward, A.~D., Fenster, A., and
  Parraga, G.
\newblock Chronic obstructive pulmonary disease: thoracic ct texture analysis
  and machine learning to predict pulmonary ventilation.
\newblock \emph{Radiology}, 293\penalty0 (3):\penalty0 676--684, 2019.

\bibitem[Wong et~al.(2019)Wong, Rice, and Kolter]{wong2019fast}
Wong, E., Rice, L., and Kolter, J.~Z.
\newblock Fast is better than free: Revisiting adversarial training.
\newblock In \emph{International Conference on Learning Representations}, 2019.

\bibitem[Xing et~al.(2021)Xing, Song, and Cheng]{xing2021generalization}
Xing, Y., Song, Q., and Cheng, G.
\newblock On the generalization properties of adversarial training.
\newblock In \emph{International Conference on Artificial Intelligence and
  Statistics}, pp.\  505--513, 2021.

\bibitem[Ying et~al.(2016)Ying, Wen, and Lyu]{ying2016stochastic}
Ying, Y., Wen, L., and Lyu, S.
\newblock Stochastic online auc maximization.
\newblock \emph{Advances in Neural Information Processing Systems},
  29:\penalty0 451--459, 2016.

\bibitem[Zagoruyko \& Komodakis(2016)Zagoruyko and
  Komodakis]{zagoruyko2016wide}
Zagoruyko, S. and Komodakis, N.
\newblock Wide residual networks.
\newblock In \emph{British Machine Vision Conference}, 2016.

\bibitem[Zhang et~al.(2019{\natexlab{a}})Zhang, Zhang, Lu, Zhu, and
  Dong]{zhang2019you}
Zhang, D., Zhang, T., Lu, Y., Zhu, Z., and Dong, B.
\newblock You only propagate once: Accelerating adversarial training via
  maximal principle.
\newblock \emph{Advances in Neural Information Processing Systems},
  32:\penalty0 227--238, 2019{\natexlab{a}}.

\bibitem[Zhang et~al.(2019{\natexlab{b}})Zhang, Yu, Jiao, Xing, El~Ghaoui, and
  Jordan]{zhang2019theoretically}
Zhang, H., Yu, Y., Jiao, J., Xing, E., El~Ghaoui, L., and Jordan, M.
\newblock Theoretically principled trade-off between robustness and accuracy.
\newblock In \emph{International Conference on Machine Learning}, pp.\
  7472--7482. PMLR, 2019{\natexlab{b}}.

\bibitem[Zhang et~al.(2020)Zhang, Zhu, Niu, Han, Sugiyama, and
  Kankanhalli]{zhang2020geometry}
Zhang, J., Zhu, J., Niu, G., Han, B., Sugiyama, M., and Kankanhalli, M.
\newblock Geometry-aware instance-reweighted adversarial training.
\newblock In \emph{International Conference on Learning Representations}, 2020.

\bibitem[Zhao et~al.(2011)Zhao, Hoi, Jin, and Yang]{zhao2011online}
Zhao, P., Hoi, S.~C., Jin, R., and Yang, T.
\newblock Online auc maximization.
\newblock In \emph{Proceedings of the 28th International Conference on
  International Conference on Machine Learning}, pp.\  233--240, 2011.

\end{thebibliography}
	\bibliographystyle{icml2022}
	
	\clearpage
	\onecolumn
	\appendix
	\section{Proofs of Main Results}
	In this section, we provide the proofs of the main results.
	\subsection{Proof of Proposition 2}
\begin{proof}
	According to \cite{ying2016stochastic}, $a,b,\alpha$ have the following closed-form solution:
	\begin{align*}
	a=\hat{\mathbb{E}}[h_{\theta}(\boldsymbol{x})|y=1],~~ b=\hat{\mathbb{E}}[h_{\theta}(\boldsymbol{x})|y=0],~~ \alpha=\hat{\mathbb{E}}[h_{\theta}(\boldsymbol{x})|y=0]-\hat{\mathbb{E}}[h_{\theta}(\boldsymbol{x})|y=1].
	\end{align*}
Then it is easy to check that 
\begin{align*}
	r(a, b) = \max_{\alpha} \frac{1}{n} \sum^{n}_{i=1}g(\boldsymbol{\theta}, a, b, \alpha, (\boldsymbol{x}_i+\boldsymbol{\delta}_i, y_i))
\end{align*}
is a strongly-convex problem w.r.t. $(a,b)$. However, $r$ is in general not concave w.r.t. $\bm{\delta}_i$. In this sense, we then try to find a surrogate objective to induce the concavity w.r.t $\bm{\delta}_i$. Specifically, we adopt a concavity regularization term $-\gamma || \bm{x}_i +\bm{\delta}_i||^2_2$, and define a surrogate objective:
$$f(\boldsymbol{w}, \alpha, \boldsymbol{x}_i+\boldsymbol{\delta}_i) = g(\boldsymbol{\theta}, a, b, \alpha, (\boldsymbol{x}_i+\boldsymbol{\delta}_i, y_i)) - \gamma\left\| \boldsymbol{x}_i + \boldsymbol{\delta}_i \right\|^2_2.$$
Therefore, if $r$ is $\gamma_\star$-weakly concave w.r.t. $\bm{\delta}_i, i=1,2,\cdots,n$ \cite{liu2021first, bohm2021variable}, then we can define $\gamma  > \gamma_\star$ to obtain an objective $f(\boldsymbol{w}, \alpha, \boldsymbol{x}+\boldsymbol{\delta})$ such that $\max_{\alpha} \frac{1}{n} \sum^{n}_{i=1} f(\boldsymbol{w}, \alpha, \boldsymbol{x}_i+\boldsymbol{\delta}_i)$ is strongly concave w.r.t. $\bm{\delta}_i, i=1,2,\cdots,n$. Above all, with the weakly concavity assumption, we can instead solve the following surrogate problem by the von Neumann's minimax theorem:
\begin{equation}
		\begin{aligned}
		\op ~~	&\min\limits_{\boldsymbol{w}} \max\limits_{\alpha} \max\limits_{\boldsymbol{\delta}} 
			\frac{1}{n} \sum^{n}_{i=1}\left[f(\boldsymbol{w}, \alpha, \boldsymbol{x}_i+\boldsymbol{\delta}_i)\right] \\
			& = \min\limits_{\boldsymbol{w}} \max\limits_{\alpha} \frac{1}{n} \sum^{n}_{i=1} \max\limits_{\boldsymbol{\delta}_i} 
			\left[f(\boldsymbol{w}, \alpha, \boldsymbol{x}_i+\boldsymbol{\delta}_i)\right],
		\end{aligned}
	\end{equation}
	where $\boldsymbol{w}=(\boldsymbol{\theta}, a, b)$.
\end{proof}

	\subsection{Proof of Lemma 1}
	\label{Lemma1_sec}
	\begin{lem}
		\label{Lemma1}
		For all $x \in \mathcal{X}$, $c(\boldsymbol{x}^k)=0$ when 1) $\nabla_{\boldsymbol{x}} f(\boldsymbol{w}, \alpha, \boldsymbol{x}^k) = 0$, or 2) $\boldsymbol{x}^k -\boldsymbol{x}^0 = \epsilon \cdot \operatorname{sign}(\nabla_{\boldsymbol{x}} f(\boldsymbol{w}, \alpha, \boldsymbol{x}^k))$.
		
		\begin{proof}
			\begin{equation}
				\begin{aligned}
					c(\boldsymbol{x}^k) &= \max\limits_{\boldsymbol{x} \in \mathcal{X}} \langle\boldsymbol{x}-\boldsymbol{x}^k, \nabla_{\boldsymbol{x}} f(\boldsymbol{w}, \alpha, \boldsymbol{x}^k)\rangle \\
					&= \max\limits_{\boldsymbol{x} \in \mathcal{X}} \langle\boldsymbol{x}-\boldsymbol{x}^0 + \boldsymbol{x}^0 -\boldsymbol{x}^k, \nabla_{\boldsymbol{x}} f(\boldsymbol{w}, \alpha, \boldsymbol{x}^k)\rangle \\
					&= \max\limits_{\boldsymbol{x} \in \mathcal{X}} \langle\boldsymbol{x}-\boldsymbol{x}^0, \nabla_{\boldsymbol{x}} f(\boldsymbol{w}, \alpha, \boldsymbol{x}^k)\rangle - \langle\boldsymbol{x}^k -\boldsymbol{x}^0, \nabla_{\boldsymbol{x}} f(\boldsymbol{w}, \alpha, \boldsymbol{x}^k)\rangle \\
					&= \epsilon \cdot \left\| \nabla_{\boldsymbol{x}} f(\boldsymbol{w}, \alpha, \boldsymbol{x}^k) \right\|_1 - \langle\boldsymbol{x}^k -\boldsymbol{x}^0, \nabla_{\boldsymbol{x}} f(\boldsymbol{w}, \alpha, \boldsymbol{x}^k)\rangle
				\end{aligned}
			\end{equation}
			This completes the proof.
		\end{proof}
		
	\end{lem}

	\subsection{Proof of Technical Lemmas}
	In this subsections, we present seven key lemmas which are important for the proof of Theorem \ref{thm:minimax}. 
	
	\begin{lem}
	\label{lem2}
		Under the Asm.\ref{asm1} and \ref{asm3}, we have $L(\alpha)$ is $L_\alpha$-smooth where $L_\alpha=\frac{L_{\alpha x}L_{x \alpha}}{\mu} + L_{\alpha \alpha}$. For any $\alpha_1, \alpha_2$, it holds
	\begin{equation}
		\begin{aligned}
			L(\alpha_1) \leq L(\alpha_2) + \langle \nabla L(\alpha_2), \alpha_1-\alpha_2 \rangle + \frac{L_\alpha}{2} \left \| \alpha_1-\alpha_2 \right \|^2_2 \\
			\left\| \nabla L(\alpha_1) - \nabla L(\alpha_2) \right\|_2 \leq L_\alpha \left\| \alpha_1 - \alpha_2 \right \|_2.
		\end{aligned}
	\end{equation}
	
	We also have that $L(\boldsymbol{w})$ is $L_w$-smooth where $L_w=\frac{L_{w x}L_{x w}}{\mu} + L_{w w}$.
	\end{lem}
	\begin{proof}
		Here, since we only focus on $\alpha$ and $\boldsymbol{x}$ when the $w$ is fixed, we abbreviate $\boldsymbol{x}^*(w, \alpha)$ as $\boldsymbol{x}^*(\alpha)$ for convenience.
		
		By the Asm.\ref{asm3}, we have
		\begin{equation}
			\label{equation_lemma_1}
			f(\boldsymbol{w}, \alpha_2, \boldsymbol{x}^*_i(\alpha_1)) \leq f(\boldsymbol{w}, \alpha_2, \boldsymbol{x}^*_i(\alpha_2)) + \langle \nabla_{\boldsymbol{x}} f(\boldsymbol{w}, \alpha_2, \boldsymbol{x}^*_i(\alpha_2)), \boldsymbol{x}^*_i(\alpha_1)-\boldsymbol{x}^*_i(\alpha_2) \rangle - \frac{\mu}{2} \left \|  \boldsymbol{x}^*_i(\alpha_1)-\boldsymbol{x}^*_i(\alpha_2) \right \|^2_2, 
		\end{equation}
		\begin{equation}
			\label{equation_lemma_2}
			f(\boldsymbol{w}, \alpha_2, \boldsymbol{x}^*_i(\alpha_2)) \leq f(\boldsymbol{w}, \alpha_2, \boldsymbol{x}^*_i(\alpha_1)) + \langle \nabla_{\boldsymbol{x}} f(\boldsymbol{w}, \alpha_2, \boldsymbol{x}^*_i(\alpha_1)), \boldsymbol{x}^*_i(\alpha_2)-\boldsymbol{x}^*_i(\alpha_1) \rangle - \frac{\mu}{2} \left \|  \boldsymbol{x}^*_i(\alpha_1)-\boldsymbol{x}^*_i(\alpha_2) \right \|^2_2.
		\end{equation}
		
		Since $\langle \nabla_{\boldsymbol{x}} f(\boldsymbol{w}, \alpha_2, \boldsymbol{x}^*_i(\alpha_2)), \boldsymbol{x}^*_i(\alpha_1)-\boldsymbol{x}^*_i(\alpha_2) \rangle \leq 0$, combining (\ref{equation_lemma_1}) and (\ref{equation_lemma_2}), we obtain
		\begin{equation}
			\label{equation_lemma_3}
			\begin{aligned}
				\mu \left\| \boldsymbol{x}^*_i(\alpha_1)-\boldsymbol{x}^*_i(\alpha_2) \right\|_2^2 &\leq \langle \nabla_{\boldsymbol{x}} f(\boldsymbol{w}, \alpha_2, \boldsymbol{x}^*_i(\alpha_1)), \boldsymbol{x}^*_i(\alpha_2)-\boldsymbol{x}^*_i(\alpha_1) \rangle \\
				& \leq \langle \nabla_{\boldsymbol{x}} f(\boldsymbol{w}, \alpha_2, \boldsymbol{x}^*_i(\alpha_1))-\nabla_{\boldsymbol{x}} f(\boldsymbol{w}, \alpha_1, \boldsymbol{x}^*_i(\alpha_1)), \boldsymbol{x}^*_i(\alpha_2)-\boldsymbol{x}^*_i(\alpha_1) \rangle\\
				& \leq \left\| \nabla_{\boldsymbol{x}} f(\boldsymbol{w}, \alpha_2, \boldsymbol{x}^*_i(\alpha_1))-\nabla_{\boldsymbol{x}} f(\boldsymbol{w}, \alpha_1, \boldsymbol{x}^*_i(\alpha_1)) \right\|_2 \left\| \boldsymbol{x}^*_i(\alpha_2)-\boldsymbol{x}^*_i(\alpha_1) \right\|_2 \\
				& \leq L_{x \alpha}\left\| \alpha_1 - \alpha_2 \right\|_2 \left\| \boldsymbol{x}^*_i(\alpha_1)-\boldsymbol{x}^*_i(\alpha_2) \right\|_2 ,\\
			\end{aligned}
		\end{equation}
		where the second inequality holds because $\langle \nabla_{\boldsymbol{x}} f(\boldsymbol{w}, \alpha_1, \boldsymbol{x}^*_i(\alpha_1)), \boldsymbol{x}^*_i(\alpha_2)-\boldsymbol{x}^*_i(\alpha_1) \rangle \leq 0$.
		
		Then (\ref{equation_lemma_3}) immediately yields 
		\begin{equation}
			\label{equation_lemma_4}
			\left\| \boldsymbol{x}^*_i(\alpha_1)-\boldsymbol{x}^*_i(\alpha_2) \right\|_2 \leq \frac{L_{x \alpha}}{\mu}\left\|  \alpha_1- \alpha_2\right\|_2.
		\end{equation}
		
		Then for $i \in [n]$, we have
		\begin{equation}
			\begin{aligned}
				\left \| \nabla_\alpha f(\boldsymbol{w}, \alpha_1, \boldsymbol{x}^*_i(\alpha_1)) - \nabla_\alpha f(\boldsymbol{w}, \alpha_2, \boldsymbol{x}^*_i(\alpha_2)) \right \|_2  &\leq \left \| \nabla_\alpha f(\boldsymbol{w}, \alpha_1, \boldsymbol{x}^*_i(\alpha_1)) - \nabla_\alpha f(\boldsymbol{w}, \alpha_1, \boldsymbol{x}^*_i(\alpha_2)) \right \|_2 \\
				&\ \ \ \ \ \ \ + \left \| \nabla_\alpha f(\boldsymbol{w}, \alpha_1, \boldsymbol{x}^*_i(\alpha_2)) - \nabla_\alpha f(\boldsymbol{w}, \alpha_2, \boldsymbol{x}^*_i(\alpha_2)) \right \|_2 \\
				& \leq L_{\alpha x}\left\| \boldsymbol{x}^*_i(\alpha_1)-\boldsymbol{x}^*_i(\alpha_2) \right\|_2 +  L_{\alpha \alpha}\left\| \alpha_1- \alpha_2 \right\|_2  \\
				& \leq \left( \frac{L_{\alpha x}L_{x \alpha}}{\mu} + L_{\alpha \alpha} \right) \left\| \alpha_1- \alpha_2 \right\|_2.
			\end{aligned}
		\end{equation}
		
		Finally, by the definition of $L(\alpha)$, we have
		\begin{equation}
			\begin{aligned}
				\left \| \nabla L(\alpha_1) - \nabla L(\alpha_2) \right \|_2 &\leq \left \| \frac{1}{n}\sum_{i=1}^n \nabla_\alpha f(\boldsymbol{w}, \alpha_1, \boldsymbol{x}^*_i(\alpha_1)) - \frac{1}{n}\sum_{i=1}^n \nabla_\alpha f(\boldsymbol{w}, \alpha_2, \boldsymbol{x}^*_i(\alpha_2)) \right \|_2 \\ 
				& \leq \frac{1}{n}\sum_{i=1}^n \left \| \nabla_\alpha f(\boldsymbol{w}, \alpha_1, \boldsymbol{x}^*_i(\alpha_1)) - \nabla_\alpha f(\boldsymbol{w}, \alpha_2, \boldsymbol{x}^*_i(\alpha_2)) \right \|_2 \\
				& \leq  \left( \frac{L_{\alpha x}L_{x \alpha}}{\mu} + L_{\alpha \alpha} \right) \left\| \alpha_1- \alpha_2 \right\|_2.
			\end{aligned}
		\end{equation}
		
		This completes the proof.
	\end{proof}

	\begin{lem}
	\label{lem3}
	Under the Asm.\ref{asm1}, we have $\Phi(\boldsymbol{w})$ is $L_w$-smooth where$L_w=\frac{L_{w \alpha}L_{\alpha w}}{\mu} + L_{w w}$. For any $\boldsymbol{w}_1, \boldsymbol{w}_2$, it holds
	\begin{equation}
		\begin{aligned}
			\Phi(\boldsymbol{w}_1) \leq \Phi(\boldsymbol{w}_2) + \langle \nabla \Phi(\boldsymbol{w}_2), \boldsymbol{w}_1-\boldsymbol{w}_2 \rangle + \frac{L_w}{2} \left \| \boldsymbol{w}_1 - \boldsymbol{w}_2 \right \|^2_2 \\
			\left\| \nabla \Phi(\boldsymbol{w}_1) - \nabla \Phi(\boldsymbol{w}_2) \right\|_2 \leq L_w \left\| \boldsymbol{w}_1 - \boldsymbol{w}_2 \right \|_2.
		\end{aligned}
	\end{equation}
	\end{lem}

	\begin{proof}
		By the Rem.\ref{rem:restrict}, we have
		\begin{equation}
			\label{equation_lemma2_1}
			L(\boldsymbol{w}_2, \alpha^*(\boldsymbol{w}_1)) \leq L(\boldsymbol{w}_2, \alpha^*(\boldsymbol{w}_2)) + \langle \nabla_{\alpha} L(\boldsymbol{w}_2, \alpha^*(\boldsymbol{w}_2)), \alpha^*(\boldsymbol{w}_1)-\alpha^*(\boldsymbol{w}_2) \rangle - \frac{\mu}{2} \left \|  \alpha^*(\boldsymbol{w}_1)-\alpha^*(\boldsymbol{w}_2) \right \|^2_2 ,
		\end{equation}
		\begin{equation}
			\label{equation_lemma2_2}
			L(\boldsymbol{w}_2, \alpha^*(\boldsymbol{w}_2)) \leq L(\boldsymbol{w}_2, \alpha^*(\boldsymbol{w}_1)) + \langle \nabla_{\alpha} L(\boldsymbol{w}_2, \alpha^*(\boldsymbol{w}_1)), \alpha^*(\boldsymbol{w}_2)-\alpha^*(\boldsymbol{w}_1) \rangle - \frac{\mu}{2} \left \|  \alpha^*(\boldsymbol{w}_1)-\alpha^*(\boldsymbol{w}_2) \right \|^2_2 .
		\end{equation}
		
		Since $\langle \nabla_{\alpha} L(\boldsymbol{w}_2, \alpha^*(\boldsymbol{w}_2)), \alpha^*(\boldsymbol{w}_1)-\alpha^*(\boldsymbol{w}_2) \rangle \leq 0$, combining (\ref{equation_lemma2_1}) and (\ref{equation_lemma2_2}), we obtain
		\begin{equation}
			\label{equation_lemma2_3}
			\begin{aligned}
				\mu \left\| \alpha^*(\boldsymbol{w}_1)-\alpha^*(\boldsymbol{w}_2) \right\|_2^2 &\leq \langle \nabla_{\alpha} L(\boldsymbol{w}_2, \alpha^*(\boldsymbol{w}_1)), \alpha^*(\boldsymbol{w}_2)-\alpha^*(\boldsymbol{w}_1) \rangle \\
				& \leq \langle \nabla_{\alpha} L(\boldsymbol{w}_2, \alpha^*(\boldsymbol{w}_1), )-\nabla_{\alpha} L(\boldsymbol{w}_1, \alpha^*(\boldsymbol{w}_1)), \alpha^*(\boldsymbol{w}_2)-\alpha^*(\boldsymbol{w}_1) \rangle\\
				& \leq \left\| \nabla_{\alpha} L(\boldsymbol{w}_2, \alpha^*(\boldsymbol{w}_1))-\nabla_{\alpha} L(\boldsymbol{w}_1, \alpha^*(\boldsymbol{w}_1)) \right\|_2 \left\| \alpha^*(\boldsymbol{w}_2)-\alpha^*(\boldsymbol{w}_1) \right\|_2 \\
				& \leq L_{\alpha w}\left\| \boldsymbol{w}_2- \boldsymbol{w}_1 \right\|_2 \left\|\alpha^*(\boldsymbol{w}_2)-\alpha^*(\boldsymbol{w}_1) \right\|_2 .\\
			\end{aligned}
		\end{equation}
		
		Then (\ref{equation_lemma2_3}) yields 
		\begin{equation}
			\label{equation_lemma2_4}
			\left\| \alpha^*(\boldsymbol{w}_1)-\alpha^*(\boldsymbol{w}_2) \right\|_2 \leq \frac{L_{\alpha w}}{\mu}\left\| \boldsymbol{w}_1- \boldsymbol{w}_2 \right\|_2.
		\end{equation}
		
		Then we have for $i \in [n]$, 
		\begin{equation}
			\begin{aligned}
				\left \| \nabla_{\boldsymbol{w}} L(\boldsymbol{w}_1, \alpha^*(\boldsymbol{w}_1)) - \nabla_{\boldsymbol{w}} L(\boldsymbol{w}_2, \alpha^*(\boldsymbol{w}_2)) \right \|_2  &\leq \left \| \nabla_{\boldsymbol{w}} L(\boldsymbol{w}_1, \alpha^*(\boldsymbol{w}_1))) - \nabla_{\boldsymbol{w}} L(\boldsymbol{w}_1, \alpha^*(\boldsymbol{w}_2)) \right \|_2 \\
				&\ \ \ \ \ \ \ + \left \| \nabla_{\boldsymbol{w}} L(\boldsymbol{w}_1, \alpha^*(\boldsymbol{w}_2)) - \nabla_{\boldsymbol{w}} L(\boldsymbol{w}_2, \alpha^*(\boldsymbol{w}_1)) \right \|_2 \\
				& \leq L_{w \alpha}\left\| \alpha^*(\boldsymbol{w}_1)- \alpha^*(\boldsymbol{w}_2) \right\|_2 +  L_{w w}\left\| \boldsymbol{w}_1- \boldsymbol{w}_2 \right\|_2  \\
				& \leq \big( \frac{L_{w \alpha}L_{\alpha w}}{\mu} + L_{w w} \big) \left\| \boldsymbol{w}_1- \boldsymbol{w}_2 \right\|_2 .
			\end{aligned}
		\end{equation}
		
		Finally, by the definition of $\Phi(\boldsymbol{w})$, we have
		\begin{equation}
			\begin{aligned}
				\left \| \nabla \Phi(\boldsymbol{w}_1) - \nabla \Phi(\boldsymbol{w}_2) \right \|_2 &\leq \left \| \frac{1}{n}\sum_{i=1}^n \nabla_{\boldsymbol{w}} L(\boldsymbol{w}_1, \alpha^*(\boldsymbol{w}_1)) - \frac{1}{n}\sum_{i=1}^n \nabla_{\boldsymbol{w}} L(\boldsymbol{w}_2, \alpha^*(\boldsymbol{w}_2)) \right \|_2 \\ 
				& \leq \frac{1}{n}\sum_{i=1}^n \left \| \nabla_\alpha L(\boldsymbol{w}_1, \alpha^*(\boldsymbol{w}_1)) - \nabla_\alpha L(\boldsymbol{w}_2, \alpha^*(\boldsymbol{w}_2)) \right \|_2 \\
				& \leq  \big( \frac{L_{w \alpha}L_{\alpha w}}{\mu} + L_{w w} \big) \left\| \boldsymbol{w}_1- \boldsymbol{w}_2  \right\|_2.
			\end{aligned}
		\end{equation}
		
		This completes the proof.
	\end{proof}
	
	\begin{lem}
		\label{lemma4}
	Under Asm.\ref{asm1} and \ref{asm3}, the approximate stochastic gradient $\hat{g}(\alpha)$ satisfies
	\begin{equation}
		\left \|  \hat{g}(\alpha)-g(\alpha) \right \|_2 \leq L_{\alpha x} \sqrt{\frac{\delta}{\mu}} .
	\end{equation}
	
	We also have 
	\begin{equation}
		\begin{aligned}
			\left \|  \hat{g}(\boldsymbol{w})-g(\boldsymbol{w}) \right \|_2 \leq L_{w x} \sqrt{\frac{\delta}{\mu}} .\\
		\end{aligned}
	\end{equation}
	\end{lem}

	\begin{proof}
		We have
		\begin{equation}
			\label{equation_lemma3_0}
			\begin{aligned}
				\left \| \hat{g}(\alpha)-g(\alpha)  \right\|_2 &\leq \left \| \frac{1}{|\mathcal{B}|} \sum_{i \in \mathcal{B}}\big(\nabla_\alpha f(\boldsymbol{w}, \alpha, \hat{\boldsymbol{x}}_i(\alpha)) - \nabla_\alpha f(\boldsymbol{w}, \alpha, \boldsymbol{x}^*_i(\alpha))\big) \right\|_2 \\
				& \leq  \frac{1}{|\mathcal{B}|}\sum_{i \in \mathcal{B}} \left \| \nabla_\alpha f(\boldsymbol{w}, \alpha, \hat{\boldsymbol{x}}_i(\alpha)) - \nabla_\alpha f(\boldsymbol{w}, \alpha, \boldsymbol{x}^*_i(\alpha)) \right\|_2 \\
				& \leq \frac{1}{|\mathcal{B}|}\sum_{i \in \mathcal{B}}  L_{\alpha x} \left \| \hat{\boldsymbol{x}}_i(\alpha) - \boldsymbol{x}^*_i(\alpha) \right\|_2 .
			\end{aligned}
		\end{equation}
		
		By the Asm.\ref{asm3}, we have
		\begin{equation}
			\label{equation_lemma3_1}
			\mu \left \| \hat{\boldsymbol{x}}_i(\alpha) - \boldsymbol{x}^*_i(\alpha) \right\|_2^2 \leq \langle  \nabla_\alpha f(\boldsymbol{w}, \alpha, \hat{\boldsymbol{x}}_i(\alpha)) - \nabla_\alpha f(\boldsymbol{w}, \alpha, \boldsymbol{x}^*_i(\alpha)), \hat{\boldsymbol{x}}_i(\alpha) - \boldsymbol{x}^*_i(\alpha) \rangle.
		\end{equation}
		
		Since the $\hat{\boldsymbol{x}}_i(\alpha)$ is a $\delta$-approximate solution to $\boldsymbol{x}^*_i(\alpha)$, if it satisfies that 
		\begin{equation}
			\label{equation_lemma3_2}
			\begin{aligned}
				\langle \boldsymbol{x}^*_i(\alpha) - \hat{\boldsymbol{x}}_i(\alpha), \nabla_{\boldsymbol{x}} f(\boldsymbol{w}, \alpha, \hat{\boldsymbol{x}}_i(\alpha) \rangle \leq \delta .\\
			\end{aligned}
		\end{equation}
		
		Furthermore, we have 
		\begin{equation}
			\label{equation_lemma3_3}
			\langle \hat{\boldsymbol{x}}_i(\alpha) - \boldsymbol{x}^*_i(\alpha), \nabla_{\boldsymbol{x}} f(\boldsymbol{w}, \alpha, \boldsymbol{x}^*_i(\alpha)) \rangle \leq 0.
		\end{equation}
		
		Combining (\ref{equation_lemma3_2}) and (\ref{equation_lemma3_3}), we obtain
		\begin{equation}
			\label{equation_lemma3_4}
			\langle \hat{\boldsymbol{x}}_i(\alpha) - \boldsymbol{x}^*_i(\alpha), \nabla_{\boldsymbol{x}} f(\boldsymbol{w}, \alpha, \boldsymbol{x}^*_i(\alpha))-\nabla_{\boldsymbol{x}} f(\alpha, \hat{\boldsymbol{x}}_i)(\alpha)  \rangle \leq \delta.
		\end{equation}
		
		Combining (\ref{equation_lemma3_1}) and (\ref{equation_lemma3_4}), we obtain
		\begin{equation}
			\label{equation_lemma3_5}
			\mu \left \| \hat{\boldsymbol{x}}_i(\alpha) - \boldsymbol{x}^*_i(\alpha) \right\|_2^2 \leq \delta
		\end{equation}
		
		Combining (\ref{equation_lemma3_0}) and (\ref{equation_lemma3_5}), we obtain
		\begin{equation}
			\label{equation_lemma3_6}
			\left \| \hat{g}(\alpha)-g(\alpha)  \right\|_2 \leq L_{\alpha x} \sqrt{\frac{\delta}{\mu}} .
		\end{equation}
		
	\end{proof}
	
	\begin{lem}
		\label{lemma5}
    $g(\boldsymbol{w})=\frac{1}{M}\sum_{i=1}^{M}G_w(\boldsymbol{w}_t,\alpha_t,\xi_i)$ and $g(\alpha)=\frac{1}{M}\sum_{i=1}^{M}G_\alpha(\boldsymbol{w}_t,\alpha_t, \xi_i)$ are unbiased and have bounded variance:
	\begin{equation}
		\begin{aligned}
			& \mathbb{E}[g(\boldsymbol{w})] = \nabla_{\boldsymbol{w}} L(\boldsymbol{w}_t, \alpha_t),
			&\mathbb{E}\left[\left\| g(\boldsymbol{w}) \right\|_2^2\right] = \left\| \nabla_{\boldsymbol{w}} L(\boldsymbol{w}_t, \alpha_t) \right\|^2_2 + \frac{\sigma^2}{M} ,\\
			& \mathbb{E}[g(\alpha)] = \nabla_\alpha L(\boldsymbol{w}_t, \alpha_t), &\mathbb{E}\left[\left\| g(\alpha) \right\|_2^2\right] = \left\| \nabla_\alpha L(\boldsymbol{w}_t, \alpha_t) \right\|^2_2 + \frac{\sigma^2}{M} .\\
		\end{aligned}
	\end{equation}	
	\end{lem}

	\begin{proof}
		Since $\hat{g}$ is unbiased, we have
		\begin{equation}\nonumber
			\begin{aligned}
				& \mathbb{E}[g(\boldsymbol{w})] = \nabla_{\boldsymbol{w}} L(\boldsymbol{w}_t, \alpha_t),& \mathbb{E}[g(\alpha)] = \nabla_\alpha L(\boldsymbol{w}_t, \alpha_t).
			\end{aligned}
		\end{equation}
		Furthermore, we have
		\begin{equation}\nonumber
			\begin{aligned}
				\mathbb{E}[\left\| g(\boldsymbol{w})-\nabla_{\boldsymbol{w}} L(\boldsymbol{w}_t, \alpha_t)\right\|_2^2 ] & = \mathbb{E}\left[ \left\| \frac{1}{M}\sum_{i=1}^{M}G_w(\boldsymbol{w}_t,\alpha_t, \xi_i) - \nabla_{\boldsymbol{w}} L(\boldsymbol{w}_t, \alpha_t)\right\|^2_2 \right] \\
				& = \frac{\sum_{i=1}^{M}\mathbb{E}\left[ \left\| G_w(\boldsymbol{w}_t,\alpha_t, \xi_i) - \nabla_{\boldsymbol{w}} L(\boldsymbol{w}_t, \alpha_t)  \right\|_2^2 \right] }{M^2}
				& \leq \frac{\sigma^2}{M}.
			\end{aligned}
		\end{equation}
		\begin{equation}\nonumber
			\begin{aligned}
				\mathbb{E}[\left\| g(\alpha)-\nabla_\alpha L(\boldsymbol{w}_t, \alpha_t)\right\|_2^2 ] & = \mathbb{E}\left[ \left\| \frac{1}{M}\sum_{i=1}^{M}G_\alpha(\boldsymbol{w}_t,\alpha_t, \xi_i) - \nabla_\alpha L(\boldsymbol{w}_t, \alpha_t)\right\|^2_2 \right] \\
				& = \frac{\sum_{i=1}^{M}\mathbb{E}\left[ \left\| G_\alpha(\boldsymbol{w}_t,\alpha_t, \xi_i) - \nabla_\alpha L(\boldsymbol{w}_t, \alpha_t)  \right\|_2^2 \right] }{M^2}
				& \leq \frac{\sigma^2}{M}.
			\end{aligned}
		\end{equation}
	\end{proof}

	\begin{lem}
	\label{lemma6}
	    The iterates $\{ \boldsymbol{w}_t \}_{t \geq 1}$ satisfy the following inequality: 
	\begin{equation}
		\begin{aligned}
			\mathbb{E}[\Phi(\boldsymbol{w}_t)] & \leq \mathbb{E}[\Phi(\boldsymbol{w}_{t-1})] - \left(\frac{\eta_w}{2}-2L_w\eta_w^2\right) \mathbb{E}\left[\left\| \nabla\Phi(\boldsymbol{w}_{t-1}) \right\|_2^2\right] \\
			& \ \ \ \ \ + (\frac{\eta_w}{2} + 2L_w\eta_w^2) \mathbb{E}\left[\left\| \nabla\Phi(\boldsymbol{w}_{t-1})-\nabla_{\boldsymbol{w}} L(\boldsymbol{w}_{t-1}, \alpha_{t-1}) \right\|_2^2\right] + \frac{L_w\eta_w^2\sigma^2}{M} + \frac{\delta L_w L_{w x}^2\eta_w^2}{\mu} + L_{w x}\ell_w\sqrt{\frac{\delta}{\mu}} .\\
		\end{aligned}
	\end{equation}
	\end{lem}
	
	\begin{proof}
		By Lem.\ref{lem2}, we have 
		\begin{equation}
			\label{equ_lemma4_1}
			\Phi(\boldsymbol{w}_t) \leq \Phi(\boldsymbol{w}_{t-1}) + \langle \nabla\Phi(\boldsymbol{w}_{t-1}), \boldsymbol{w}_t-\boldsymbol{w}_{t-1} \rangle + \frac{L_w}{2}\left\| \boldsymbol{w}_t-\boldsymbol{w}_{t-1} \right\|_2^2 .
		\end{equation}
		Plugging $\boldsymbol{w}_t-\boldsymbol{w}_{t-1}=-\eta_w \hat{g}(\boldsymbol{w})$ into (\ref{equ_lemma4_1}) yields that
		\begin{equation}
			\begin{aligned}
				\Phi(\boldsymbol{w}_t) & \leq \Phi(\boldsymbol{w}_{t-1}) - \eta_w \langle \nabla\Phi(\boldsymbol{w}_{t-1}), \hat{g}(\boldsymbol{w}) \rangle + \frac{L_w\eta_w^2}{2}\left\| \hat{g}(\boldsymbol{w})\right\|_2^2 \\
				& = \Phi(\boldsymbol{w}_{t-1}) - \eta_w \left\| \nabla\Phi(\boldsymbol{w}_{t-1}) \right\|_2^2 - \eta_w \langle \nabla\Phi(\boldsymbol{w}_{t-1}), \hat{g}(\boldsymbol{w})- \nabla\Phi(\boldsymbol{w}_{t-1}) \rangle + \frac{L_w\eta_w^2}{2}\left\| \hat{g}(\boldsymbol{w})\right\|_2^2 \\
				& \leq  \Phi(\boldsymbol{w}_{t-1}) - \eta_w \left\| \nabla\Phi(\boldsymbol{w}_{t-1}) \right\|_2^2 - \eta_w \langle \nabla\Phi(\boldsymbol{w}_{t-1}), g(\boldsymbol{w})- \nabla\Phi(\boldsymbol{w}_{t-1}) \rangle + L_w\eta_w^2\left\| g(\boldsymbol{w})\right\|_2^2 \\
				& \ \ \ \ \  + L_w\eta_w^2\left\| \hat{g}(\boldsymbol{w}) - g(\boldsymbol{w}) \right\|_2^2 - \eta_w \langle \nabla\Phi(\boldsymbol{w}_{t-1}), \hat{g}(\boldsymbol{w})- g(\boldsymbol{w}) \rangle \\
				& = \Phi(\boldsymbol{w}_{t-1}) - \eta_w \left\| \nabla\Phi(\boldsymbol{w}_{t-1}) \right\|_2^2 - \eta_w \langle \nabla\Phi(\boldsymbol{w}_{t-1}), g(\boldsymbol{w})- \nabla_{\boldsymbol{w}} L(\boldsymbol{w}_{t-1}, \alpha_{t-1}) \rangle \\
				& \ \ \ \ \ - \eta_w \langle \nabla\Phi(\boldsymbol{w}_{t-1}), \nabla_{\boldsymbol{w}} L(\boldsymbol{w}_{t-1}, \alpha_{t-1})- \nabla\Phi(\boldsymbol{w}_{t-1}) \rangle 
				+ L_w\eta_w^2\left\| g(\boldsymbol{w})\right\|_2^2 \\
				& \ \ \ \ \ + L_w\eta_w^2\left\| \hat{g}(\boldsymbol{w}) - g(\boldsymbol{w}) \right\|_2^2 - \eta_w \langle \nabla\Phi(\boldsymbol{w}_{t-1}), \hat{g}(\boldsymbol{w})- g(\boldsymbol{w}) \rangle \\
				& \leq \Phi(\boldsymbol{w}_{t-1}) - \frac{\eta_w}{2}\left\| \nabla\Phi(\boldsymbol{w}_{t-1}) \right\|_2^2 - \eta_w \langle \nabla\Phi(\boldsymbol{w}_{t-1}), g(\boldsymbol{w})- \nabla_{\boldsymbol{w}} L(\boldsymbol{w}_{t-1}, \alpha_{t-1}) \rangle \\
				& \ \ \ \ \ + \frac{\eta_w}{2}\left\| \nabla\Phi(\boldsymbol{w}_{t-1})-\nabla_{\boldsymbol{w}} L(\boldsymbol{w}_{t-1}, \alpha_{t-1}) \right\|_2^2 
				+ L_w\eta_w^2\left\| g(\boldsymbol{w})\right\|_2^2 + \frac{\delta L_w L_{w x}^2\eta_w^2}{\mu} + \eta_w L_{w x}\ell_w\sqrt{\frac{\delta}{\mu}} .
			\end{aligned}
		\end{equation}
		By Lem.\ref{lemma5}, taking an expectation on the both sides yields that,
		\begin{equation}
			\label{equ_lemma4_2}
			\begin{aligned}
				\mathbb{E}\left[\Phi(\boldsymbol{w}_t)\right] & \leq \mathbb{E}[\Phi(\boldsymbol{w}_{t-1})] - \frac{\eta_w}{2} 	\mathbb{E}\left[\left\| \nabla\Phi(\boldsymbol{w}_{t-1}) \right\|_2^2\right] + \frac{\eta_w}{2} 	\mathbb{E}\left[\left\| \nabla\Phi(\boldsymbol{w}_{t-1})-\nabla_{\boldsymbol{w}} L(\boldsymbol{w}_{t-1}, \alpha_{t-1}) \right\|_2^2\right] \\
				& \ \ \ \ \ + L_w\eta_w^2 \left\| \nabla_{\boldsymbol{w}} L(\boldsymbol{w}_{t-1}, \alpha_{t-1}) \right\|_2^2 + \frac{L_w\eta_w^2\sigma^2}{M} + \frac{\delta L_w L_{w x}^2\eta_w^2}{\mu} + \eta_w L_{w x}\ell_w\sqrt{\frac{\delta}{\mu}}  .\\
			\end{aligned}
		\end{equation}
		
		By the Cauchy-Schwartz inequality, we have
		\begin{equation}
			\label{equ_lemma4_3}
			\begin{aligned}
				\left\| \nabla_{\boldsymbol{w}} L(\boldsymbol{w}_{t-1}, \alpha_{t-1}) \right\|_2^2 \leq 2\left(\left\| \nabla_{\boldsymbol{w}} L(\boldsymbol{w}_{t-1}, \alpha_{t-1}) - \nabla\Phi(\boldsymbol{w}_{t-1}) \right\|_2^2 + \left\| \nabla\Phi(\boldsymbol{w}_{t-1}) \right\|_2^2 \right) .
			\end{aligned}
		\end{equation}
		Plugging (\ref{equ_lemma4_3}) into (\ref{equ_lemma4_2}) and taking the expectation of both side, yields that
		\begin{equation}\nonumber
			\begin{aligned}
				\mathbb{E}[\Phi(\boldsymbol{w}_t)] & \leq \mathbb{E}[\Phi(\boldsymbol{w}_{t-1})] - \left(\frac{\eta_w}{2}-2L_w\eta_w^2\right) \mathbb{E}\left[\left\| \nabla\Phi(\boldsymbol{w}_{t-1}) \right\|_2^2\right] \\
				& \ \ \ \ \ + \left(\frac{\eta_w}{2} + 2L_w\eta_w^2\right) \mathbb{E}\left[\left\| \nabla\Phi(\boldsymbol{w}_{t-1})-\nabla_{\boldsymbol{w}} L(\boldsymbol{w}_{t-1}, \alpha_{t-1}) \right\|_2^2\right] + \frac{L_w\eta_w^2\sigma^2}{M} + \frac{\delta L_w L_{w x}^2\eta_w^2}{\mu} + \eta_w L_{w x}\ell_w\sqrt{\frac{\delta}{\mu}} .\\
			\end{aligned}
		\end{equation}
		This completes the proof.	
	\end{proof}

	\begin{lem}
	\label{lemma7}
	let $\delta_t = \mathbb{E}\left[\left\|  \alpha^*(\boldsymbol{w}_t, x_t^*) - \alpha_t(\boldsymbol{w}_t, \hat{x}_t)  \right\|^2_2\right] $, the following statements holds that
	\begin{equation}\nonumber 
		\begin{aligned}
			\delta_t \leq  & \left(1-\frac{\mu}{2L_{\alpha w}} + \frac{L_{w w}^2 L_{\alpha w}^3 \eta_w^2}{2\mu^3} \right)\delta_{t-1} + \frac{L_{\alpha w}^3 \eta_w^2}{4\mu^3}\mathbb{E} \left[ \left\| \nabla\Phi(\boldsymbol{w}_{t-1}) \right\|_2^2\right]  \\
			& + \frac{2\sigma^2}{L_{\alpha w}^2 M} + \frac{L_{\alpha w}^3 \eta_{w}^2 \sigma^2}{4\mu^3 M} + \frac{L_{\alpha w}^3 L_{w x}^2 \eta_{w}^2 \sigma^2}{4\mu^4} + \frac{L_{\alpha w} L_{\alpha x}^2 \Delta}{8\mu^3} + \frac{2L_{\alpha x}^2 \delta}{\mu^3} .
		\end{aligned}
	\end{equation} 
	\end{lem}

	\begin{proof}
		\begin{equation}
			\label{equ_lemma6_1}
			\begin{aligned}
				\delta_t & = \mathbb{E}\left[\left\| \alpha^*(\boldsymbol{w}_t, x_t^*)-\alpha^*(\boldsymbol{w}_t, \hat{x}_t) + \alpha^*(\boldsymbol{w}_t, \hat{x}_t) -\alpha_t(\boldsymbol{w}_t, \hat{x}_t) \right\|_2^2\right] \\
				& \leq 2 \mathbb{E}\left[\left\|  \alpha^*(\boldsymbol{w}_t, x_t^*)-\alpha^*(\boldsymbol{w}_t, \hat{x}_t)   \right\|_2^2\right] + 2
				\mathbb{E}\left[\left\|  \alpha^*(\boldsymbol{w}_t, \hat{x}_t) -\alpha_t(\boldsymbol{w}_t, \hat{x}_t)   \right\|_2^2\right] .
			\end{aligned}
		\end{equation}
		
		By Young's inequality, for any $\epsilon_0 \le 0$, we have
		\begin{equation}
			\label{equ_lemma6_2}
			\begin{aligned}
				\mathbb{E}\left[\left\|  \alpha^*(\boldsymbol{w}_t, \hat{x}_t) -\alpha_t(\boldsymbol{w}_t, \hat{x}_t)   \right\|_2^2\right] \leq & \left(  1+ \frac{1}{\epsilon_0}  \right) \mathbb{E}\left[\left\|  \alpha^*(\boldsymbol{w}_{t-1}, \hat{x}_{t-1}) -\alpha_t(\boldsymbol{w}_t, \hat{x}_t)   \right\|_2^2\right] \\
				& + \left(  1+ \epsilon_0 \right) \mathbb{E}\left[\left\|  \alpha^*(\boldsymbol{w}_t, \hat{x}_t) -\alpha^*(\boldsymbol{w}_{t-1}, \hat{x}_{t-1})   \right\|_2^2\right] .
			\end{aligned}
		\end{equation}

		By the Rem.\ref{rem:restrict} and $\eta_{\alpha}=1/L_{\alpha w}$ (refer to \cite{nesterov1998introductory} Theorem 2.3.4), we have
		\begin{equation}
			\label{equ_lemma6_3}
			\begin{aligned}
				\mathbb{E} \left[\left\|  \alpha^*(\boldsymbol{w}_{t-1}, \hat{x}_t) -\alpha_t(\boldsymbol{w}_t, \hat{x}_t) \right\|_2^2\right] & \leq \left(1-\frac{\mu}{L_{\alpha w}} \right) \mathbb{E} \left[ \left\| \alpha^*(\boldsymbol{w}_{t-1}, \hat{x}_{t-1}) - \alpha_{t-1}(\boldsymbol{w}_{t-1}, \hat{x}_{t-1}) \right\|_2^2 \right] + \frac{\sigma^2}{L_{\alpha w}^2 M} .\\	
			\end{aligned}
		\end{equation}

		By triangle inequality, we have 
		\begin{equation}
			\label{equ_lemma6_4}
			\begin{aligned}
				\mathbb{E}\left[\left\| \alpha^*(\boldsymbol{w}_{t-1}, \hat{x}_{t-1}) - \alpha_{t-1}(\boldsymbol{w}_{t-1}, \hat{x}_{t-1}) \right\|_2^2\right] \leq & 2 \mathbb{E}\left[\left\| \alpha^*(\boldsymbol{w}_{t-1}, \hat{x}_{t-1}) - \alpha^*(\boldsymbol{w}_{t-1}, x^*_{t-1}) \right\|_2^2\right]  \\
				& + 2 \mathbb{E}\left[\left\| \alpha^*(\boldsymbol{w}_{t-1}, x^*_{t-1}) - \alpha_{t-1}(\boldsymbol{w}_{t-1}, \hat{x}_{t-1}) \right\|_2^2\right] .
			\end{aligned}
		\end{equation}	
		
		Plugging (\ref{equ_lemma6_4}) into (\ref{equ_lemma6_3}), yields that
		\begin{equation}
			\label{equ_lemma6_5}
			\begin{aligned}
				\mathbb{E} \left[\left\|  \alpha^*(\boldsymbol{w}_{t-1}, \hat{x}_{t-1}) -\alpha_t(\boldsymbol{w}_t, \hat{x}_t) \right\|_2^2\right] & \leq 2\left(1-\frac{\mu}{L_{\alpha w}} \right) \left[ \delta_{t-1} + \mathbb{E} \left[ \left\| \alpha^*(\boldsymbol{w}_{t-1}, \hat{x}_{t-1}) - \alpha^*(\boldsymbol{w}_{t-1}, x^*_{t-1}) \right\|_2^2 \right] \right] + \frac{\sigma^2}{L_{\alpha w}^2 M}.
			\end{aligned}
		\end{equation} 
		
		By triangle inequality, we have 
		\begin{equation}
			\label{equ_lemma6_6}
			\begin{aligned}
				\mathbb{E}\left[\left\| \alpha^*(\boldsymbol{w}_t, \hat{x}_t) -\alpha^*(\boldsymbol{w}_{t-1}, \hat{x}_{t-1})  \right\|_2^2\right] \leq & 2 \mathbb{E}\left[\left\| \alpha^*(\boldsymbol{w}_t, \hat{x}_{t}) - \alpha^*(\boldsymbol{w}_{t-1}, \hat{x}_{t}) \right\|_2^2\right]  \\
				& + 2 \mathbb{E}\left[\left\| \alpha^*(\boldsymbol{w}_{t-1}, \hat{x}_{t}) - \alpha^*(\boldsymbol{w}_{t-1}, \hat{x}_{t-1}) \right\|_2^2\right] .
			\end{aligned}
		\end{equation}
		
		Since $\alpha^*(\cdot)$ is $\frac{L_{\alpha w}}{\mu}$-Lipschitz, we have 
		\begin{equation}
			\label{equ_lemma6_7}
			\begin{aligned}
				\mathbb{E}\left[\left\| \alpha^*(\boldsymbol{w}_t, \hat{x}_{t}) - \alpha^*(\boldsymbol{w}_{t-1}, \hat{x}_{t}) \right\|_2^2 \right] \leq \frac{L_{\alpha w}^2}{\mu^2} \mathbb{E}\left[ \left\| \boldsymbol{w}_t - \boldsymbol{w}_{t-1} \right\|^2_2\right] .\\	
			\end{aligned}
		\end{equation}  
		
		Furthermore, we have
		\begin{equation}
			\label{equ_lemma6_8}
			\begin{aligned}
				\mathbb{E}\left[ \left\| \boldsymbol{w}_t - \boldsymbol{w}_{t-1} \right\|^2_2\right] &= \eta_w^2 \mathbb{E} \left[ \left\|  \hat{g}(\boldsymbol{w}) \right\|_2^2 \right]\\
				& = 2\eta_w^2 \mathbb{E}\left[ \left\| g(\boldsymbol{w}) \right\|_2^2 \right]  + 2\eta_w^2 \mathbb{E}\left[ \left\| \hat{g}(\boldsymbol{w}) - g(\boldsymbol{w}) \right\|_2^2 \right]\\
				& \leq  2\eta_w^2 \left\| \nabla_{\boldsymbol{w}} L(\boldsymbol{w}_{t-1}, \alpha_{t-1}) \right\|_2^2 + \frac{2\eta_w^2 \sigma^2}{M} + \frac{2\delta L_{w x}^2\eta_w^2}{\mu}\\
				& \mathop{\leq}\limits^{(\ref{equ_lemma4_3})}  4\eta_w^2 \left( \left\| \nabla_{\boldsymbol{w}} L(\boldsymbol{w}_{t-1}, \alpha_{t-1}) - \nabla\Phi(\boldsymbol{w}_{t-1}) \right\|_2^2 + \left\| \nabla\Phi(\boldsymbol{w}_{t-1}) \right\|_2^2\right) + \frac{2\eta_w^2 \sigma^2}{M}+ \frac{2\delta L_{w x}^2\eta_w^2}{\mu} \\
				& \leq  4 L_{w w}^2 \eta_w^2 \delta_{t-1} + 4 \eta_w^2 \left\| \nabla\Phi(\boldsymbol{w}_{t-1}) \right\|_2^2 + \frac{2\eta_w^2 \sigma^2}{M}+ \frac{2\delta L_{w x}^2\eta_w^2}{\mu} .
			\end{aligned}
		\end{equation}
		
		Plugging (\ref{equ_lemma6_7}) and (\ref{equ_lemma6_8}) into (\ref{equ_lemma6_6}) and taking the exception of  both side, yields that
		\begin{equation}
			\label{equ_lemma6_9}
			\begin{aligned}
				\mathbb{E}\left[\left\| \alpha^*(\boldsymbol{w}_t, \hat{x}_t) -\alpha^*(\boldsymbol{w}_{t-1}, \hat{x}_{t-1})  \right\|_2^2\right] & \leq \frac{8L_{\alpha w}^2 L_{w w}^2 \eta_{w}^2}{\mu^2}\delta_{t-1} + \frac{8L_{\alpha w}^2 \eta_{w}^2}{\mu^2} \mathbb{E} \left[ \left\| \nabla\Phi(\boldsymbol{w}_{t-1}) \right\|_2^2\right] + \frac{4L_{\alpha w}^2 \eta_{w}^2 \sigma^2}{\mu^2 M}
				\\ & \ \ \ \ \  + \frac{4L_{\alpha w}^2 L_{w x}^2 \eta_{w}^2 \delta}{\mu^3} + 2 \mathbb{E}\left[\left\| \alpha^*(\boldsymbol{w}_{t-1}, \hat{x}_{t}) - \alpha^*(\boldsymbol{w}_{t-1}, \hat{x}_{t-1}) \right\|_2^2\right] .
			\end{aligned}
		\end{equation}
		
		Then, plugging (\ref{equ_lemma6_5}) and (\ref{equ_lemma6_9}) into (\ref{equ_lemma6_2}) and let $\epsilon_0 = \frac{8(\kappa - 1)}{7-6\kappa}$ and $\kappa \leq \frac{7}{6}$, yields that
		\begin{equation}
			\label{equ_lemma6_10}
			\begin{aligned}
				\mathbb{E}\left[\left\|  \alpha^*(\boldsymbol{w}_t, \hat{x}_t) -\alpha_t(\boldsymbol{w}_t, \hat{x}_t)   \right\|_2^2\right] \leq & \left( \frac{1}{2}-\frac{\mu}{4L_{\alpha w}} + \frac{L_{w w}^2 L_{\alpha w}^3 \eta_w^2}{4\mu^3} \right)\delta_{t-1} + \frac{L_{\alpha w}^3 \eta_w^2}{4\mu^3}\mathbb{E} \left[ \left\| \nabla\Phi(\boldsymbol{w}_{t-1}) \right\|_2^2\right]  \\
				& + \frac{\sigma^2}{L_{\alpha w}^2 M} + \frac{L_{\alpha w}^3 \eta_{w}^2 \sigma^2}{8\mu^3 M} + \frac{L_{\alpha w}^3 L_{w x}^2 \eta_{w}^2 \delta}{8\mu^4} \\
				& + \frac{L_{\alpha w}}{16\mu} \mathbb{E}\left[\left\| \alpha^*(\boldsymbol{w}_{t-1}, \hat{x}_{t}) - \alpha^*(\boldsymbol{w}_{t-1}, \hat{x}_{t-1}) \right\|_2^2\right].
			\end{aligned}
		\end{equation}
		
		Since $\alpha^*(\cdot)$ is $\frac{L_{\alpha x}}{\mu}$-Lipschitz, we have  
		\begin{equation}
			\label{equ_lemma6_11}
			\begin{aligned}
				\mathbb{E}\left[\left\| \alpha^*(\boldsymbol{w}_{t-1}, \hat{x}_{t}) - \alpha^*(\boldsymbol{w}_{t-1}, \hat{x}_{t-1}) \right\|_2^2\right] \leq \frac{L_{\alpha x}^2}{\mu^2}\mathbb{E}\left[ \left\| \hat{x}_t - \hat{x}_{t-1} \right\|_2^2 \right] \leq \frac{L_{\alpha x}^2}{\mu^2} \Delta, 
			\end{aligned}
		\end{equation}
		\begin{equation}
			\label{equ_lemma6_12}
			\begin{aligned}
				\mathbb{E}\left[\left\| \alpha^*(\boldsymbol{w}_t, x_t^*)-\alpha^*(\boldsymbol{w}_t, \hat{x}_t)   \right\|_2^2\right] \leq \frac{L_{\alpha x}^2}{\mu^2}\mathbb{E}\left[ \left\| x^*_t - \hat{x}_t \right\|_2^2 \right] \leq \frac{L_{\alpha x}^2 \delta}{\mu^3},
			\end{aligned}
		\end{equation}
		where $\Delta$ is the maximum distance between two adversarial samples on the same sample. We can get the value of $\Delta$ by the diameter if $\mathcal{X}_i$.
		
		Plugging (\ref{equ_lemma6_9}) (\ref{equ_lemma6_10}) and (\ref{equ_lemma6_11}) into (\ref{equ_lemma6_1}), yields that
		\begin{equation}\nonumber 
			\begin{aligned}
				\delta_t \leq  & \left(1-\frac{\mu}{2L_{\alpha w}} + \frac{L_{w w}^2 L_{\alpha w}^3 \eta_w^2}{2\mu^3} \right)\delta_{t-1} + \frac{L_{\alpha w}^3 \eta_w^2}{4\mu^3}\mathbb{E} \left[ \left\| \nabla\Phi(\boldsymbol{w}_{t-1}) \right\|_2^2\right]  \\
				& + \frac{2\sigma^2}{L_{\alpha w}^2 M} + \frac{L_{\alpha w}^3 \eta_{w}^2 \sigma^2}{4\mu^3 M} + \frac{L_{\alpha w}^3 L_{w x}^2 \eta_{w}^2 \delta}{4\mu^4} + \frac{L_{\alpha w} L_{\alpha x}^2 \Delta}{8\mu^3} + \frac{2L_{\alpha x}^2 \delta}{\mu^3} .
			\end{aligned}
		\end{equation} 
		
	\end{proof}
	
	\begin{lem}
	\label{lemma8}
	let $\delta_t = \mathbb{E}[\left\|  \alpha^*(\boldsymbol{w}_t, x_t^*) - \alpha_t(\boldsymbol{w}_t, \hat{x}_t)  \right\|^2_2] $, the following statements holds that
	\begin{equation}\nonumber
		\begin{aligned}
			\mathbb{E}[\Phi(\boldsymbol{w}_t)] & \leq \mathbb{E}\left[\Phi(\boldsymbol{w}_{t-1})\right] - \frac{\eta_w}{4} \mathbb{E}\left[\left\| \nabla\Phi(\boldsymbol{w}_{t-1}) \right\|_2^2 \right] + \frac{3\eta_w L^2\delta_{t-1}}{4} + \frac{L_w\eta_w^2\sigma^2}{M} + \frac{\delta L_w L^2\eta_w^2}{\mu} + \eta_w L\ell_w\sqrt{\frac{\delta}{\mu}} .\\
		\end{aligned}
	\end{equation}
	\end{lem}

	\begin{proof}
		Let $\eta_{w} = \frac{1}{16(\kappa+1)^2 L}$, where $L$ is the maximum value in the set $\{ L_{w w}, L_{w \alpha}, L_{\alpha w}, L_{x w}, L_{w x}, L_{x \alpha}, L_{\alpha x}\}$, and $\kappa=\frac{L}{\mu}$, hence 
		\begin{equation}
			\label{equ_lemma7_1}
			\begin{aligned}
				\frac{1}{4}\eta_w \leq \frac{\eta_w}{2} - 2L_w \eta_w^2 \leq \frac{\eta_w}{2} + 2L_w \eta_w^2 \leq \frac{3}{4}\eta_w .
			\end{aligned}
		\end{equation}
		Then we have
		\begin{equation}
			\label{equ_lemma7_2}
			\begin{aligned}
				\mathbb{E} \left[ \left\| \nabla\Phi(\boldsymbol{w}_{t-1}) - \nabla_{\boldsymbol{w}} L(\boldsymbol{w}_{t-1}, \alpha_{t-1}) \right\|^2_2 \right] \leq
				L^2 \delta_{t-1}.
			\end{aligned}
		\end{equation}
		
		Combining (\ref{equ_lemma7_1}) (\ref{equ_lemma7_2}) and Lem.\ref{lemma6} yields that
		\begin{equation}\nonumber
			\begin{aligned}
				\mathbb{E}[\Phi(\boldsymbol{w}_t)] & \leq \mathbb{E}\left[\Phi(\boldsymbol{w}_{t-1})\right] - \frac{\eta_w}{4} \mathbb{E}\left[\left\| \nabla\Phi(\boldsymbol{w}_{t-1}) \right\|_2^2 \right] + \frac{3\eta_w L^2\delta_{t-1}}{4} + \frac{L_w\eta_w^2\sigma^2}{M} + \frac{\delta L_w L^2\eta_w^2}{\mu} + \eta_w L^2 \sqrt{\frac{\delta}{\mu}} .\\
			\end{aligned}
		\end{equation}   
		
	\end{proof}
	
	\subsection{Proof of Theorem 1}
	
	\begin{proof}
		Throughout this subsection, we define $\gamma=1-\frac{1}{2\kappa} + \frac{L^2\kappa^3 \eta_w^2}{2}$. 
		Since $\delta_0 \leq D^2$, where $D=|\Omega_{\alpha}|$,we have
		\begin{equation}
			\label{equ_lemma8_1}
			\begin{aligned}
				\delta_t \leq & \gamma^t D^2 +  \frac{\kappa^3 \eta_w^2}{4} \left( \sum_{j=0}^{t-1} \gamma^{t-1-j} \mathbb{E} \left[ \left\| \nabla\Phi(\boldsymbol{w}_{t-1}) \right\|_2^2\right] \right) \\ 
				& + \left( \frac{2\sigma^2}{L^2 M} + \frac{\kappa^3 \eta_{w}^2 \sigma^2}{4M} + \frac{L\kappa^4 \eta_{w}^2 \delta}{4} + \frac{2\kappa^2 \delta}{\mu} + \frac{\kappa^3 \Delta}{8} \right) \left( \sum_{j=0}^{t-1} \gamma^{t-1-j} \right) .
			\end{aligned}
		\end{equation}
		
		Combining (\ref{equ_lemma8_1}) and Lem.\ref{lemma8} yields that 
		\begin{equation}
			\label{equ_lemma8_2}
			\begin{aligned}
				\mathbb{E}[\Phi(\boldsymbol{w}_t)] \leq & \mathbb{E}\left[\Phi(\boldsymbol{w}_{t-1})\right] - \frac{\eta_w}{4} \mathbb{E}\left[\left\| \nabla\Phi(\boldsymbol{w}_{t-1}) \right\|_2^2 \right]  + \frac{L_w\eta_w^2\sigma^2}{M} + \frac{\delta L_w L^2\eta_w^2}{\mu} + \eta_w L^2 \sqrt{\frac{\delta}{\mu}} \\
				& + \frac{3\eta_w L^2 \gamma^t D^2}{4} + 
				\frac{3L^2 \kappa^3 \eta_w^3}{16} \left( \sum_{j=0}^{t-2} \gamma^{t-2-j} \mathbb{E} \left[ \left\| \nabla\Phi(\boldsymbol{w}_{j}) \right\|_2^2\right] \right) \\
				& + \frac{3\eta_w L^2}{4} \left( \frac{2\sigma^2}{L^2 M} + \frac{\kappa^3 \eta_{w}^2 \sigma^2}{4M} + \frac{L\kappa^4 \eta_{w}^2 \delta}{4} + \frac{2\kappa^2 \delta}{\mu} + \frac{\kappa^3 \Delta}{8} \right) \left( \sum_{j=0}^{t-2} \gamma^{t-2-j} \right)
			\end{aligned}
		\end{equation}
		
		Summing up (\ref{equ_lemma8_2}) over $t=1, 2, \cdots, T+1$ and rearranging the terms yields that
		\begin{equation}
			\label{equ_lemma8_3}
			\begin{aligned}
				\mathbb{E}[\Phi(\boldsymbol{w}_{T+1})] \leq & \mathbb{E}\left[\Phi(\boldsymbol{w}_{0})\right] - \frac{\eta_w}{4} \sum_{t=0}^{T}\mathbb{E}\left[\left\| \nabla\Phi(\boldsymbol{w}_{t}) \right\|_2^2 \right] + \left(T+1\right) \left(\frac{L_w\eta_w^2\sigma^2}{M} + \frac{\delta L_w L^2\eta_w^2}{\mu} + \eta_w L^2 \sqrt{\frac{\delta}{\mu}}\right) \\
				& + \frac{3\eta_w L^2 D^2}{4} \sum_{t=0}^{T} \gamma^t+ 
				\frac{3L^2 \kappa^3 \eta_w^3}{16}\left( \sum_{t=1}^{T+1}\sum_{j=0}^{t-2} \gamma^{t-2-j} \mathbb{E} \left[ \left\| \nabla\Phi(\boldsymbol{w}_{j}) \right\|_2^2\right] \right) \\
				& + \frac{3\eta_w L^2}{4} \left( \frac{2\sigma^2}{L^2 M} + \frac{\kappa^3 \eta_{w}^2 \sigma^2}{4M} + \frac{L\kappa^4 \eta_{w}^2 \delta}{4} + \frac{2\kappa^2 \delta}{\mu} + \frac{\kappa^3 \Delta}{8} \right) \left(\sum_{t=1}^{T+1} \sum_{j=0}^{t-2} \gamma^{t-2-j} \right)
			\end{aligned}
		\end{equation}

		Since $\eta_w = \frac{1}{16(\kappa+1)^2 L}$, we have $\gamma \leq 1 - \frac{1}{4\kappa}$ and $\frac{3L^2 \kappa^3 \eta_w^3}{16} \leq \frac{3\eta_{w}}{4096\kappa}$. This implies that $\sum_{t=0}^{T} \gamma^t \leq 4\kappa$ and 
		\begin{equation}
			\label{equ_lemma8_4}
			\begin{aligned}
				\sum_{t=1}^{T+1}\sum_{j=0}^{t-2} \gamma^{t-2-j} \mathbb{E} \left[ \left\| \nabla\Phi(\boldsymbol{w}_{j}) \right\|_2^2\right] & \leq  4\kappa \sum_{t=0}^{T} \mathbb{E} \left[ \left\| \nabla\Phi(\boldsymbol{w}_{t}) \right\|_2^2\right] \\
				\sum_{t=1}^{T+1} \sum_{j=0}^{t-2} \gamma^{t-2-j} & \leq 4\kappa (T+1)
			\end{aligned}
		\end{equation}
		
		Putting these pieces together yields that 
		\begin{equation}
			\label{equ_lemma8_5}
			\begin{aligned}
				\mathbb{E}[\Phi(\boldsymbol{w}_{T+1})] \leq & \mathbb{E}\left[\Phi(\boldsymbol{w}_{0})\right] - \frac{253\eta_w}{1024} \sum_{t=0}^{T}\mathbb{E}\left[\left\| \nabla\Phi(\boldsymbol{w}_{t}) \right\|_2^2 \right] + \eta_w \left(T+1\right) \left( \frac{\sigma^2}{8\kappa M} + \frac{L\delta}{8} + L^2 \sqrt{\frac{\delta}{\mu}} \right) \\
				& + 3\eta_w \kappa L^2 D^2 + \eta_w\left( T+1 \right)\left( 
				\frac{6\sigma^2 \kappa}{M} + \frac{3\sigma^2}{1024 M} + \frac{3\kappa L\delta}{1024} + 6\kappa^4 L \delta + \frac{3\kappa^4 L^2 \Delta}{8}\right) 
			\end{aligned}
		\end{equation}
		
		Futhermore, we have 
		\begin{equation}
			\label{equ_lemma8_6}
			\begin{aligned}
				\frac{\sigma^2}{8\kappa M} + \frac{6 \sigma^2 \kappa}{M} + \frac{3\sigma^2}{1024 M} \leq  \frac{6403\kappa \sigma^2}{1024M}
			\end{aligned}
		\end{equation}
		
		By definition of $\Delta_\Phi$ and plugging (\ref{equ_lemma8_6}) into (\ref{equ_lemma8_5}), we have
		\begin{equation}
			\label{equ_lemma8_7}
			\begin{aligned}
				\frac{1}{T+1}\left( \sum_{t=0}^{T} \mathbb{E}\left[ \left\| \nabla \Phi(\boldsymbol{w}_t) \right\|_2^2 \right] \right) & \leq \frac{1024\Delta_\Phi}{253\eta_w\left( T+1 \right)} + \frac{3072L^2D^2\kappa}{253\left( T+1 \right)} + \frac{6403\kappa\sigma^2}{253M} \\
				& \ \ \ \ \ + \frac{1024}{253}\left( \frac{3\kappa L \delta}{1024} + 6\kappa^4 L \delta + \frac{L\delta}{8} + L^2 \sqrt{\frac{\delta}{\mu}} \right) + \frac{384\kappa^4 L^2 \Delta}{253} \\
				& \leq  \frac{5\Delta_\Phi}{\eta_w \left( T+1 \right)} + \frac{13\kappa L^2 D^2}{T+1} + \frac{26\kappa \sigma^2}{M} + h_\delta + h_\Delta \\
				& \leq \frac{360\kappa^2L\Delta_\Phi + 13\kappa L^2D^2}{T+1} + \frac{26\kappa \sigma^2}{M} + h_\delta + h_\Delta\\
			\end{aligned}
		\end{equation}
		where $h_\delta = \frac{1024}{253}\left( \frac{3\kappa L \delta}{1024} + 6\kappa^4 L \delta + \frac{L\delta}{8} + L^2 \sqrt{\frac{\delta}{\mu}} \right)$ and $h_\Delta= \frac{384\kappa^4 L^2 \Delta}{253}$ and $\Delta_\Phi = \Phi(\boldsymbol{w}^0) - \min_w \Phi(\boldsymbol{w})$.
		On the right side of (\ref{equ_lemma8_7}), the second term $\frac{26\kappa \sigma^2}{M}$ is due to random sampling, the third term $h_\delta$ is due to the use of an approximate solution $\hat{\boldsymbol{x}}$ to the inner maximization problem instead of the optimal solution $\boldsymbol{x}^*$, and the fourth term $h_\Delta$ is due to $\mathcal{X}$ being a bounded set. And their values are very small.

	\end{proof}
	
	\section{Adversarial Attacks}
	\label{adversarial_attacks}
	\textbf{Fast Gradient Sign Method (FGSM)} \cite{goodfellow2015explaining} is a single-step attack that generates adversarial examples through a permutation along the gradient of the loss function with respect to the clean image feature vector as:
	\begin{equation}
	    \boldsymbol{x} = \boldsymbol{x} + \epsilon \cdot sign\left( \nabla_{\boldsymbol{x}}\ell(h_\theta(\boldsymbol{x}, y))  \right)
	\end{equation}
	
	\textbf{Projected Gradient Descent (PGD)} \cite{madry2018towards} starts from an initialization point that is uniformly sampled from the allowed $\epsilon$-ball centered at the clean image, and it expends FSGM by iteratively applying multiple small steps of permutation updating with respect to the current gradient as:
	\begin{equation}
	    \boldsymbol{x}^k = Proj\left(\boldsymbol{x}^{k-1} + \beta \cdot sign\left( \nabla_{\boldsymbol{x}}\ell(h_\theta(\boldsymbol{x}^{k-1}, y))  \right)\right)
	\end{equation}

	\textbf{Carlini $\&$ Wagner (C$\&$W)} \cite{carlini2017towards} is another powerful attack based on optimization, where an auxiliary variable $\omega$ is induced and an adversarial example constrained by $l_2$ norm is represented by $x'=\frac{1}{2}\left( \tanh \omega + 1 \right)$. It can be optimized by:
	\begin{equation}
	    \arg\min\limits_{\omega} \left\{ c\cdot f(\boldsymbol{x}') + \left\| \boldsymbol{x}'- \boldsymbol{x} \right\|_2^2  \right\}
	\end{equation}
	where 
	\begin{equation}
	    f(\boldsymbol{x}) = \max \left(  \max\limits_{i \neq y} Z(\boldsymbol{x}' - Z(\boldsymbol{x}')_y), -\kappa \right)
	\end{equation}
	and here $\kappa$ controls the confidence of the adversarial examples.
	It can also be extended to other $l_\infty$.
	
	\textbf{Auto Attack (AA)} \cite{croce2020reliable} is a combination of multiple attacks that forms a parameter-free and computationally affordable ensemble of attacks to evaluate adversarial robustness.
	The standard attacks include four selected attacks: APGD, targeted version of APGD-DLR and FAB, and Square Attack.

	\section{Additional Results}
	In this section, we provide additional results to further support the conclusions in the main text.
	\label{additional_results}
	
	\subsection{Experiment Setting}
	\label{Experiment_Setting}
	The adversarial training is applied with the maximal permutation $\epsilon$ of $8/255$ and a step size of $2/255$. The max number of iterations $K$ is set as 10. 
	For CE, we use SGD momentum optimizer, while for ours, we use SGDA momentum optimizer.
	The initial learning rate $\eta_w$ is set as 0.01 with decay $5 \times 10^{-4}$, and the batch size is 128.
    And the initial learning rate $\eta_{\alpha}$ is set as 0.1.
    In the training process, we adopt a learning rate step decay schedules, which cut the learning rate by a constant factor 0.001 every 30 constant number of epochs for all methods.

	\subsection{Score Distribution}
	\label{score_distribution}
	
	\begin{figure*}[h!]
        \subfigure[Clean]{
        \includegraphics[width=0.24\textwidth]{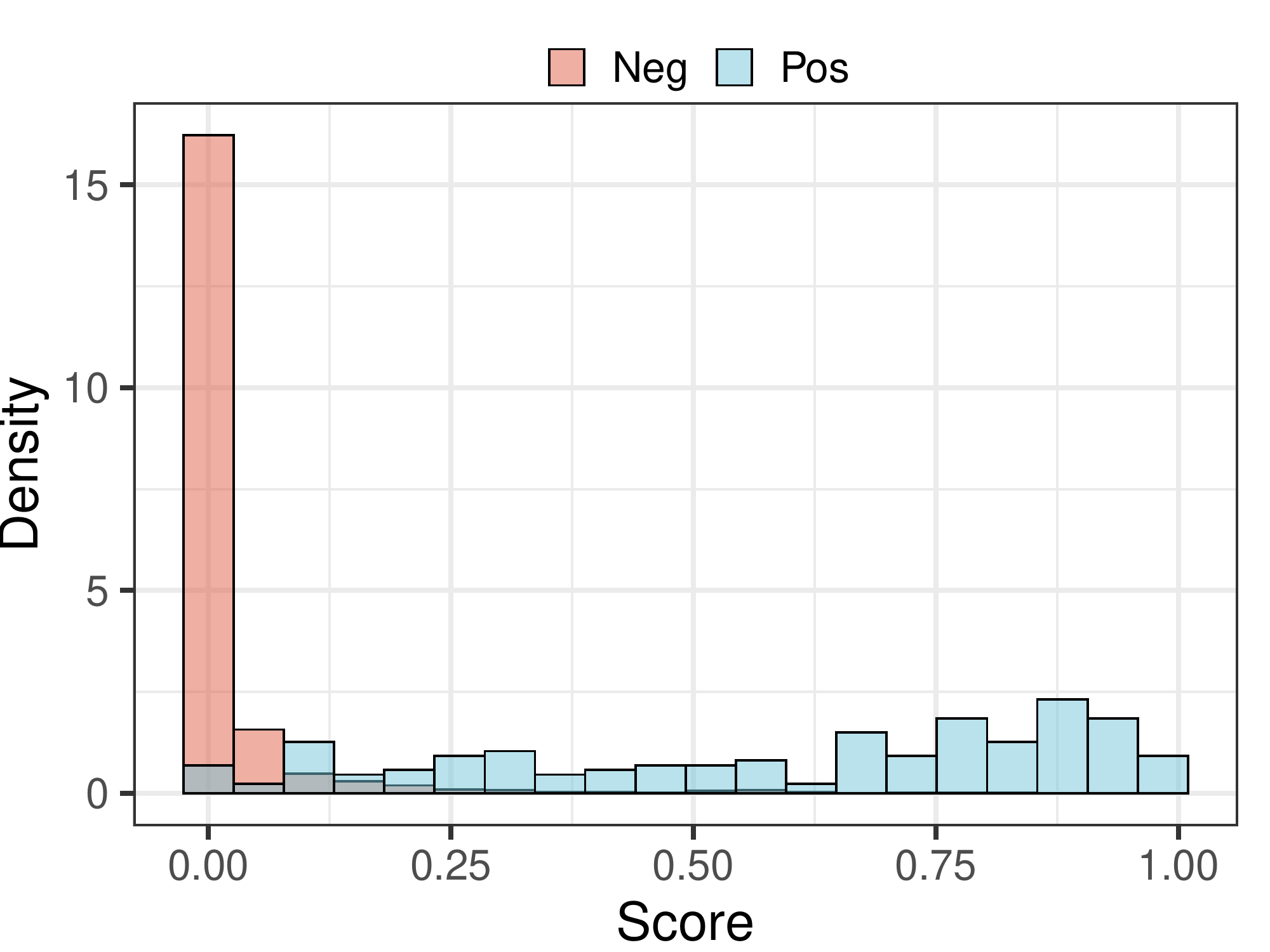}}
        \subfigure[FSGM]{
        \includegraphics[width=0.24\textwidth]{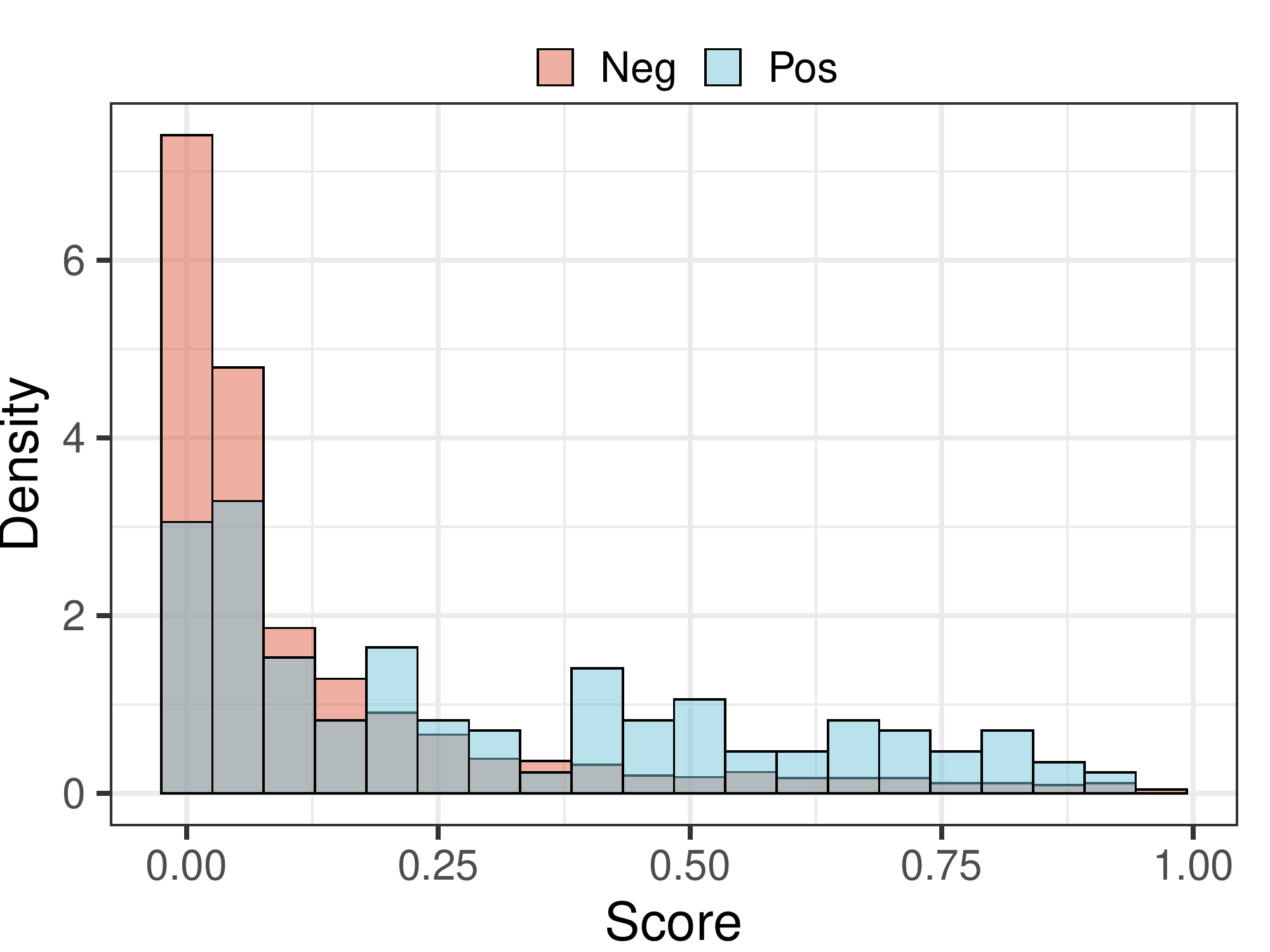}}
        \subfigure[PGD-10]{
        \includegraphics[width=0.24\textwidth]{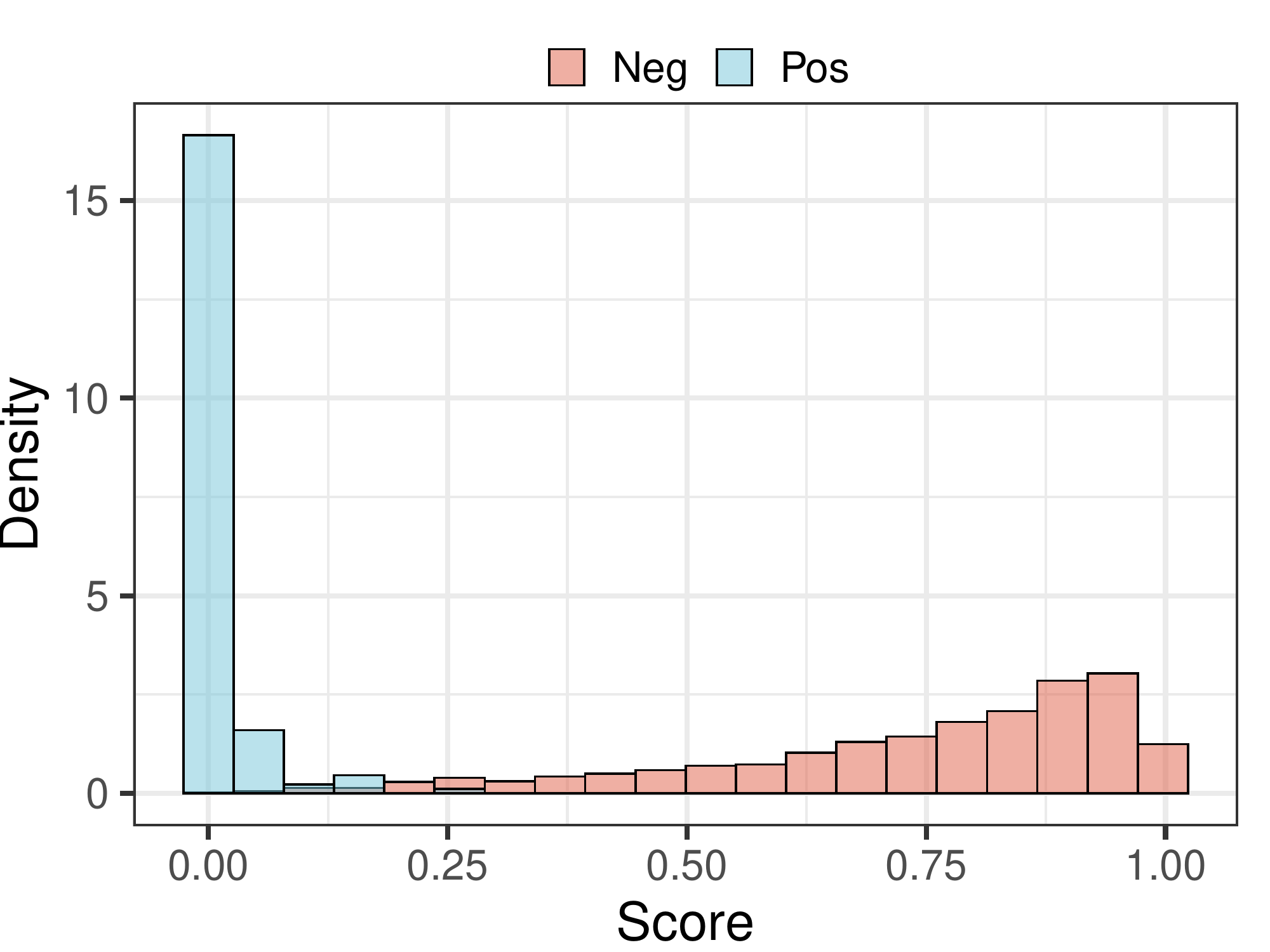}}
        \subfigure[PGD-20]{
        \includegraphics[width=0.24\textwidth]{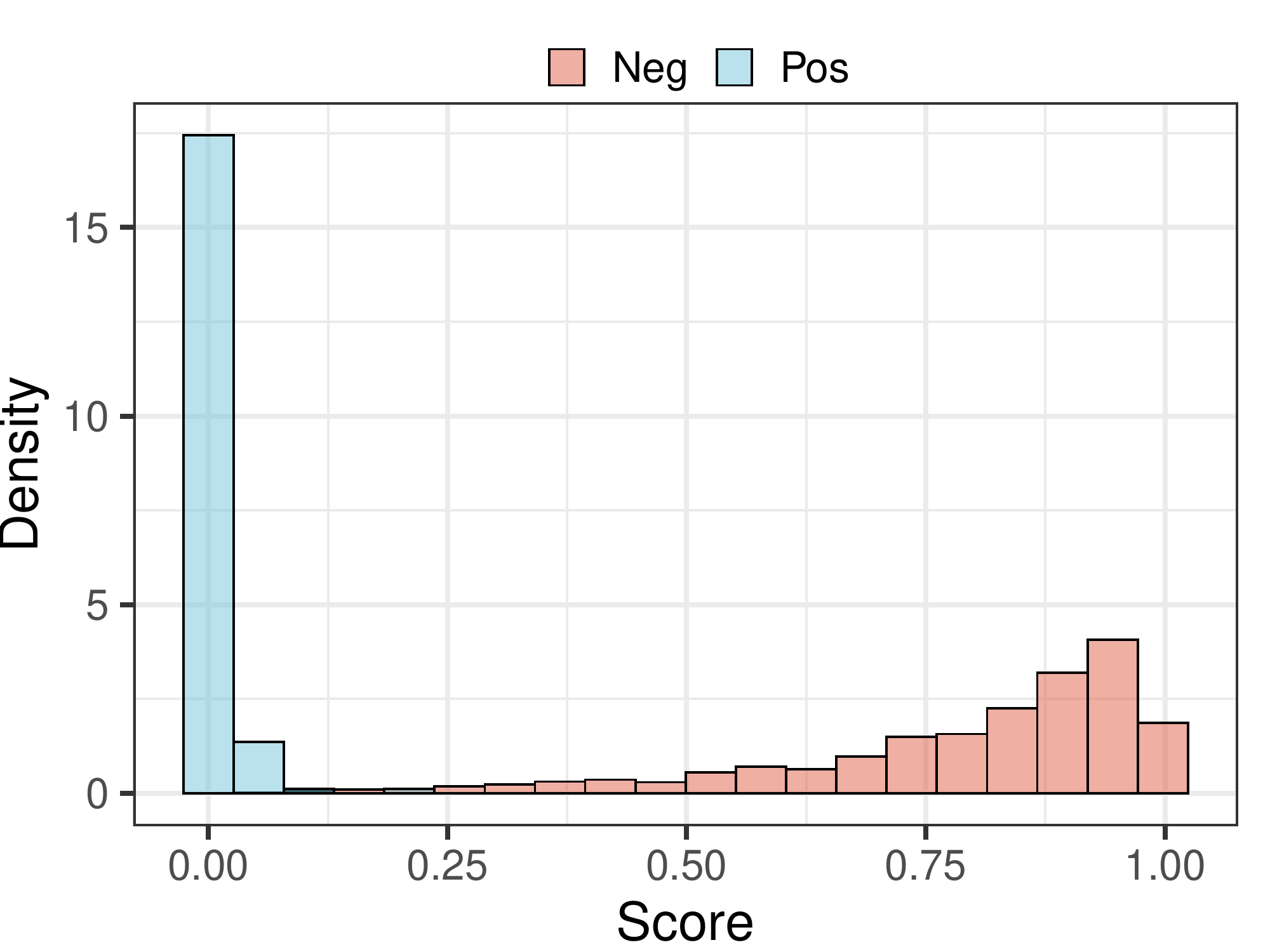}}
        
        \subfigure[Clean]{
        \includegraphics[width=0.24\textwidth]{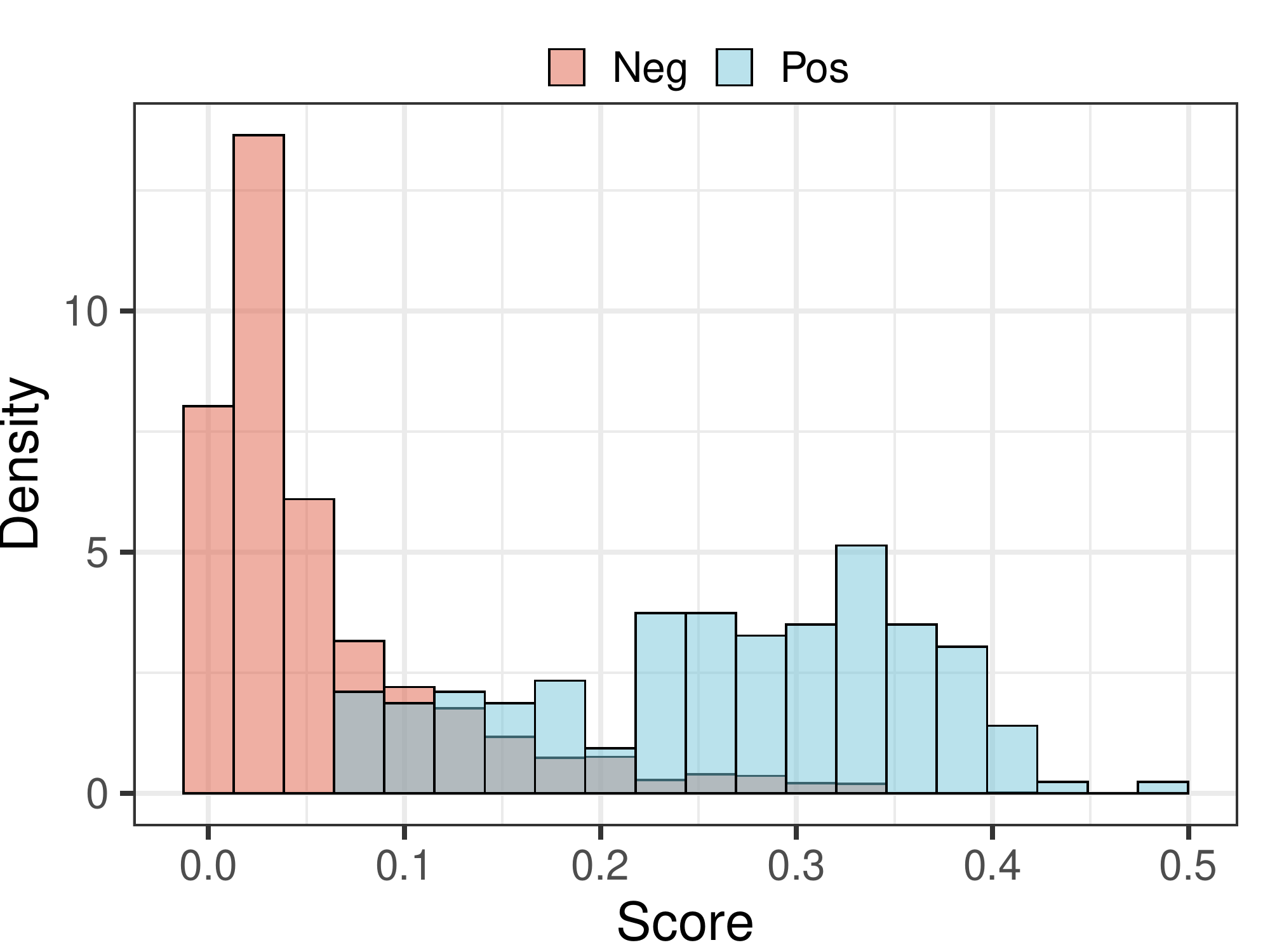}}
        \subfigure[FSGM]{
        \includegraphics[width=0.24\textwidth]{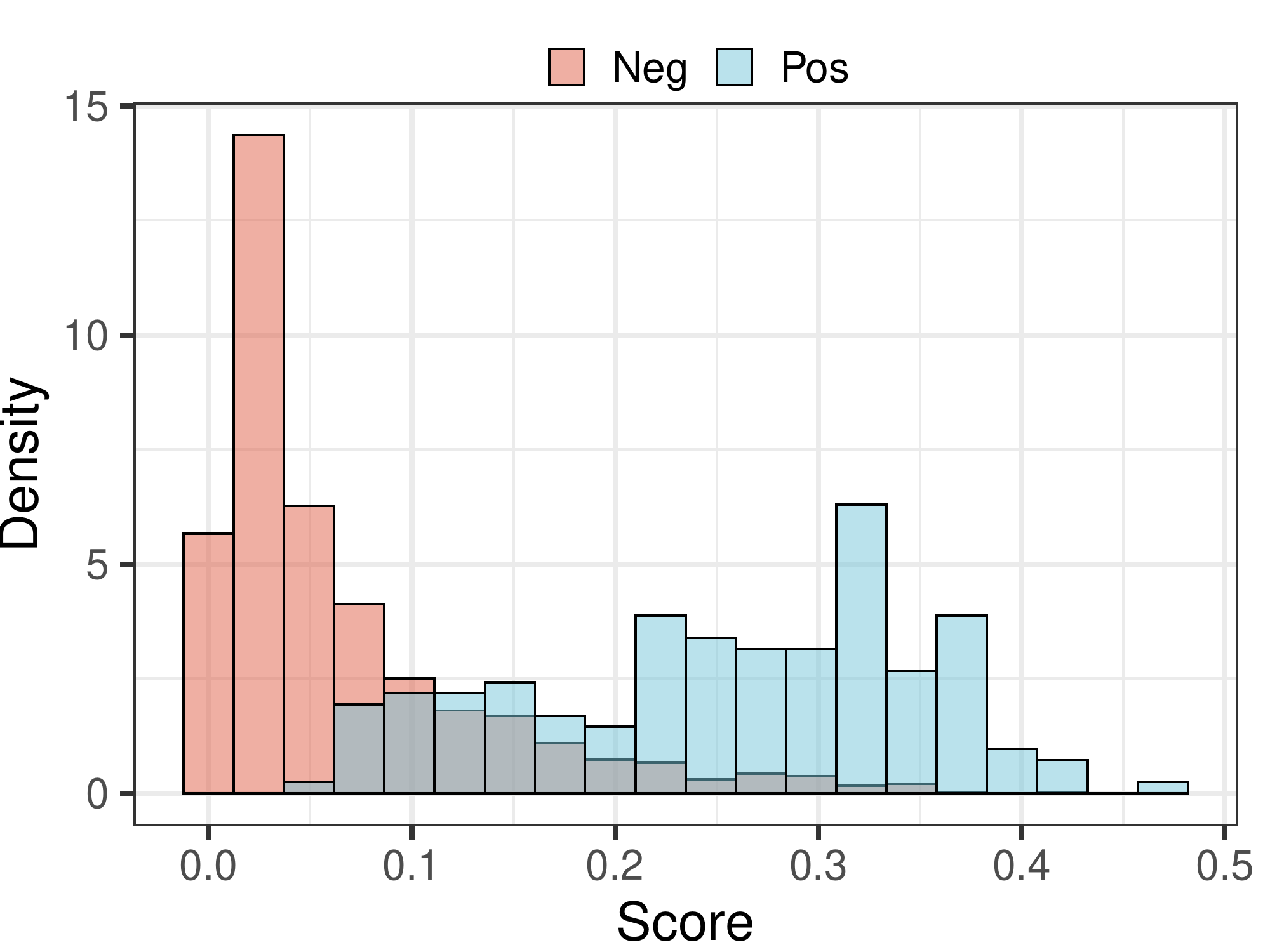}}
        \subfigure[PGD-10]{
        \includegraphics[width=0.24\textwidth]{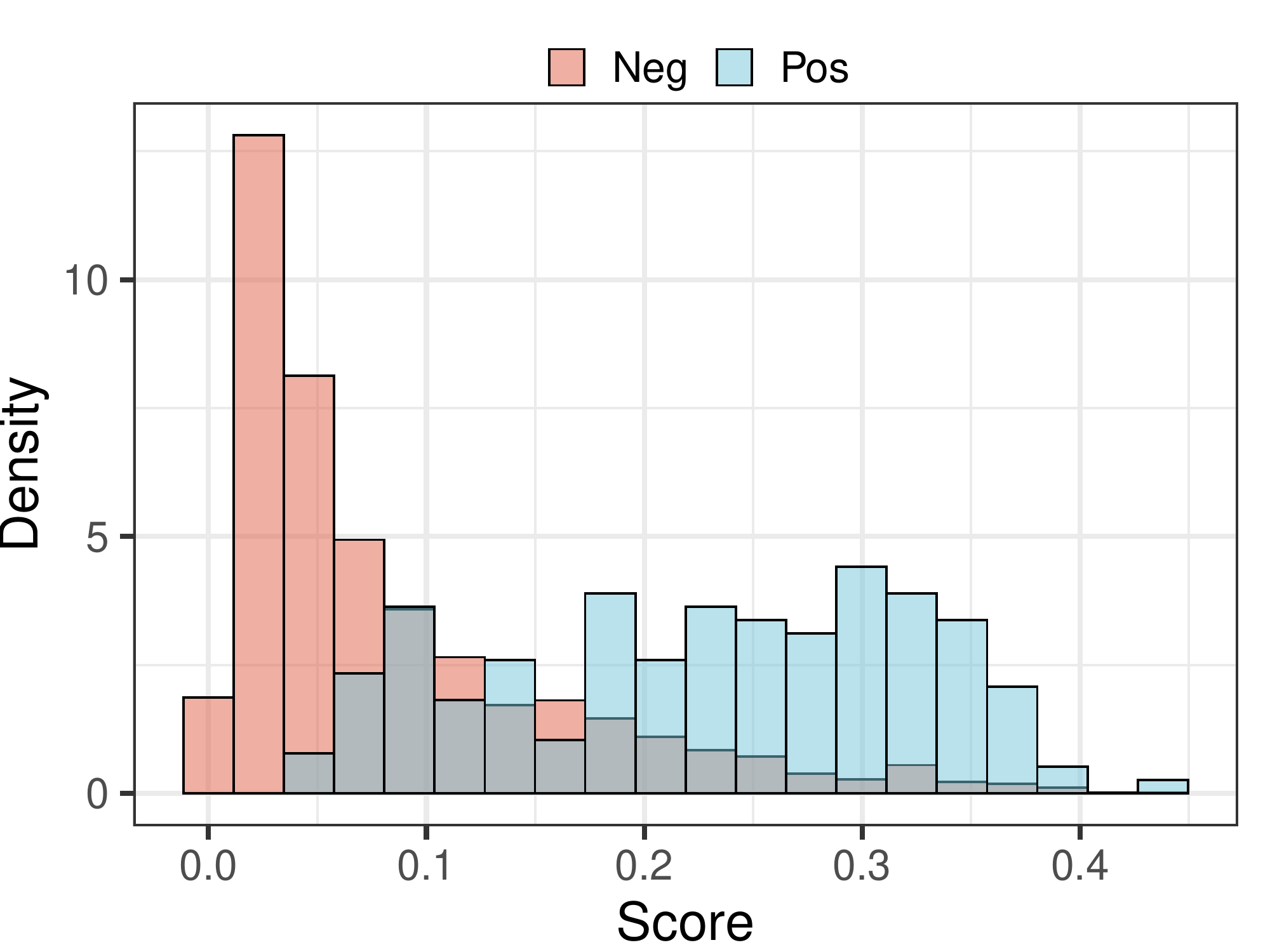}}
        \subfigure[PGD-20]{
        \includegraphics[width=0.24\textwidth]{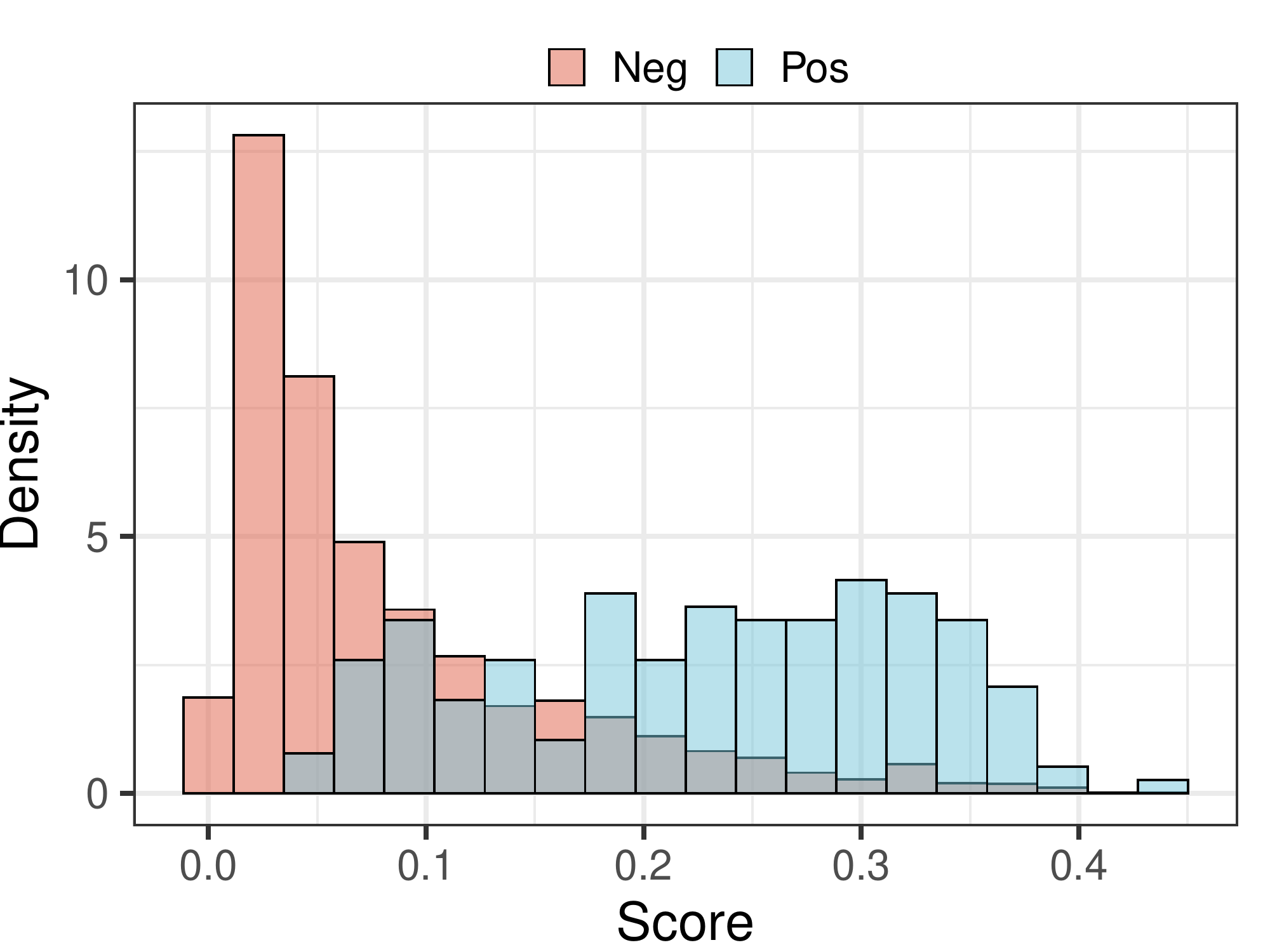}}
        
        \caption{Distribution of positive and negative example scores of CE on MNIST-LT dataset. The first row represents the score distribution against different attacks under Natural Training, and the second row represents the score distribution under Adversarial Training.}
        
        \label{Fig.Distribution.MNIST.ce}
    \end{figure*}
	
	\begin{figure*}[h!]
        \centering 
        \subfigure[Clean]{
        \includegraphics[width=0.24\textwidth]{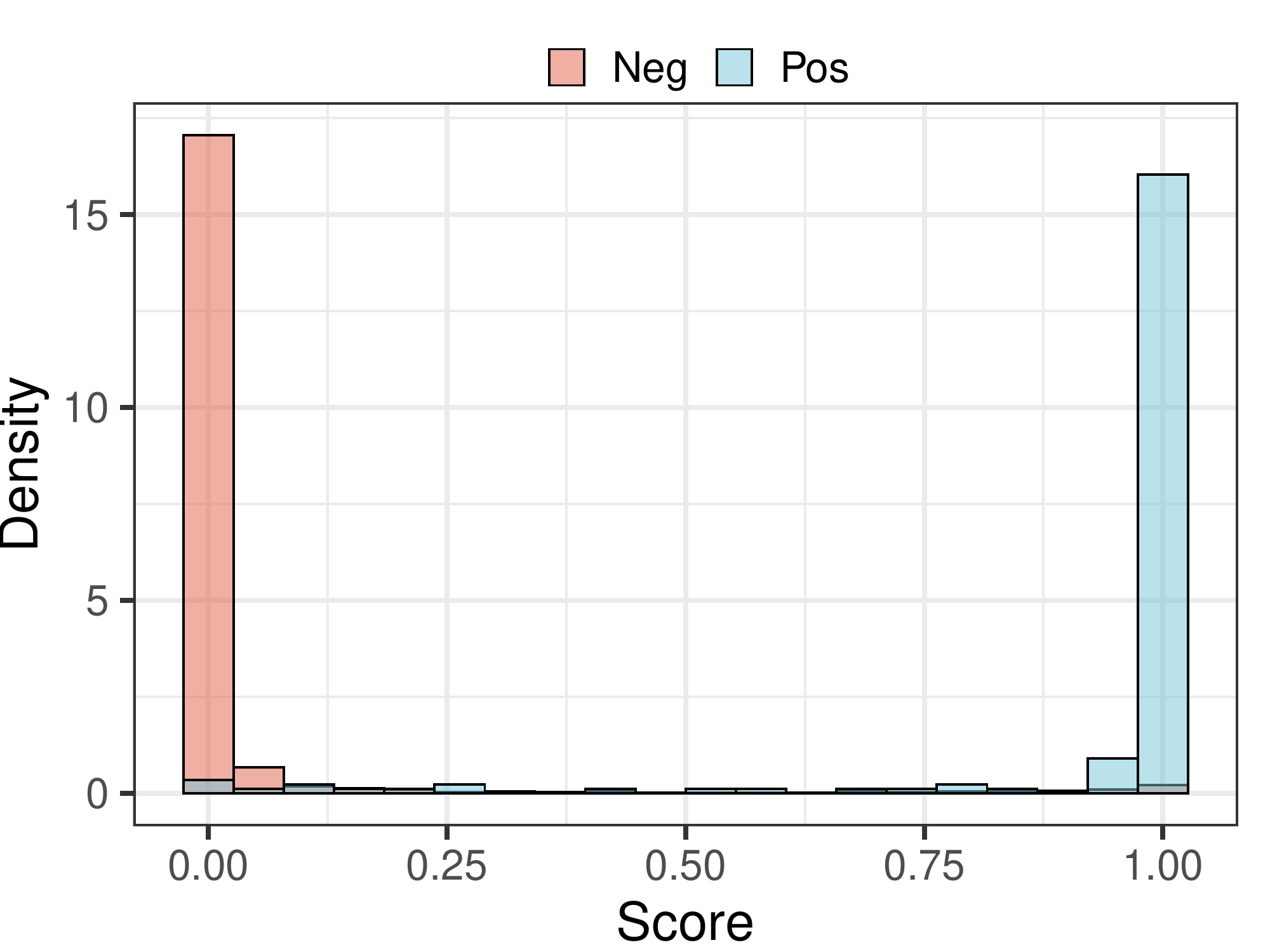}}
        \subfigure[FSGM]{
        \includegraphics[width=0.24\textwidth]{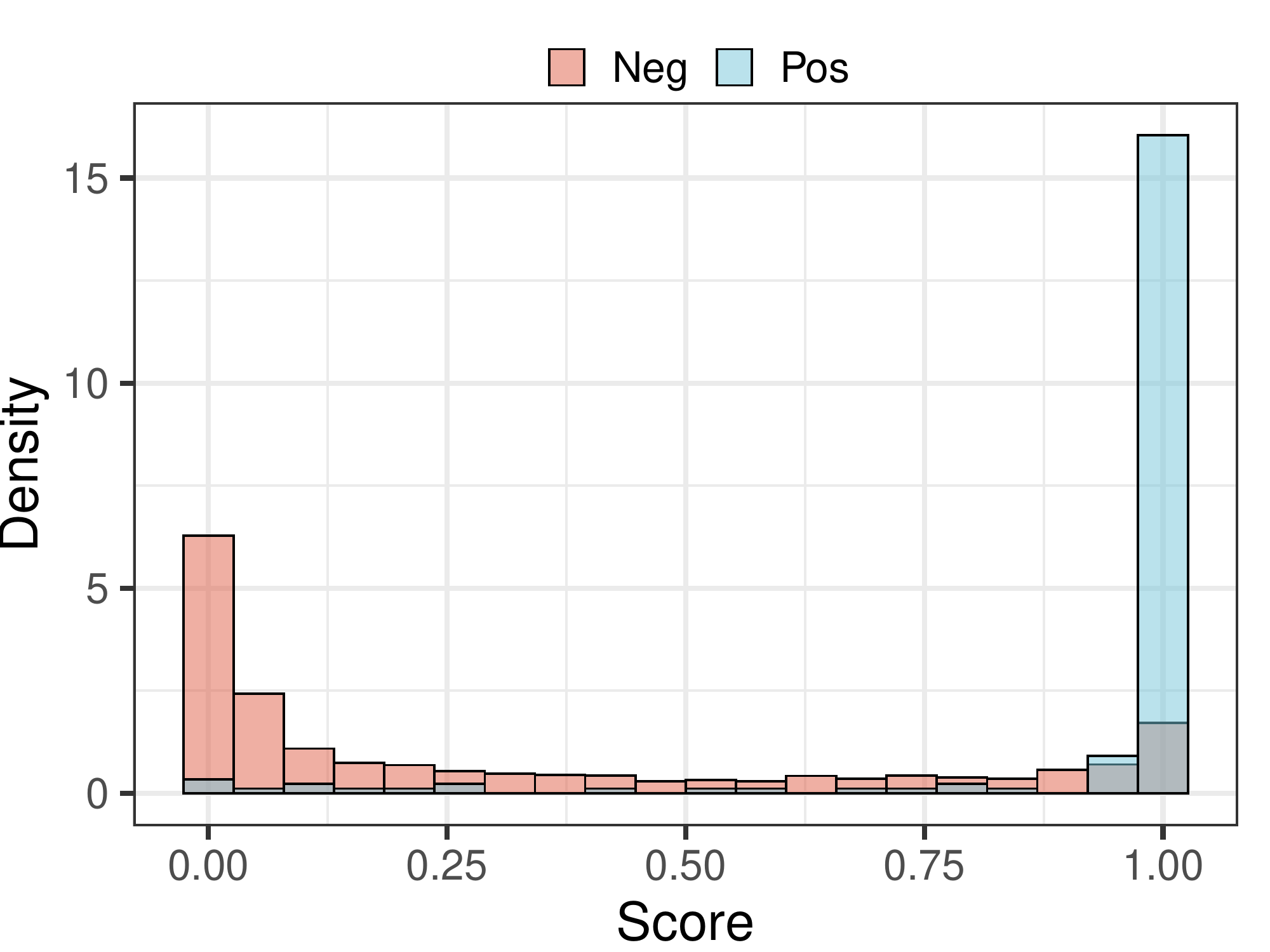}}
        \subfigure[PGD-10]{
        \includegraphics[width=0.24\textwidth]{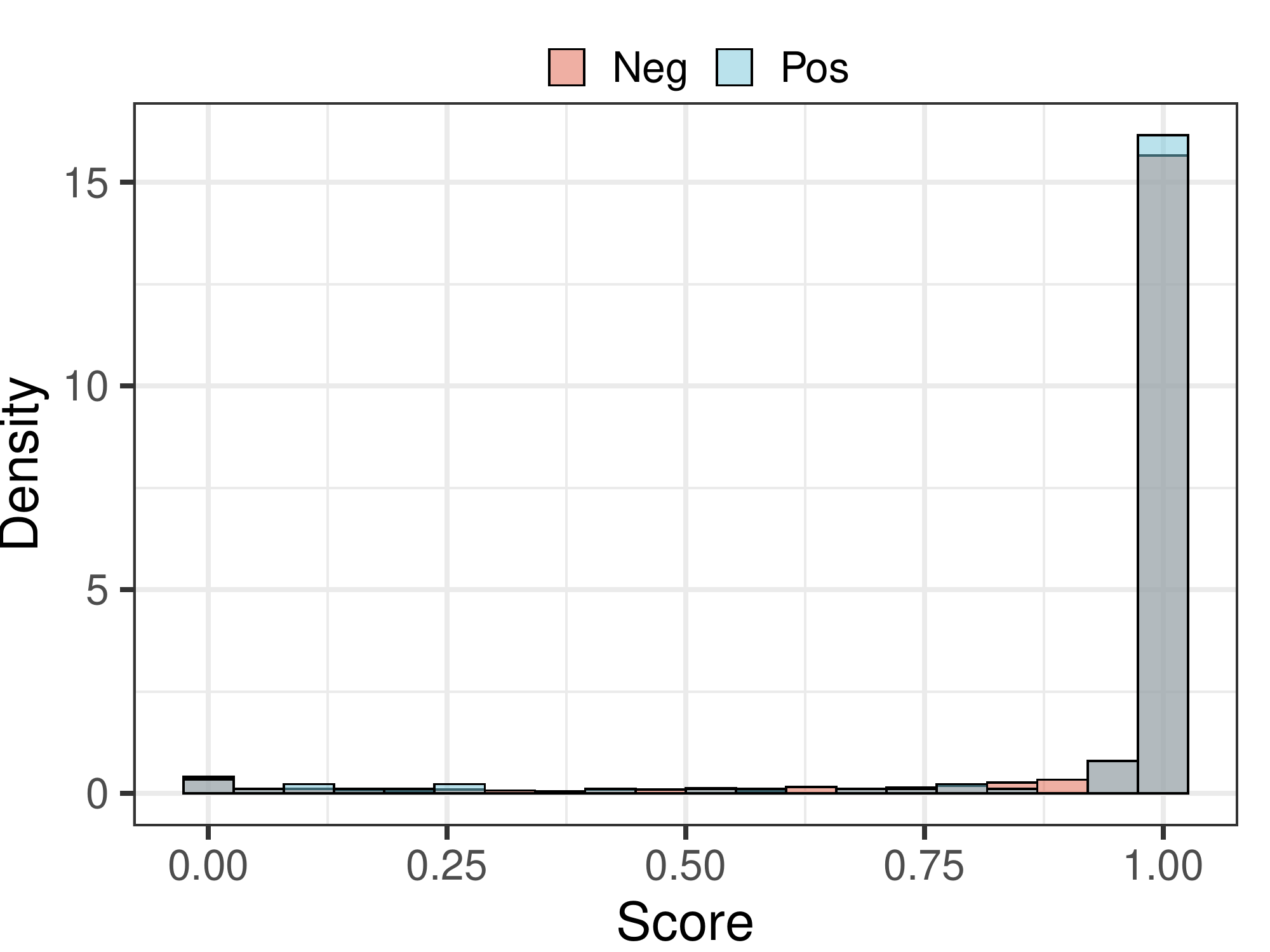}}
        \subfigure[PGD-20]{
        \includegraphics[width=0.24\textwidth]{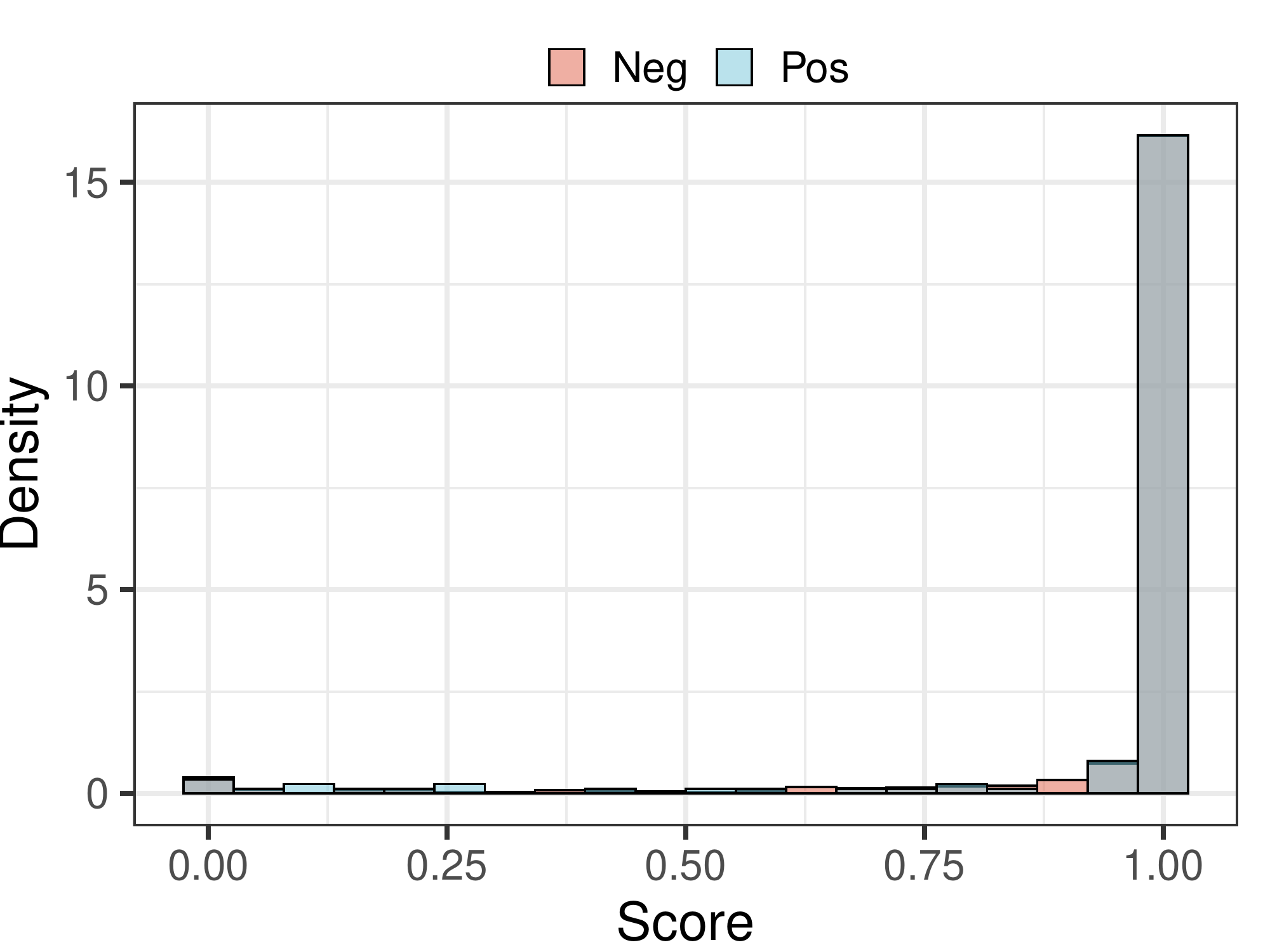}}
        
        \subfigure[Clean]{
        \includegraphics[width=0.24\textwidth]{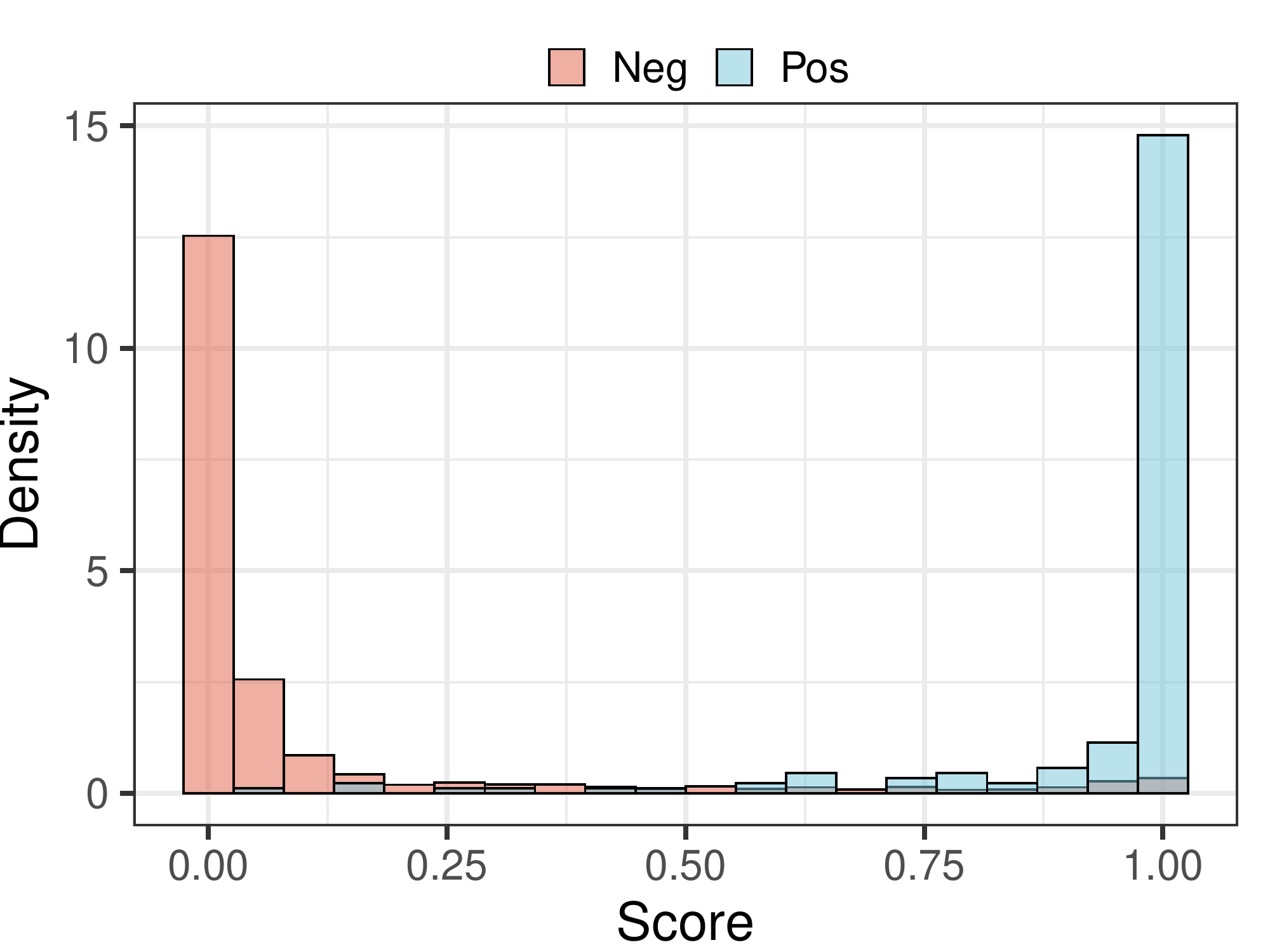}}
        \subfigure[FSGM]{
        \includegraphics[width=0.24\textwidth]{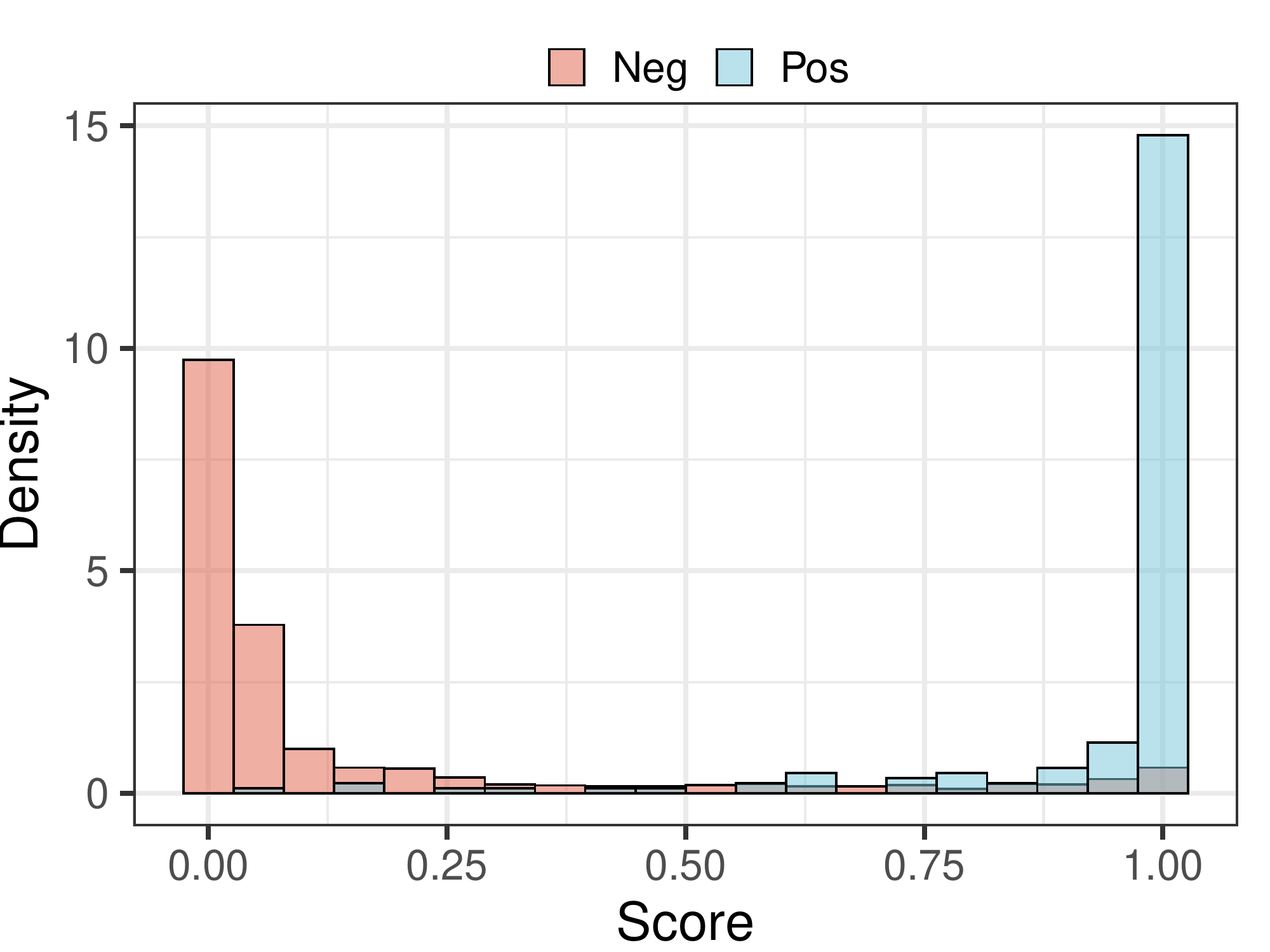}}
        \subfigure[PGD-10]{
        \includegraphics[width=0.24\textwidth]{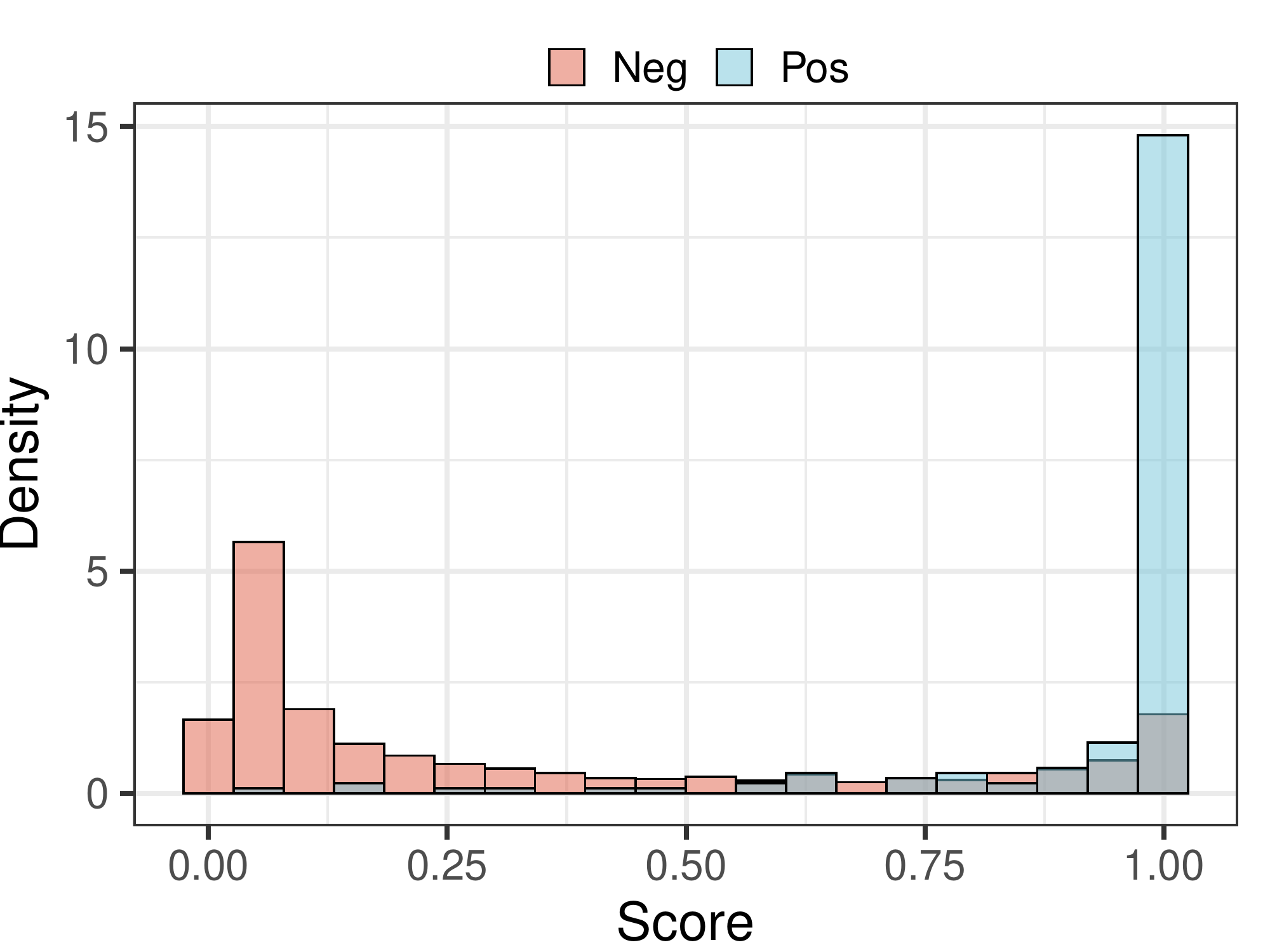}}
        \subfigure[PGD-20]{
        \includegraphics[width=0.24\textwidth]{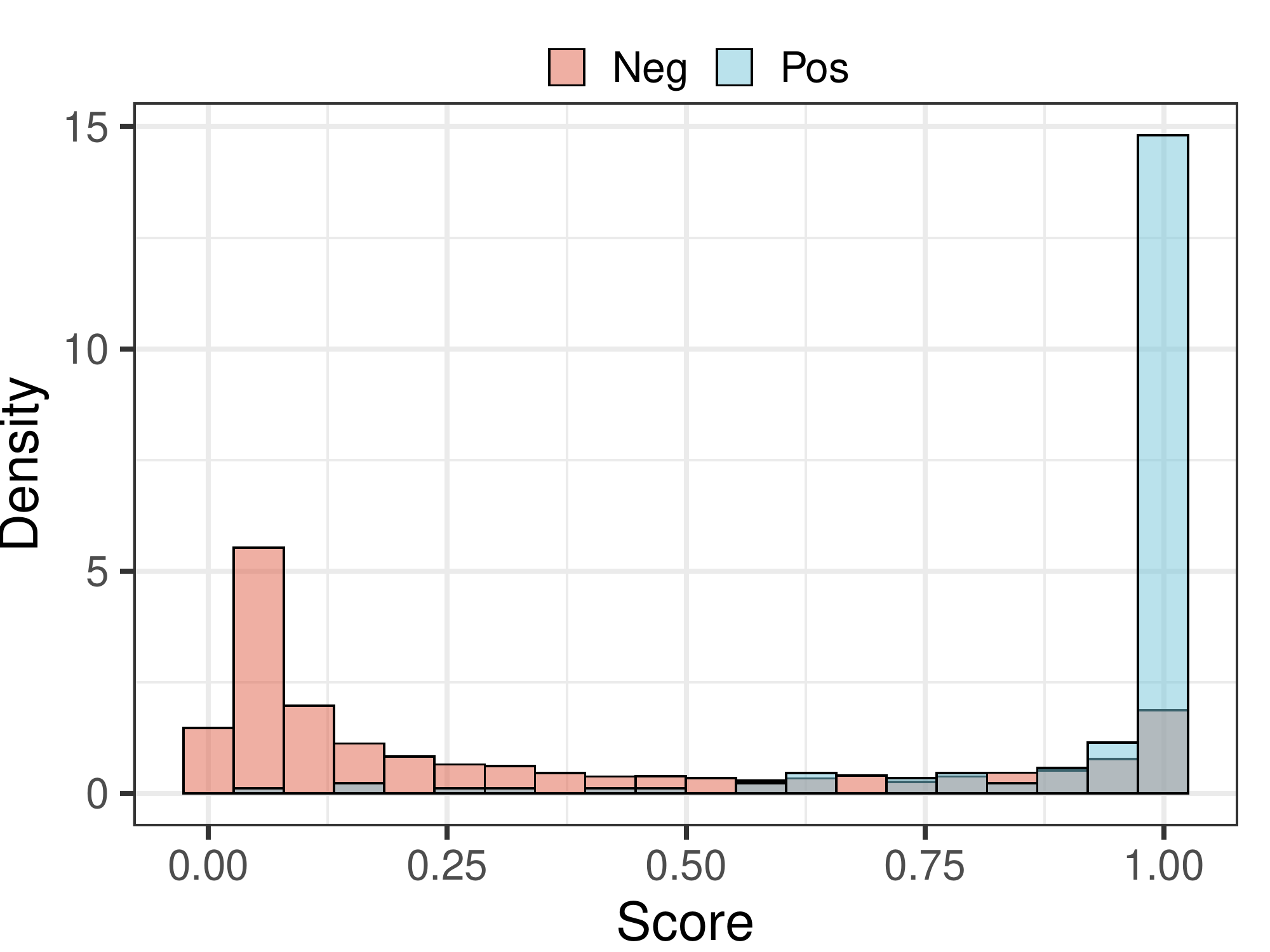}}
        
        \caption{Distribution of positive and negative example scores of our proposed AdAUC on MNIST-LT dataset. The first row represents the score distribution against different attacks under Natural Training, and the second row represents the score distribution under Adversarial Training.}
        \label{Fig.Distribution.MNIST.auc}
    \end{figure*}
    
    \begin{figure*}[h!]
        \subfigure[Clean]{
        \includegraphics[width=0.24\textwidth]{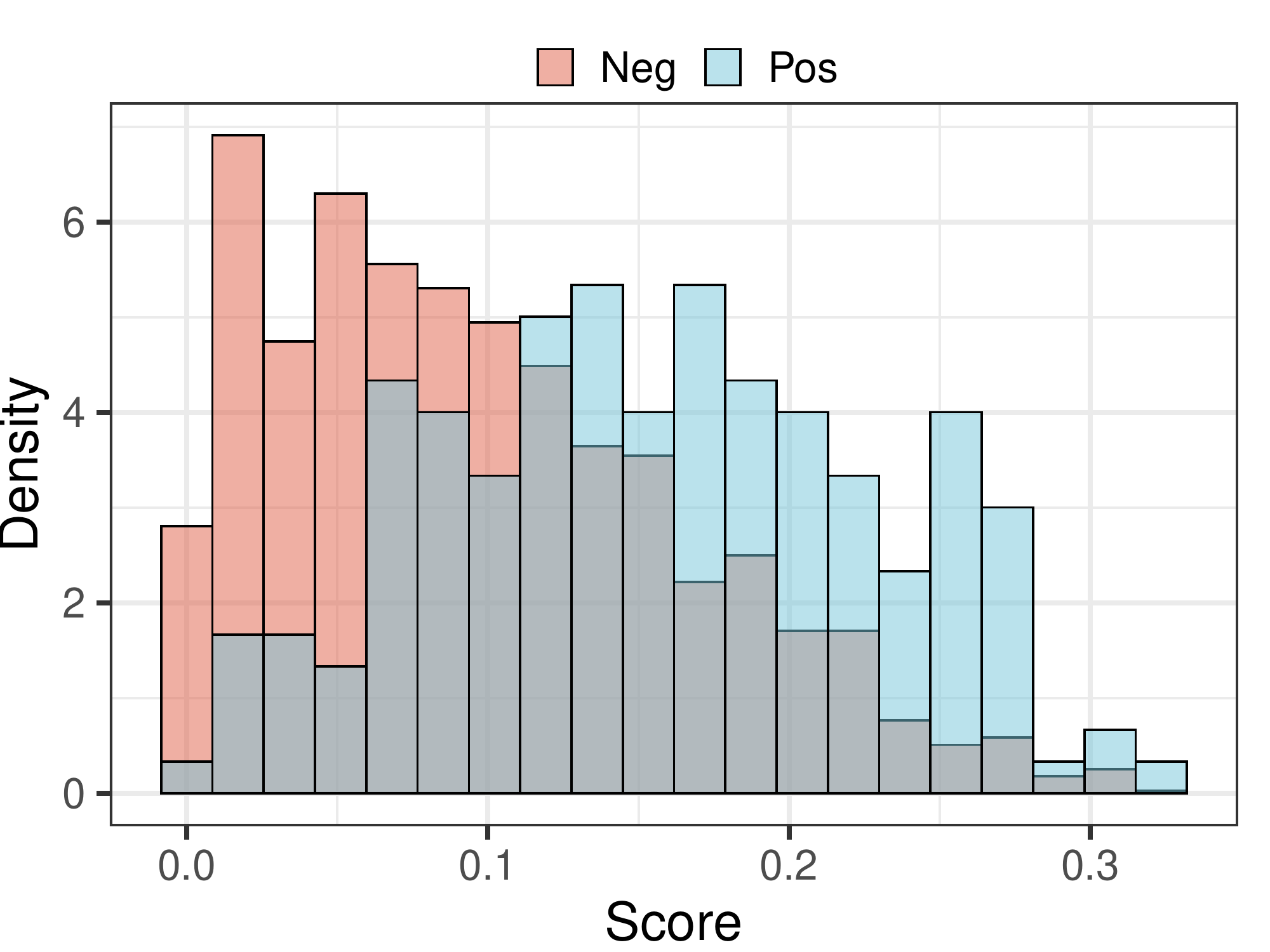}}
        \subfigure[FSGM]{
        \includegraphics[width=0.24\textwidth]{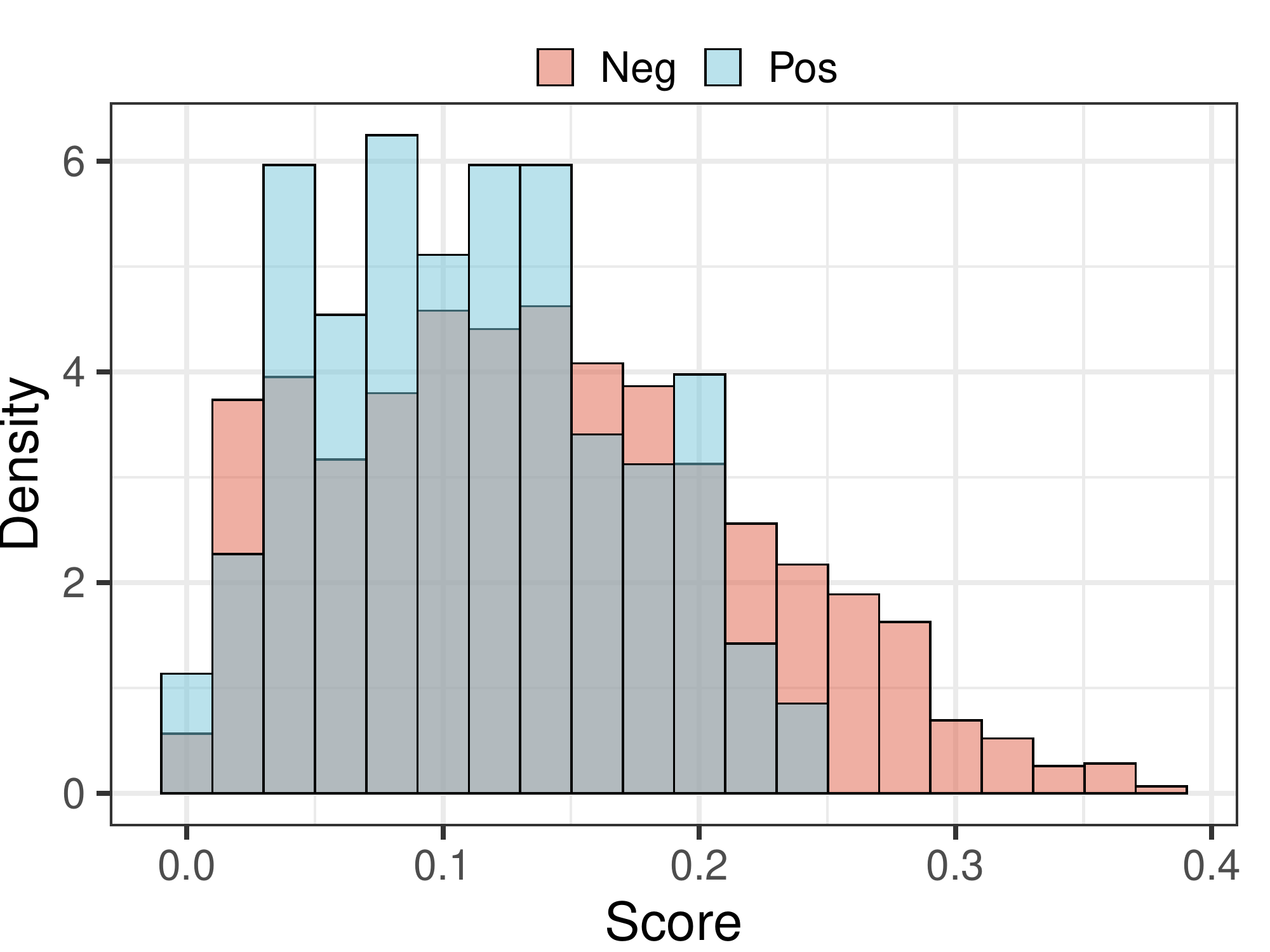}}
        \subfigure[PGD-10]{
        \includegraphics[width=0.24\textwidth]{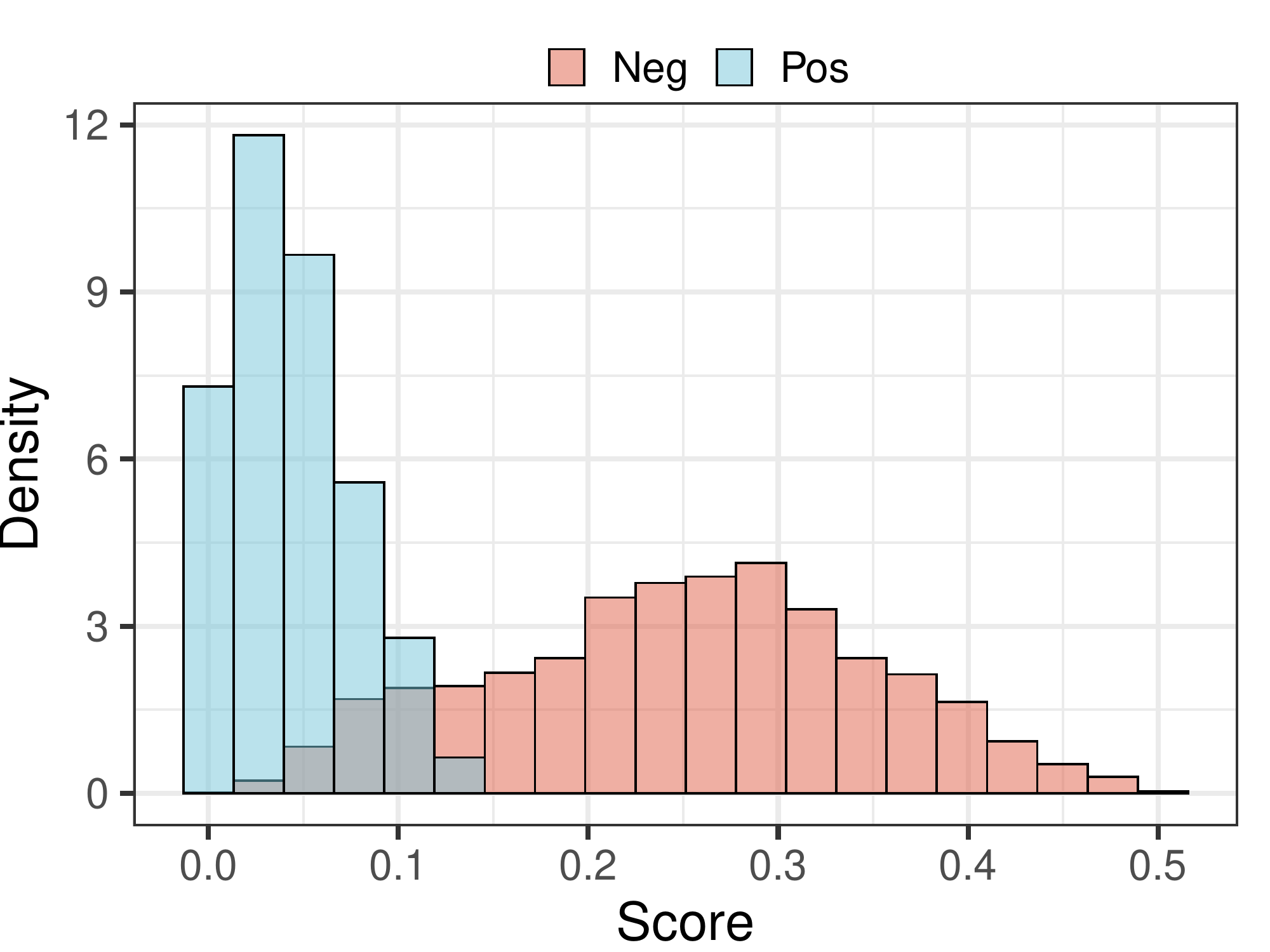}}
        \subfigure[PGD-20]{
        \includegraphics[width=0.24\textwidth]{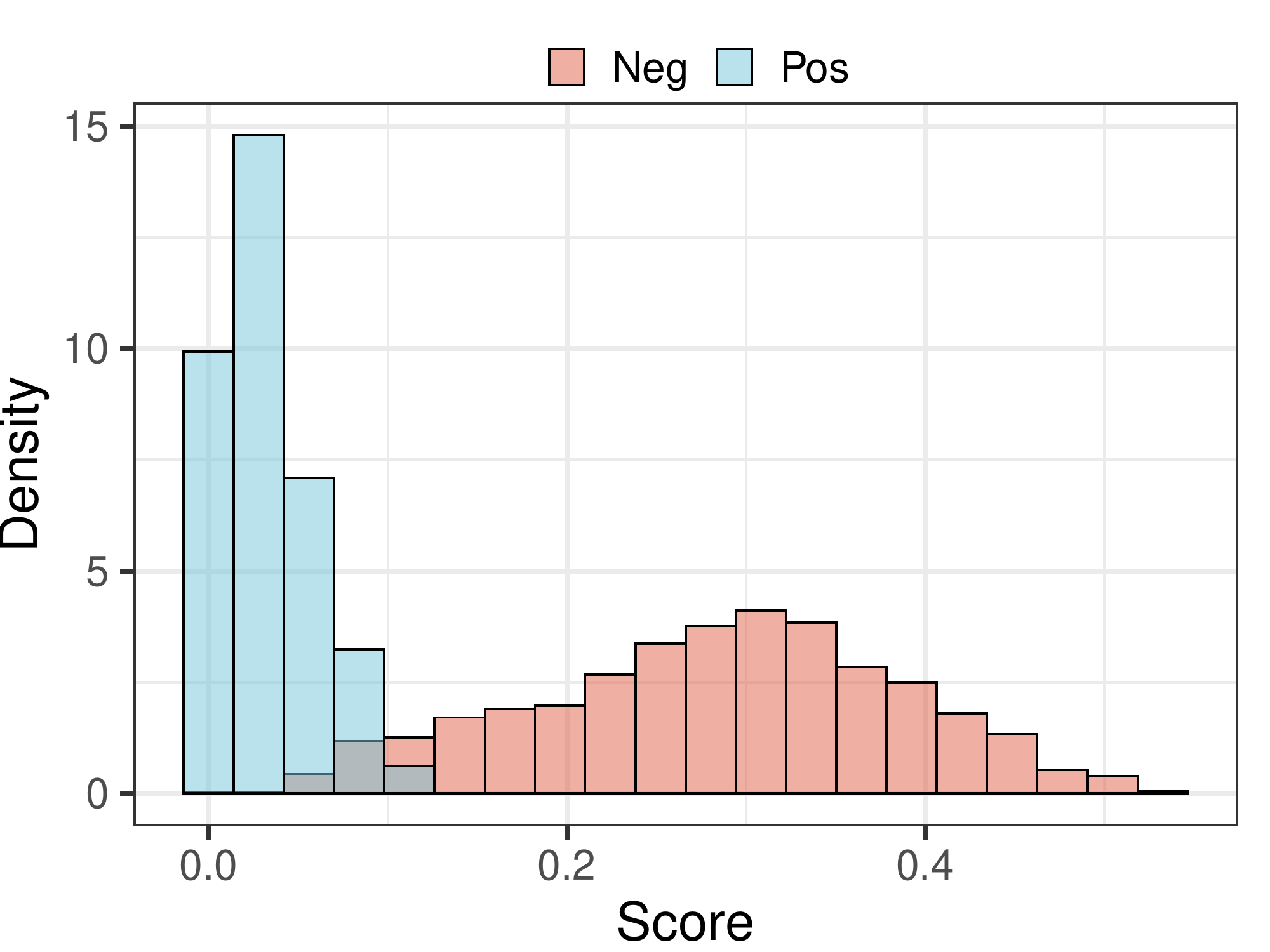}}
        
        \subfigure[Clean]{
        \includegraphics[width=0.24\textwidth]{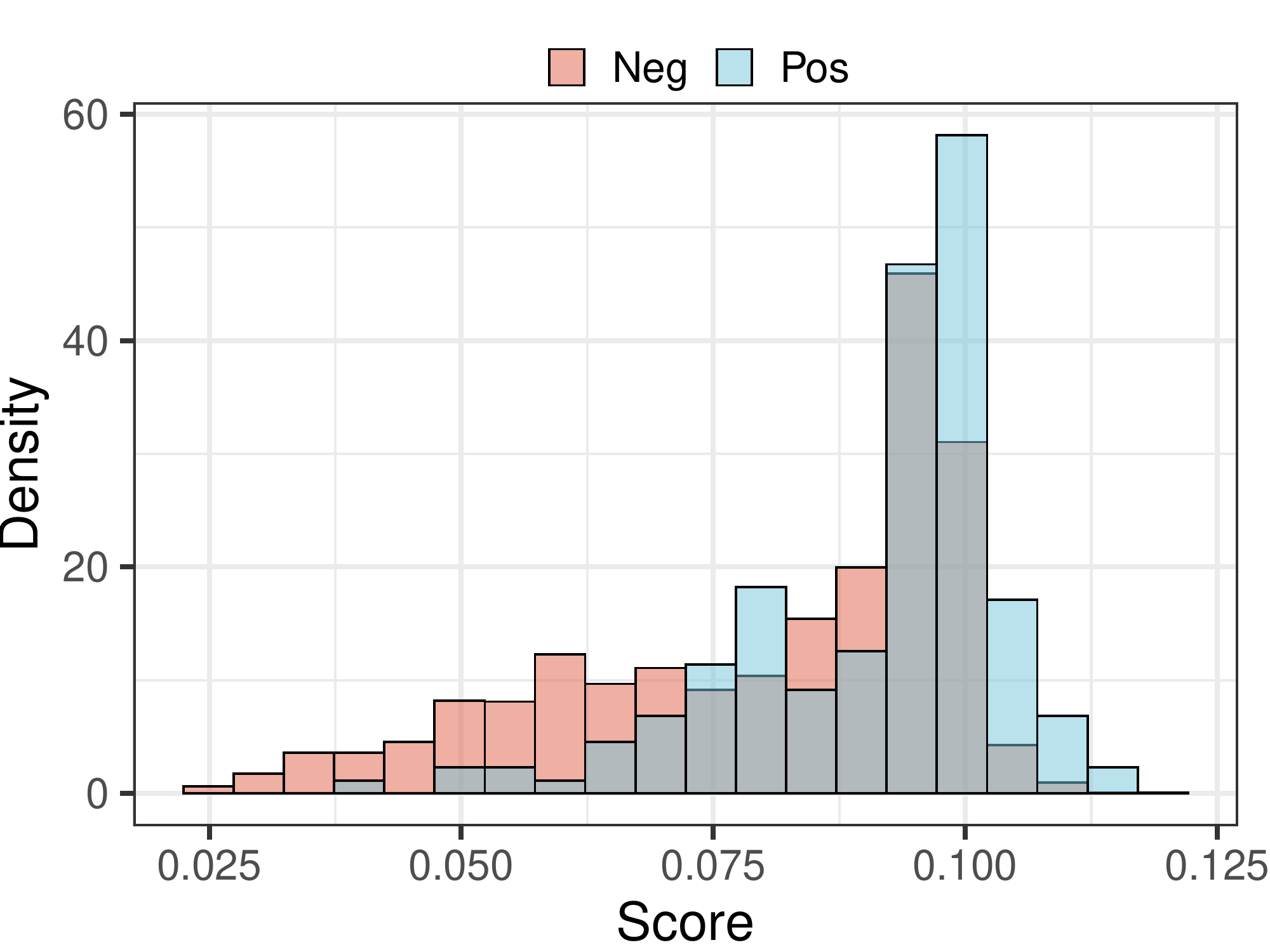}}
        \subfigure[FSGM]{
        \includegraphics[width=0.24\textwidth]{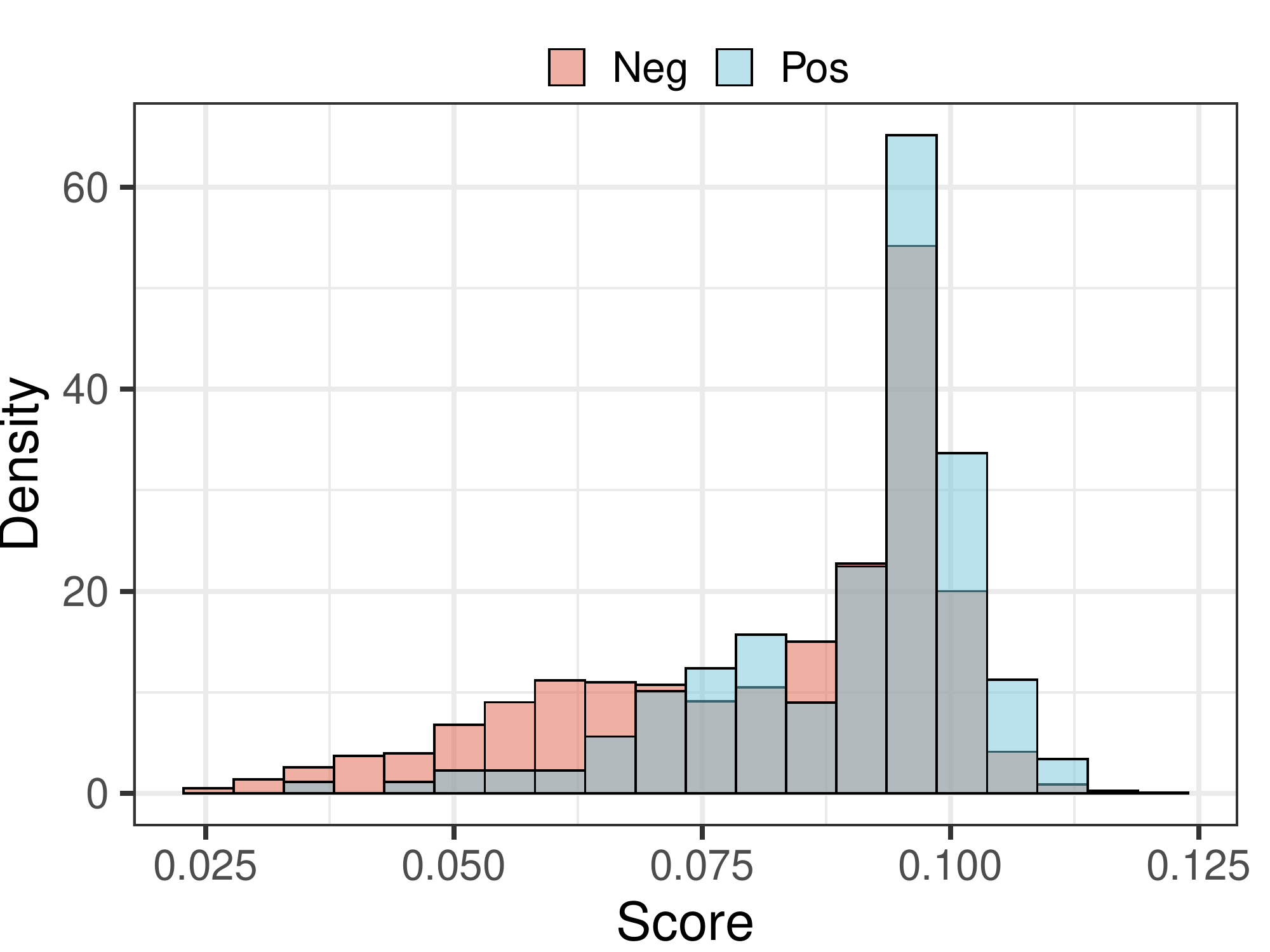}}
        \subfigure[PGD-10]{
        \includegraphics[width=0.24\textwidth]{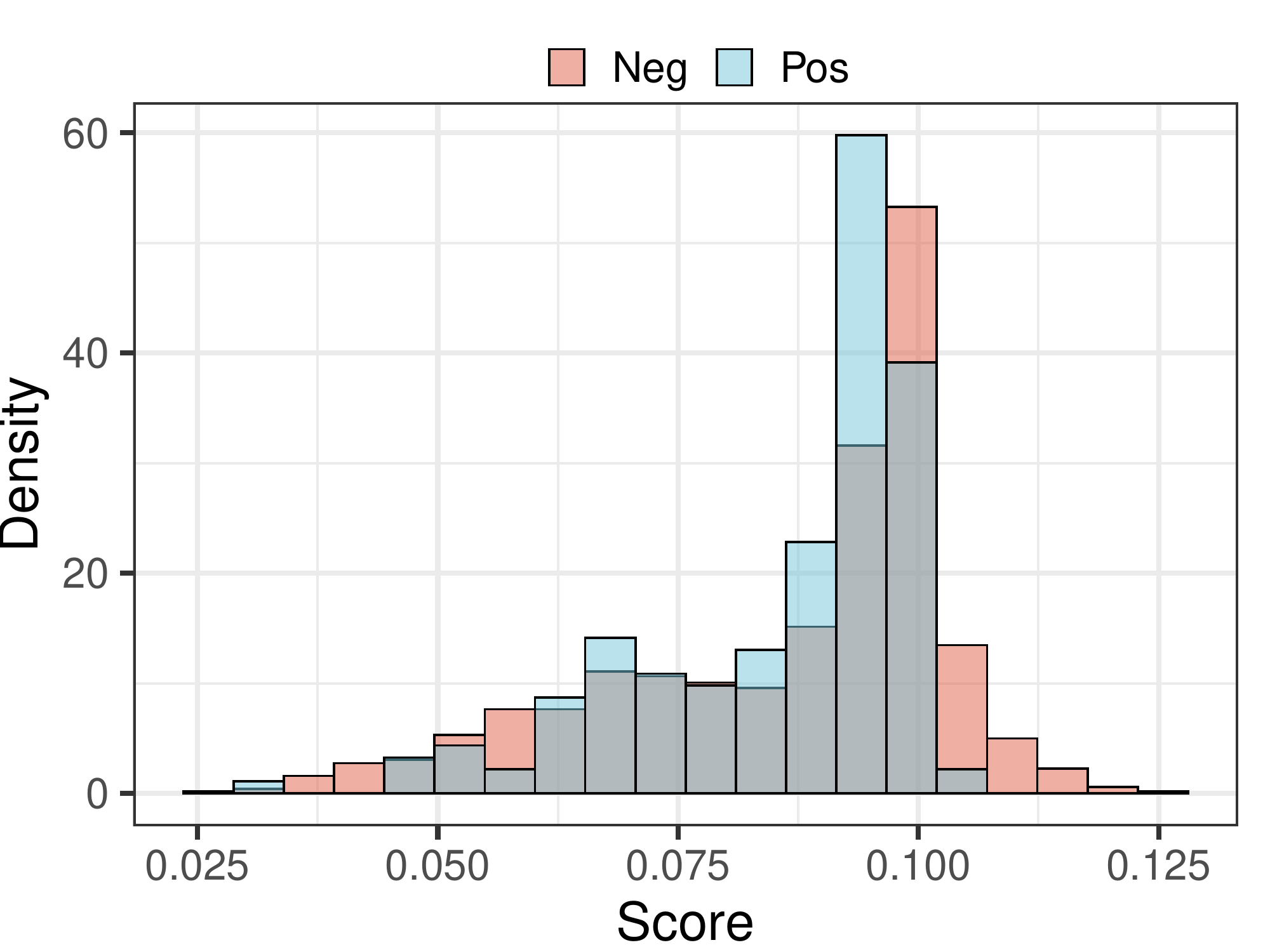}}
        \subfigure[PGD-20]{
        \includegraphics[width=0.24\textwidth]{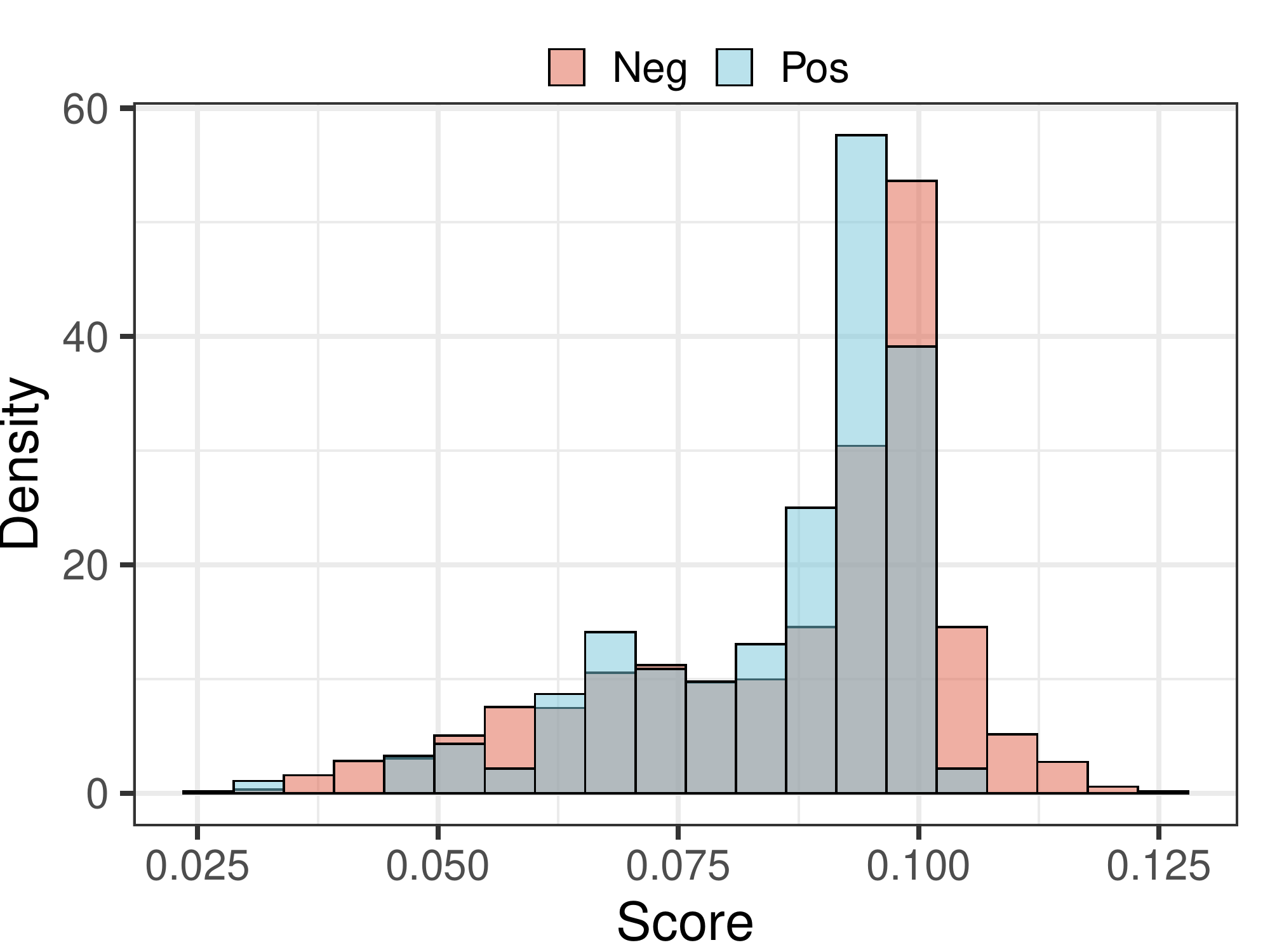}}
        
        \caption{Distribution of positive and negative example scores of CE on CIFAR-10-LT dataset. The first row represents the score distribution against different attacks under Natural Training, and the second row represents the score distribution under Adversarial Training.}
        
        \label{Fig.Distribution.CIFAR10.ce}
    \end{figure*}
    
    \begin{figure*}[h!]
        \subfigure[Clean]{
        \includegraphics[width=0.24\textwidth]{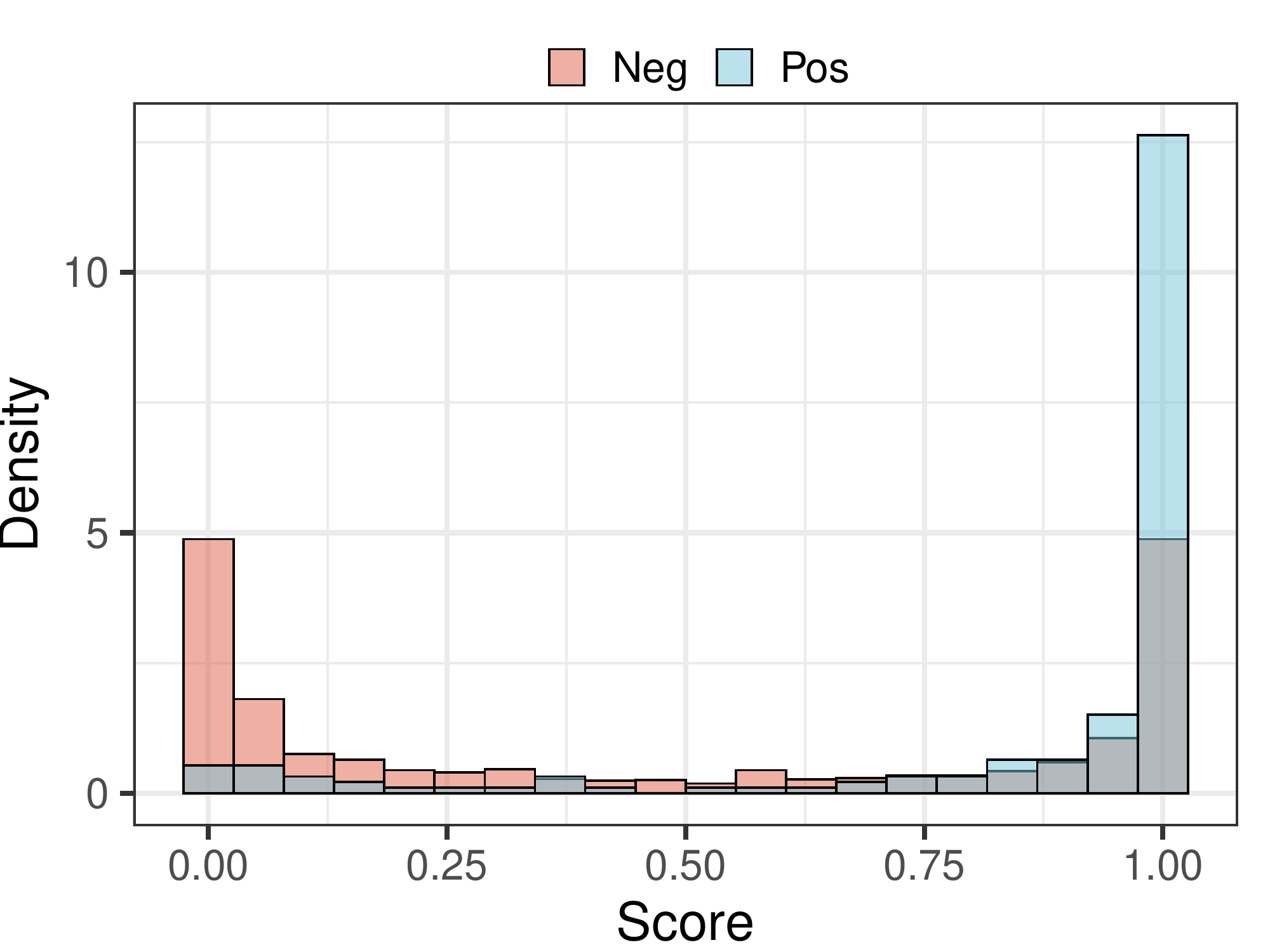}}
        \subfigure[FSGM]{
        \includegraphics[width=0.24\textwidth]{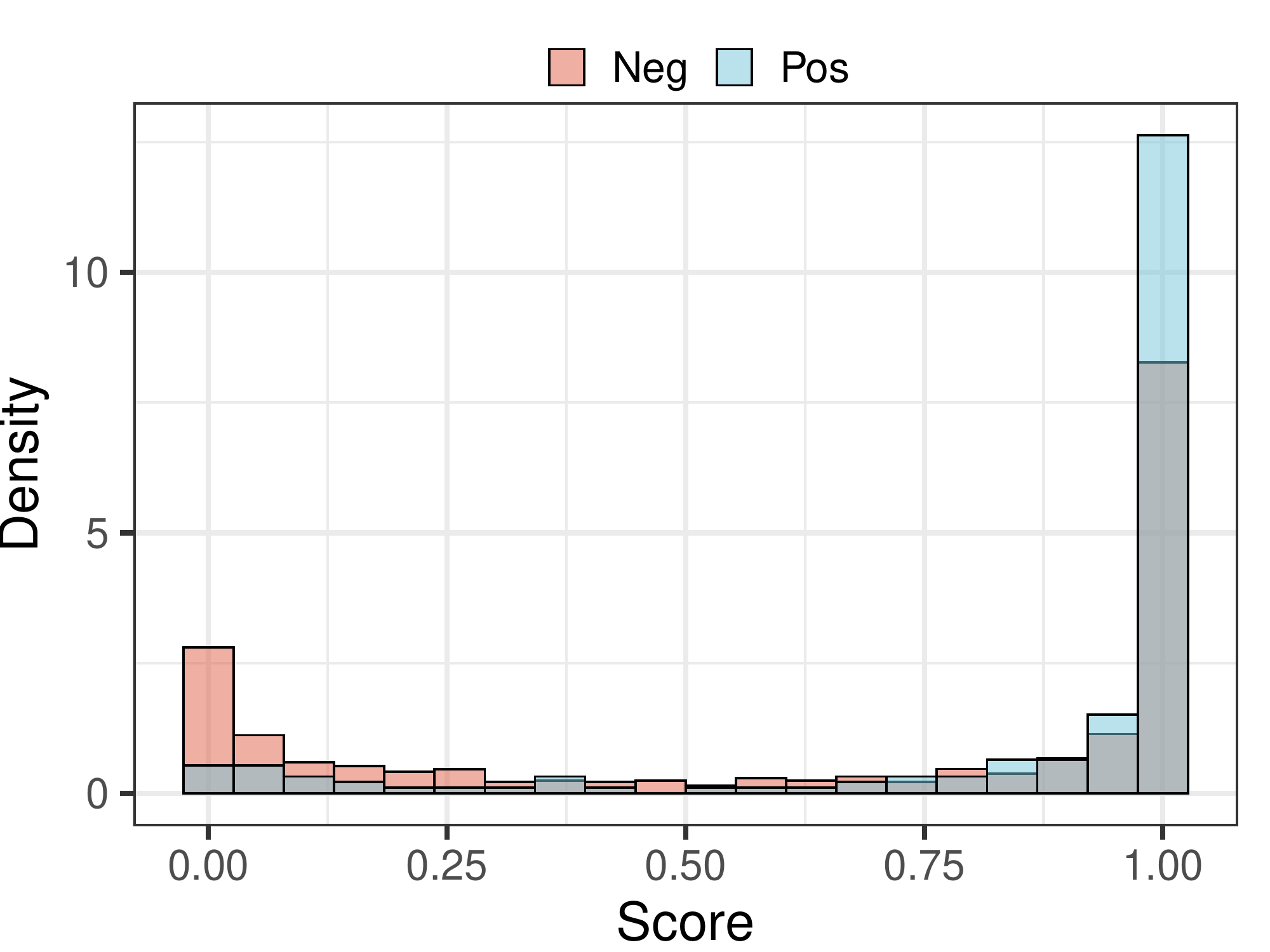}}
        \subfigure[PGD-10]{
        \includegraphics[width=0.24\textwidth]{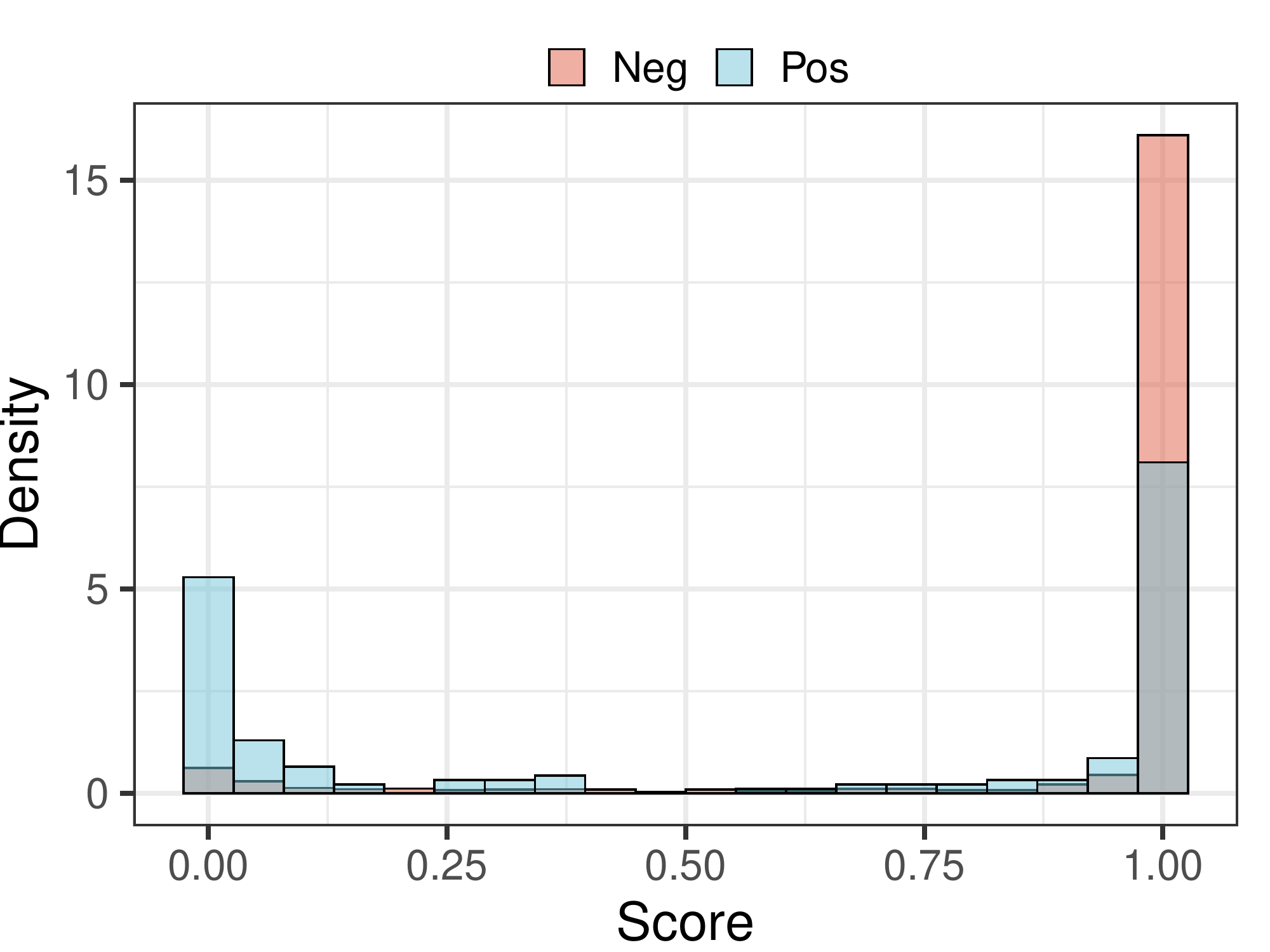}}
        \subfigure[PGD-20]{
        \includegraphics[width=0.24\textwidth]{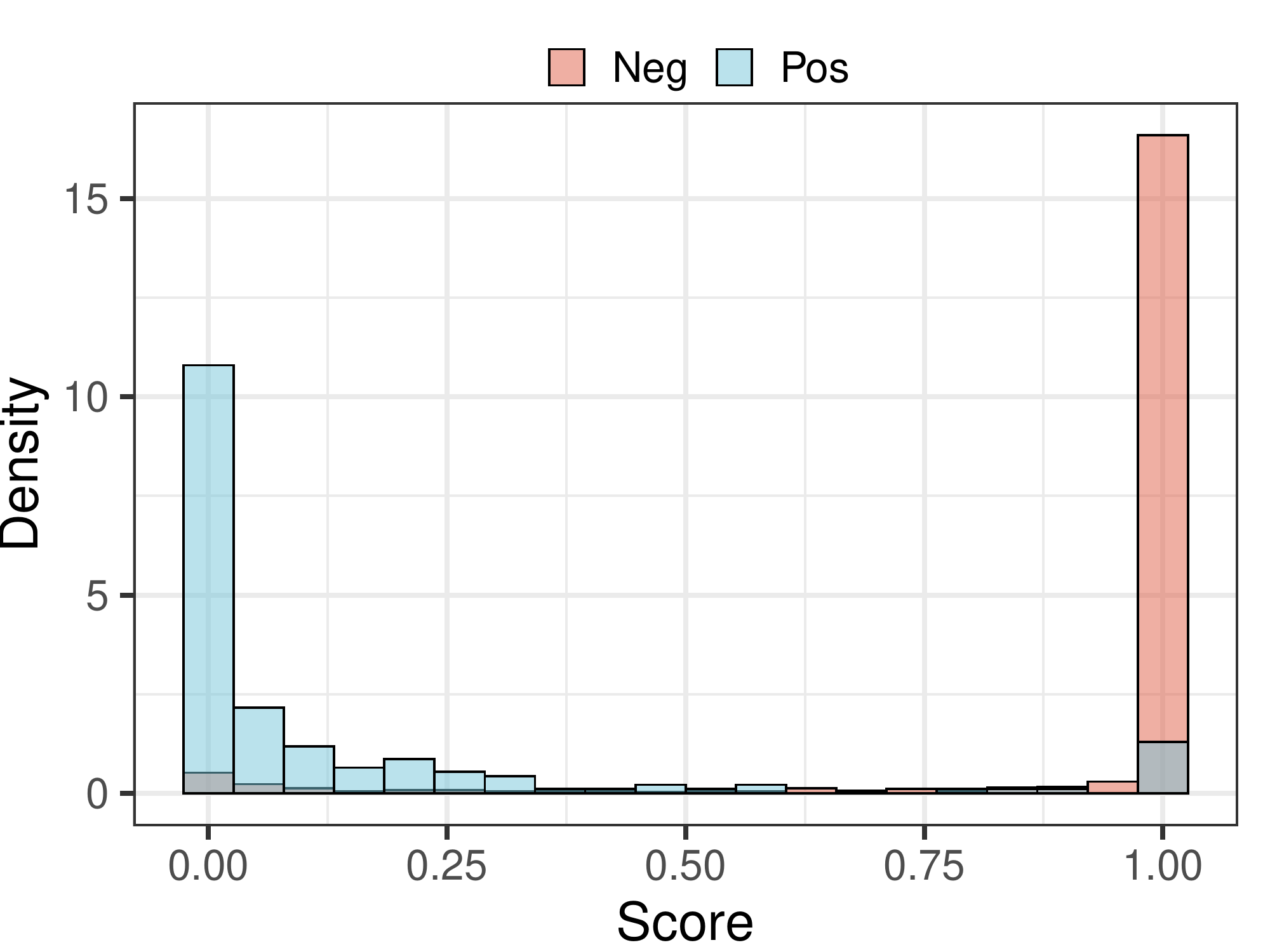}}
        
        \subfigure[Clean]{
        \includegraphics[width=0.24\textwidth]{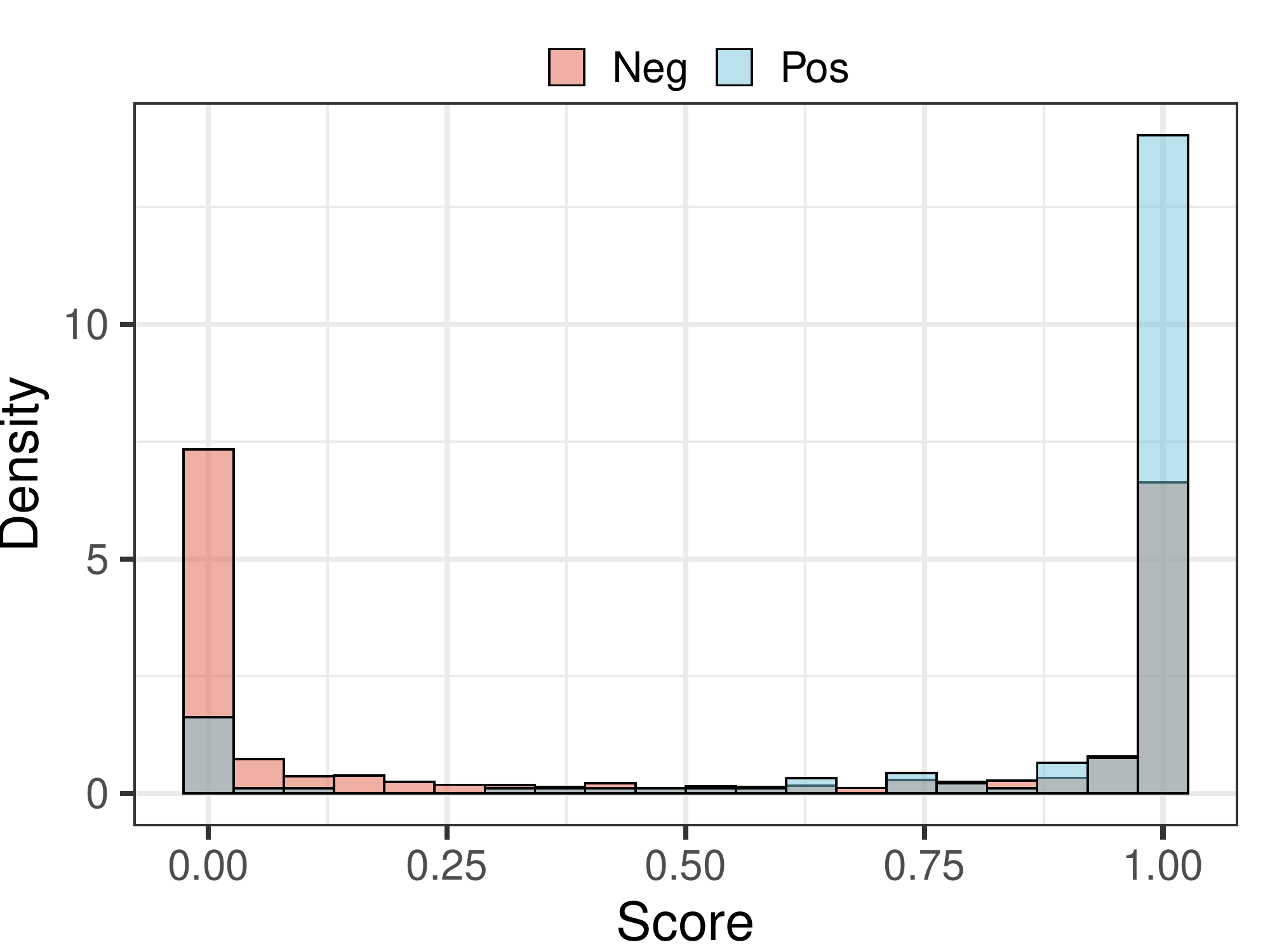}}
        \subfigure[FSGM]{
        \includegraphics[width=0.24\textwidth]{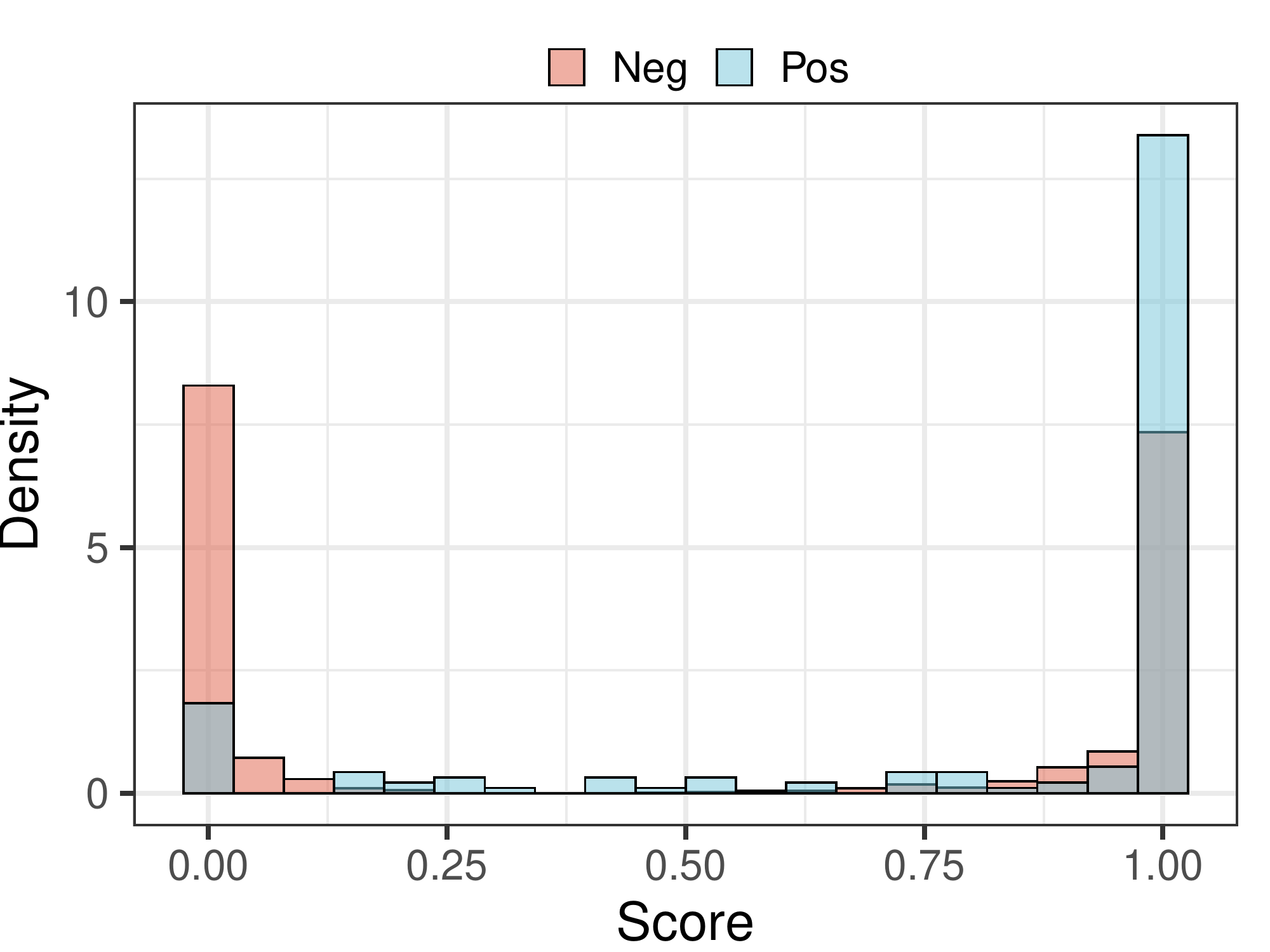}}
        \subfigure[PGD-10]{
        \includegraphics[width=0.24\textwidth]{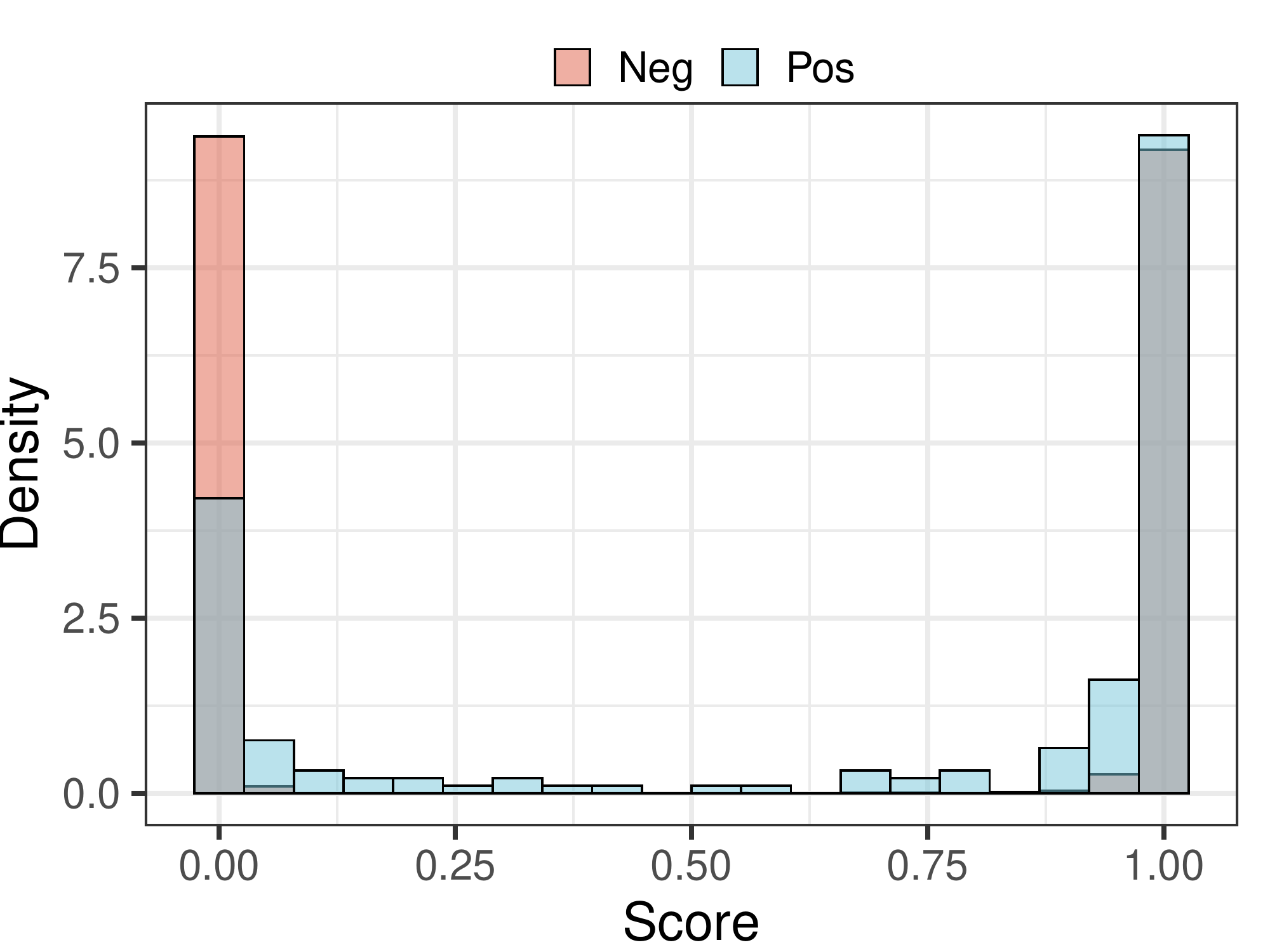}}
        \subfigure[PGD-20]{
        \includegraphics[width=0.24\textwidth]{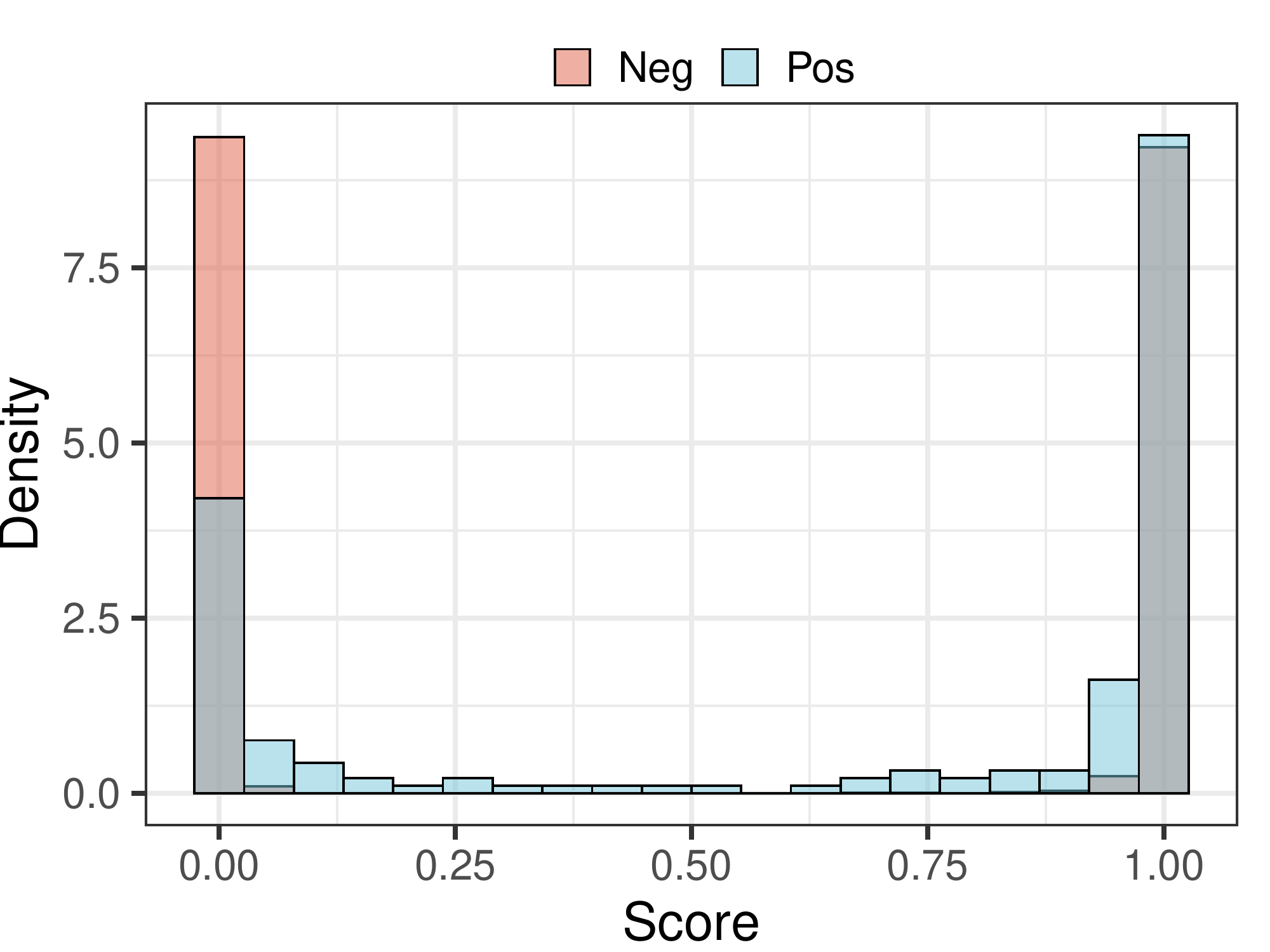}}
        
        \caption{Distribution of positive and negative example scores of AdAUC on CIFAR-10-LT dataset. The first row represents the score distribution against different attacks under Natural Training, and the second row represents the score distribution under Adversarial Training.}
        
        \label{Fig.Distribution.CIFAR10.auc}
    \end{figure*}

    \begin{figure*}[h!]
        \subfigure[Clean]{
        \includegraphics[width=0.24\textwidth]{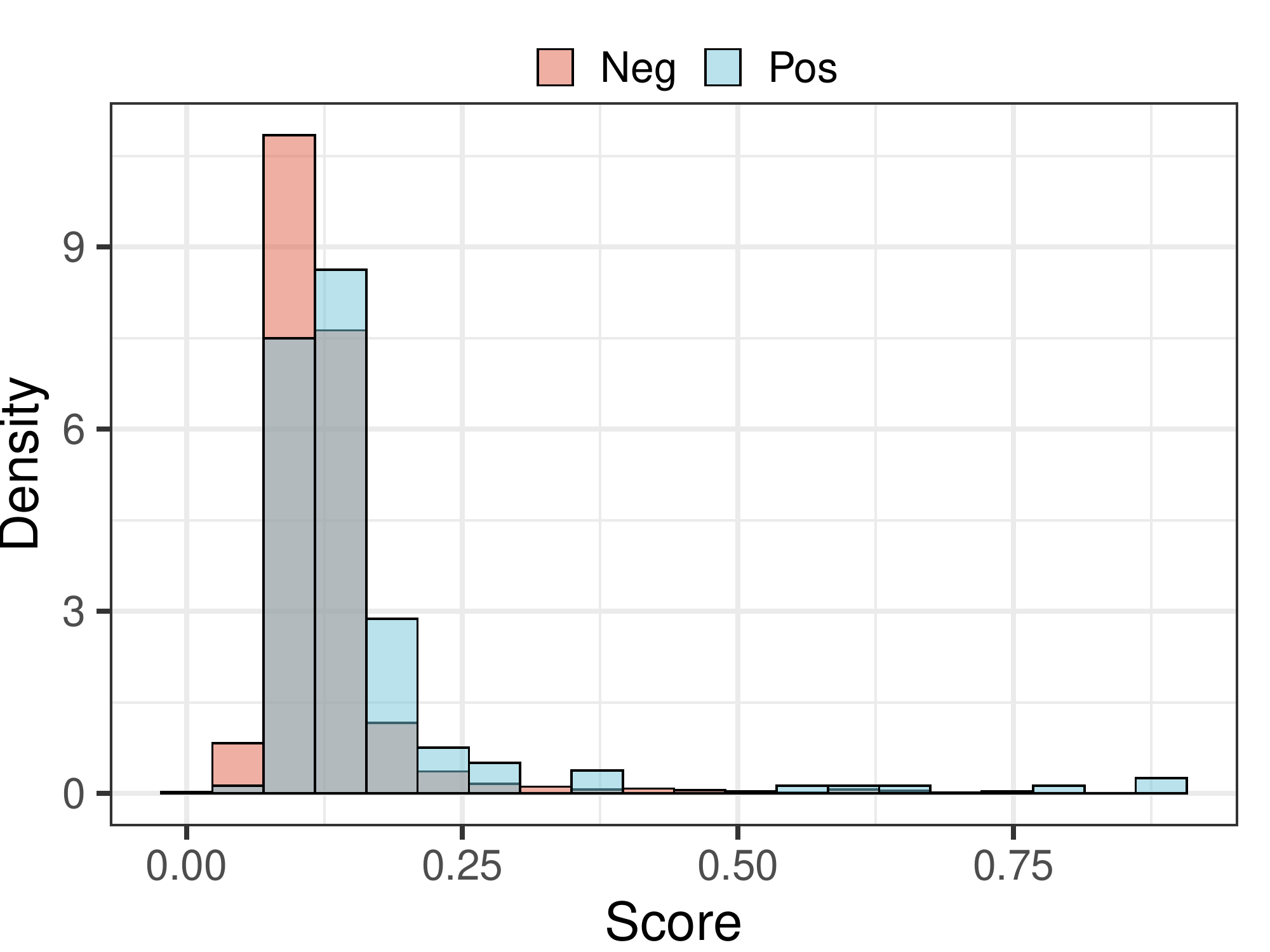}}
        \subfigure[FSGM]{
        \includegraphics[width=0.24\textwidth]{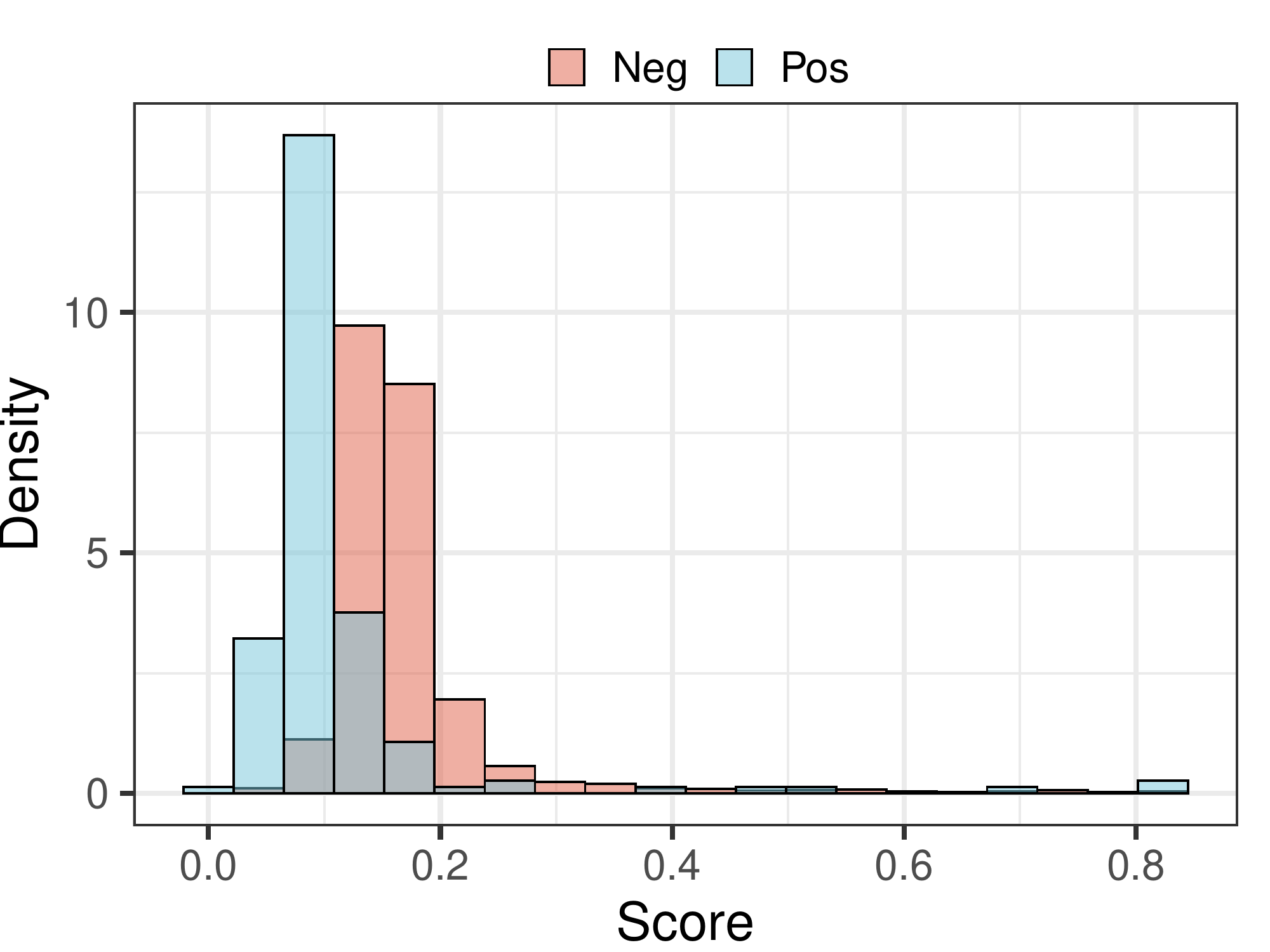}}
        \subfigure[PGD-10]{
        \includegraphics[width=0.24\textwidth]{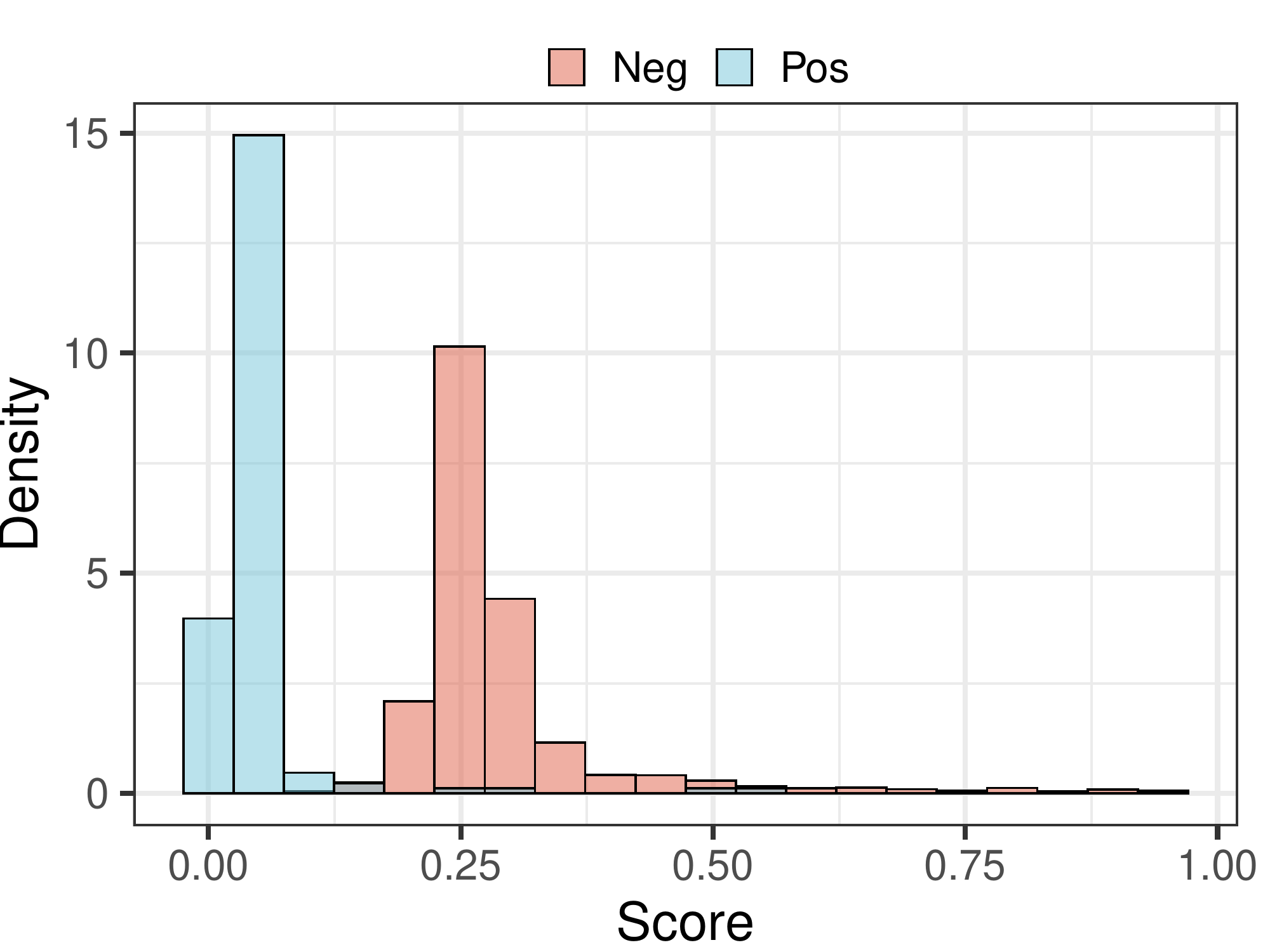}}
        \subfigure[PGD-20]{
        \includegraphics[width=0.24\textwidth]{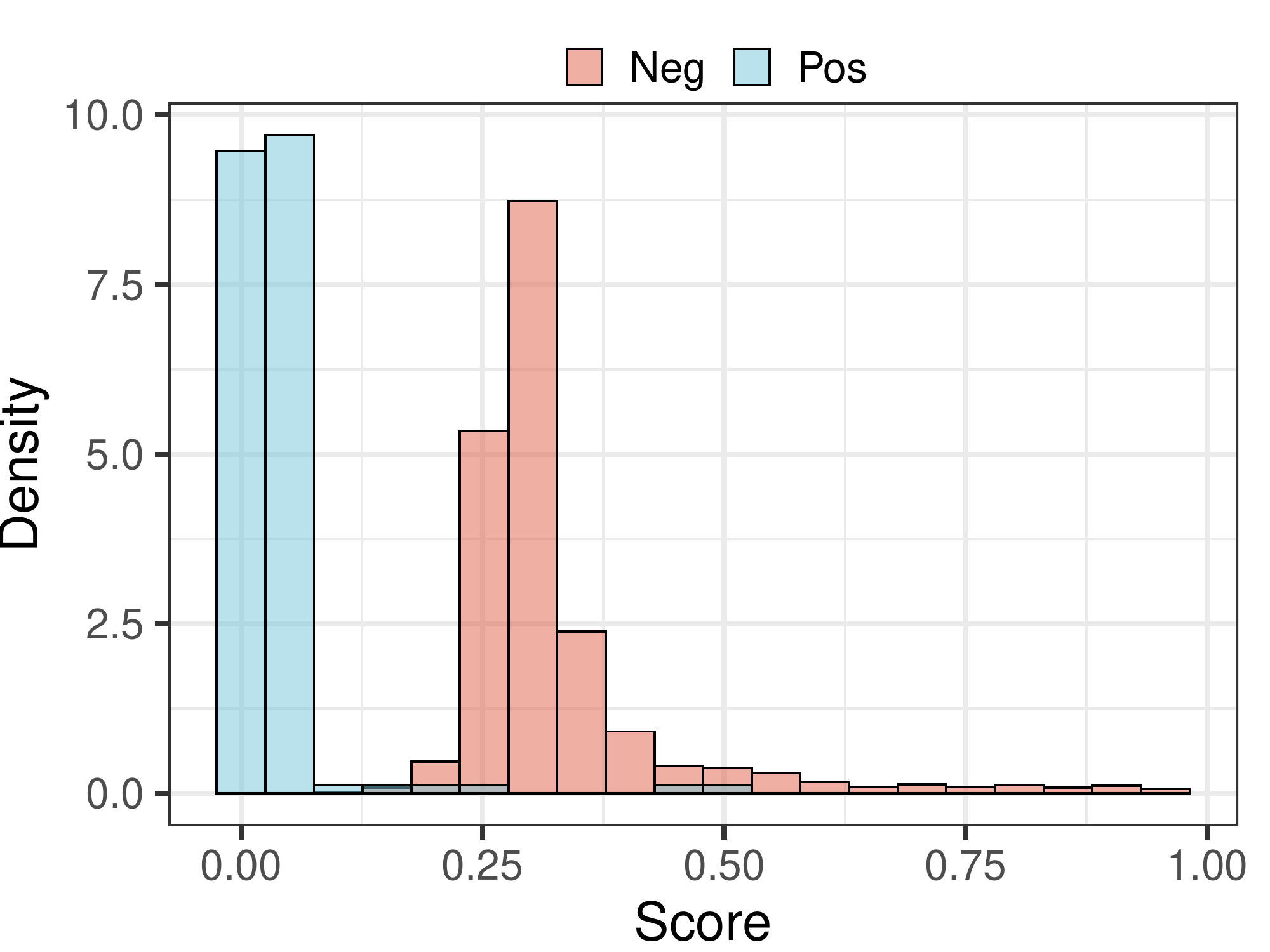}}
        
        \subfigure[Clean]{
        \includegraphics[width=0.24\textwidth]{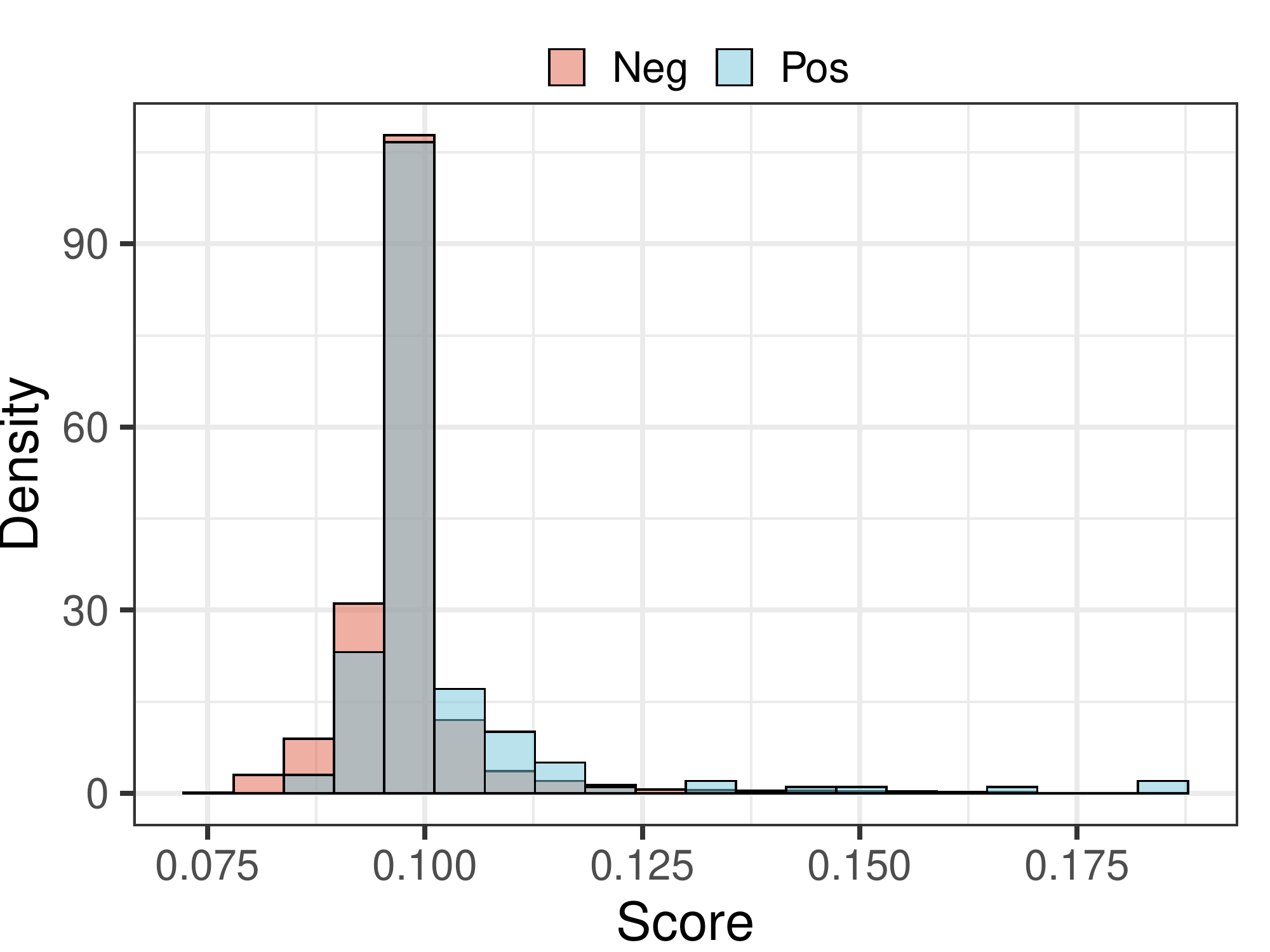}}
        \subfigure[FSGM]{
        \includegraphics[width=0.24\textwidth]{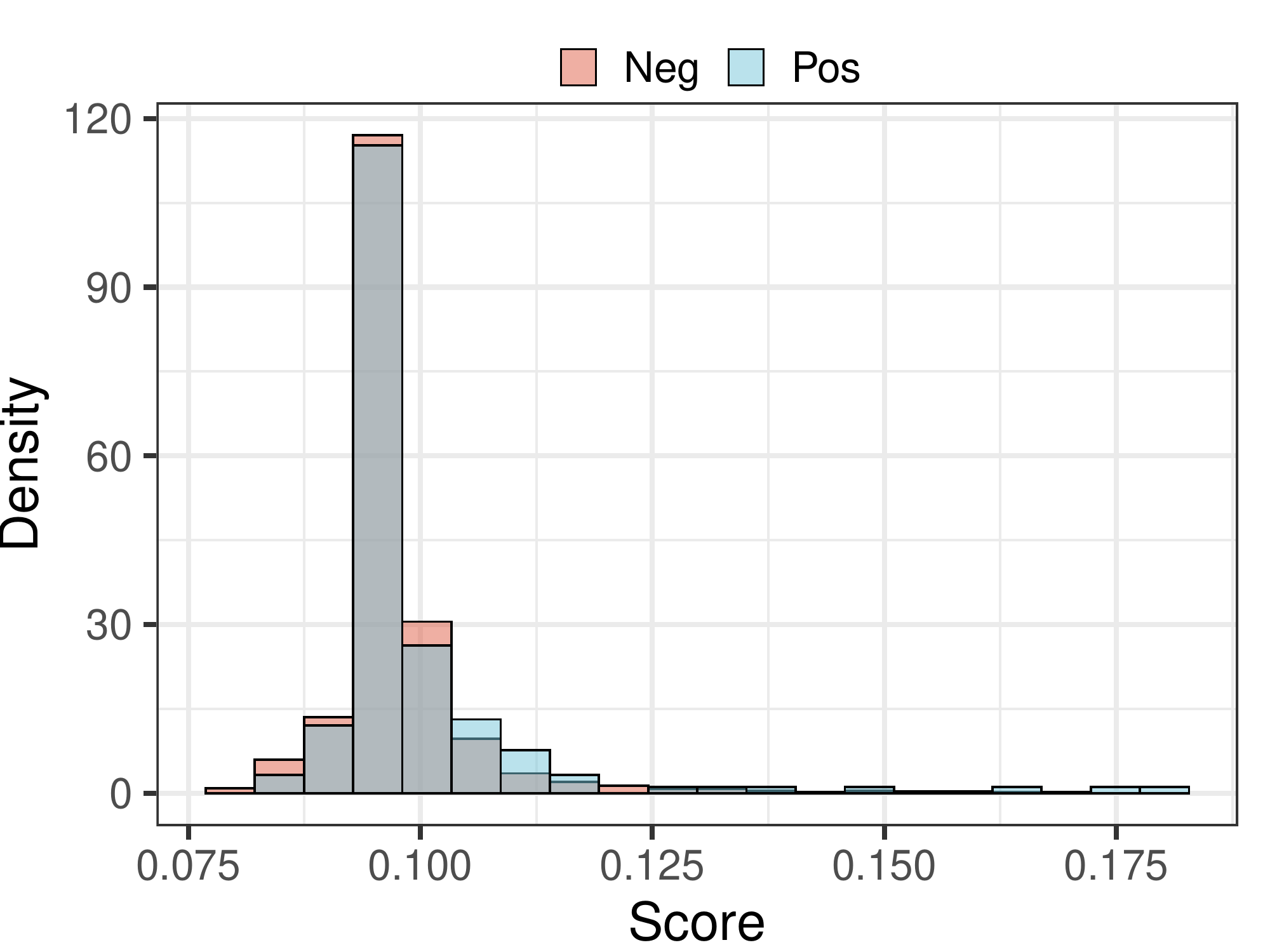}}
        \subfigure[PGD-10]{
        \includegraphics[width=0.24\textwidth]{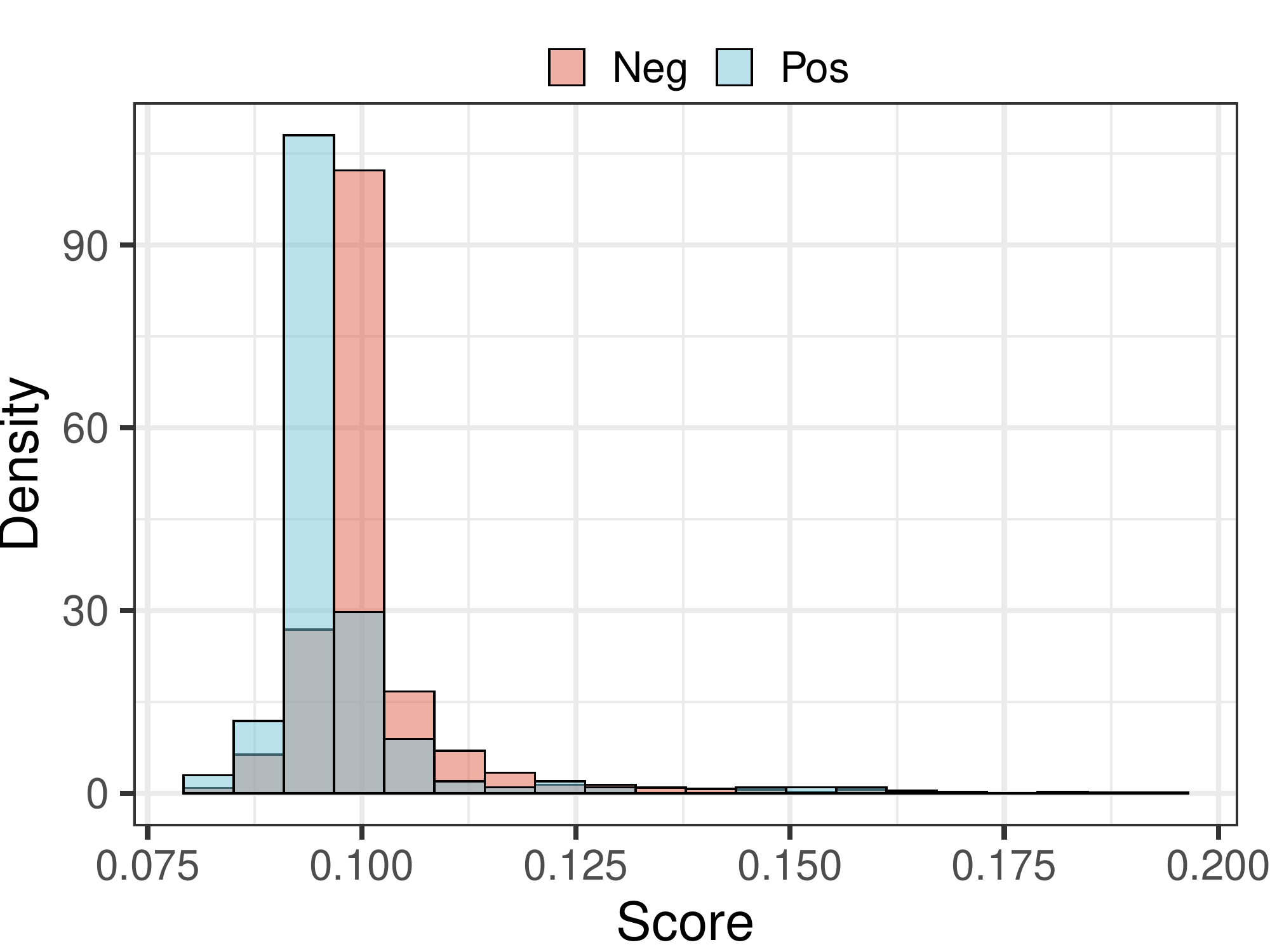}}
        \subfigure[PGD-20]{
        \includegraphics[width=0.24\textwidth]{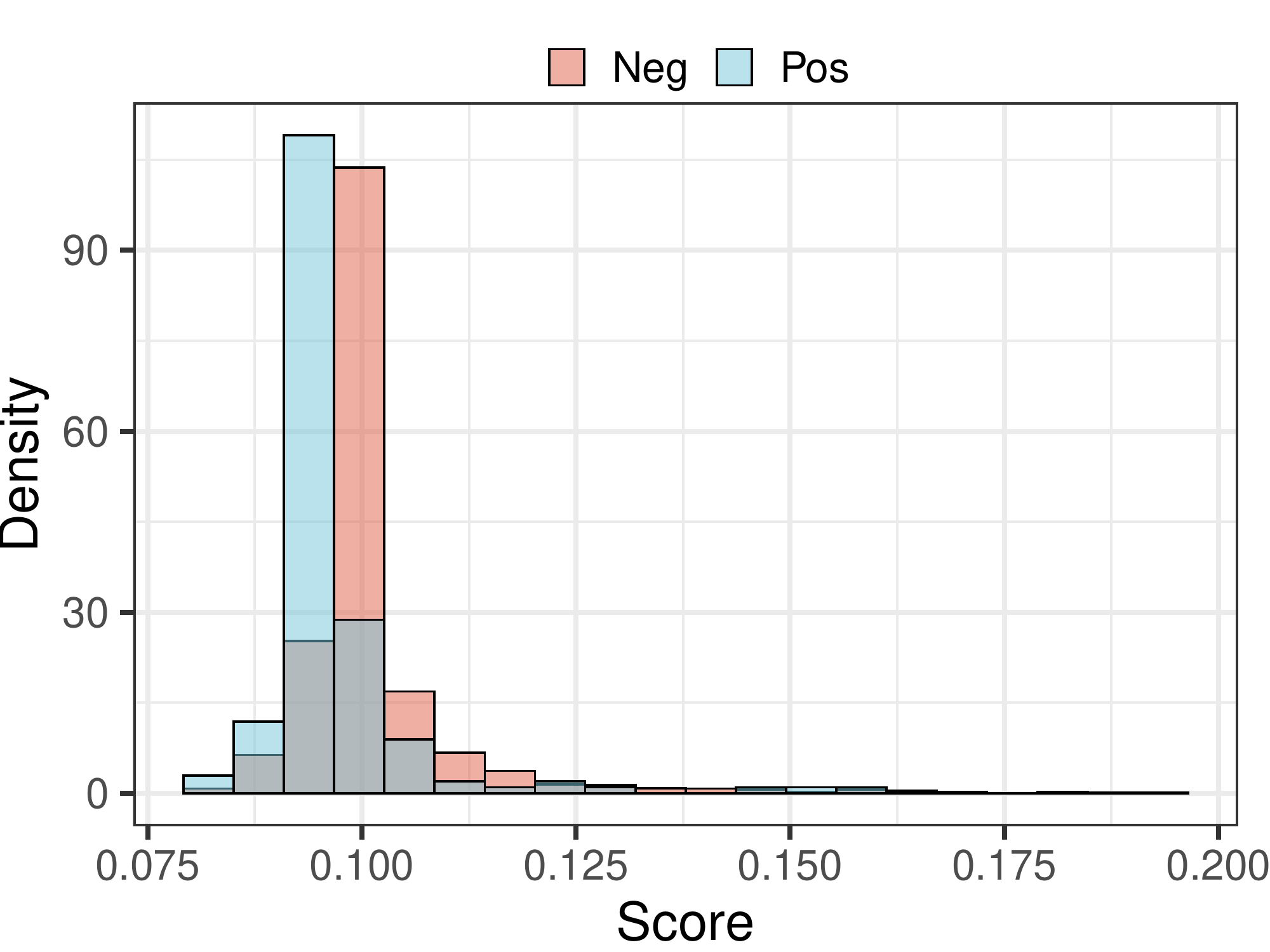}}
        
        \caption{Distribution of positive and negative example scores of CE on CIFAR-100-LT dataset. The first row represents the score distribution against different attacks under Natural Training, and the second row represents the score distribution under Adversarial Training.}
        
        \label{Fig.Distribution.CIFAR100.ce}
    \end{figure*}

    \begin{figure*}[h!]
        \subfigure[Clean]{
        \includegraphics[width=0.24\textwidth]{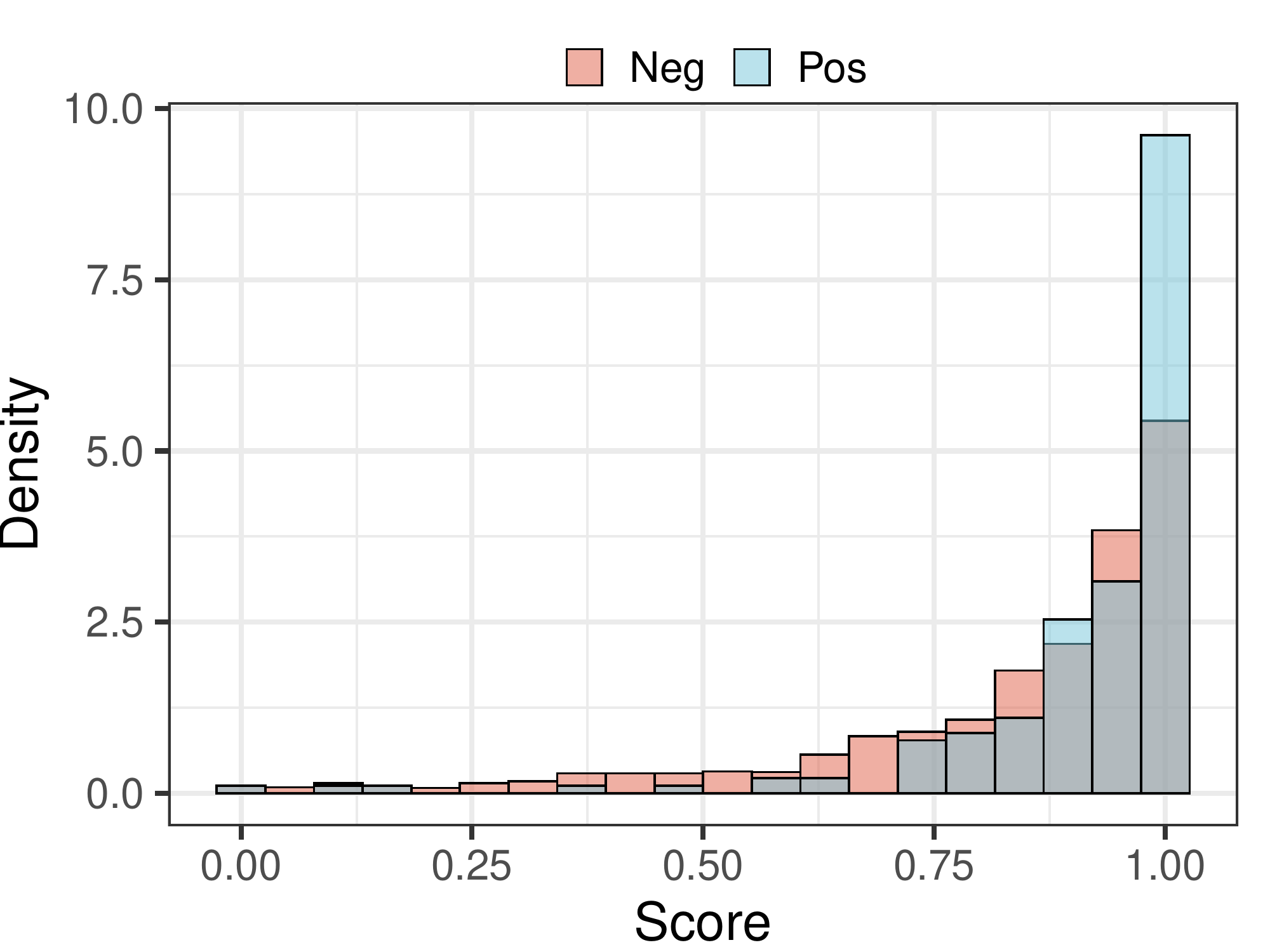}}
        \subfigure[FSGM]{
        \includegraphics[width=0.24\textwidth]{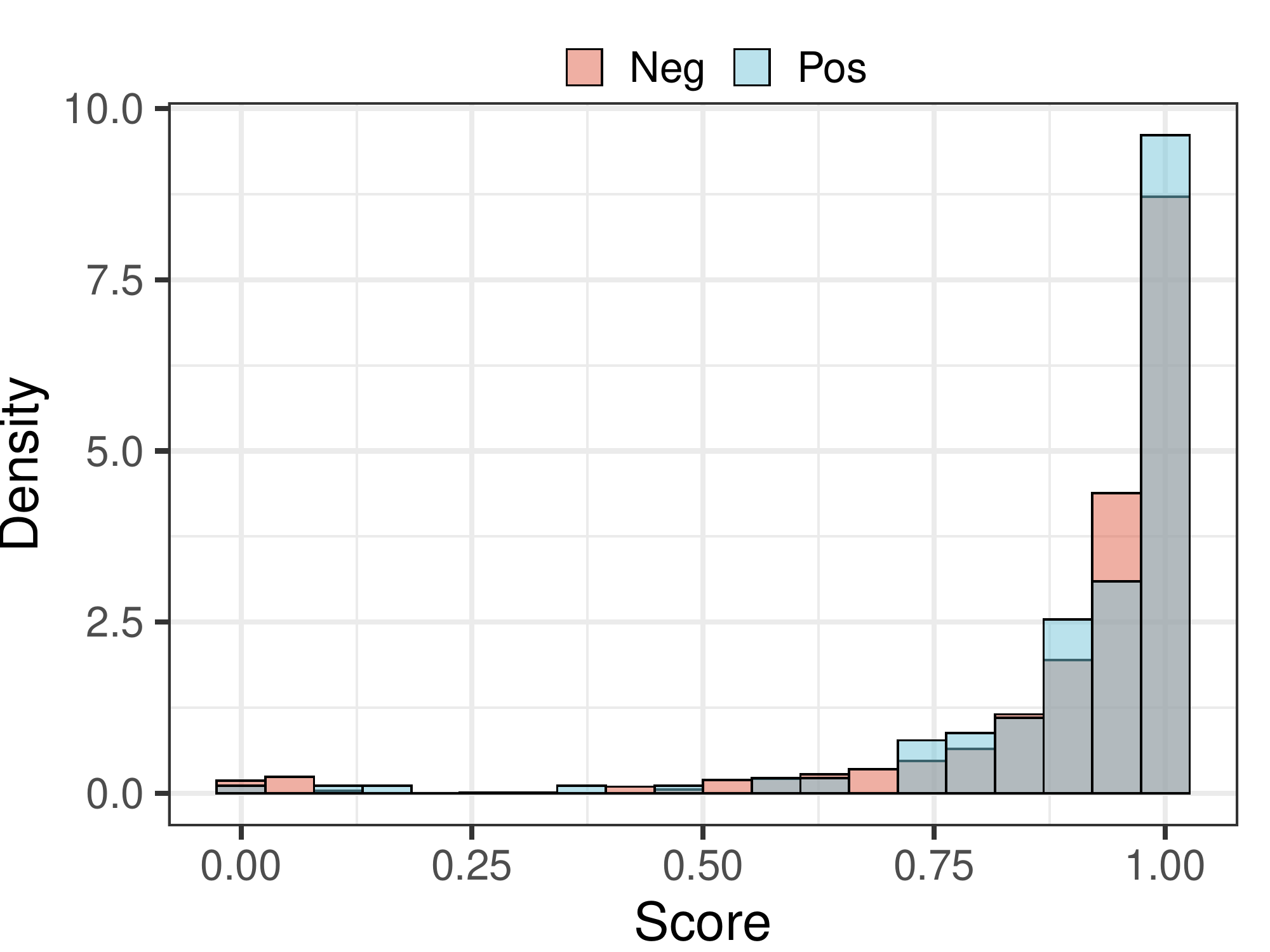}}
        \subfigure[PGD-10]{
        \includegraphics[width=0.24\textwidth]{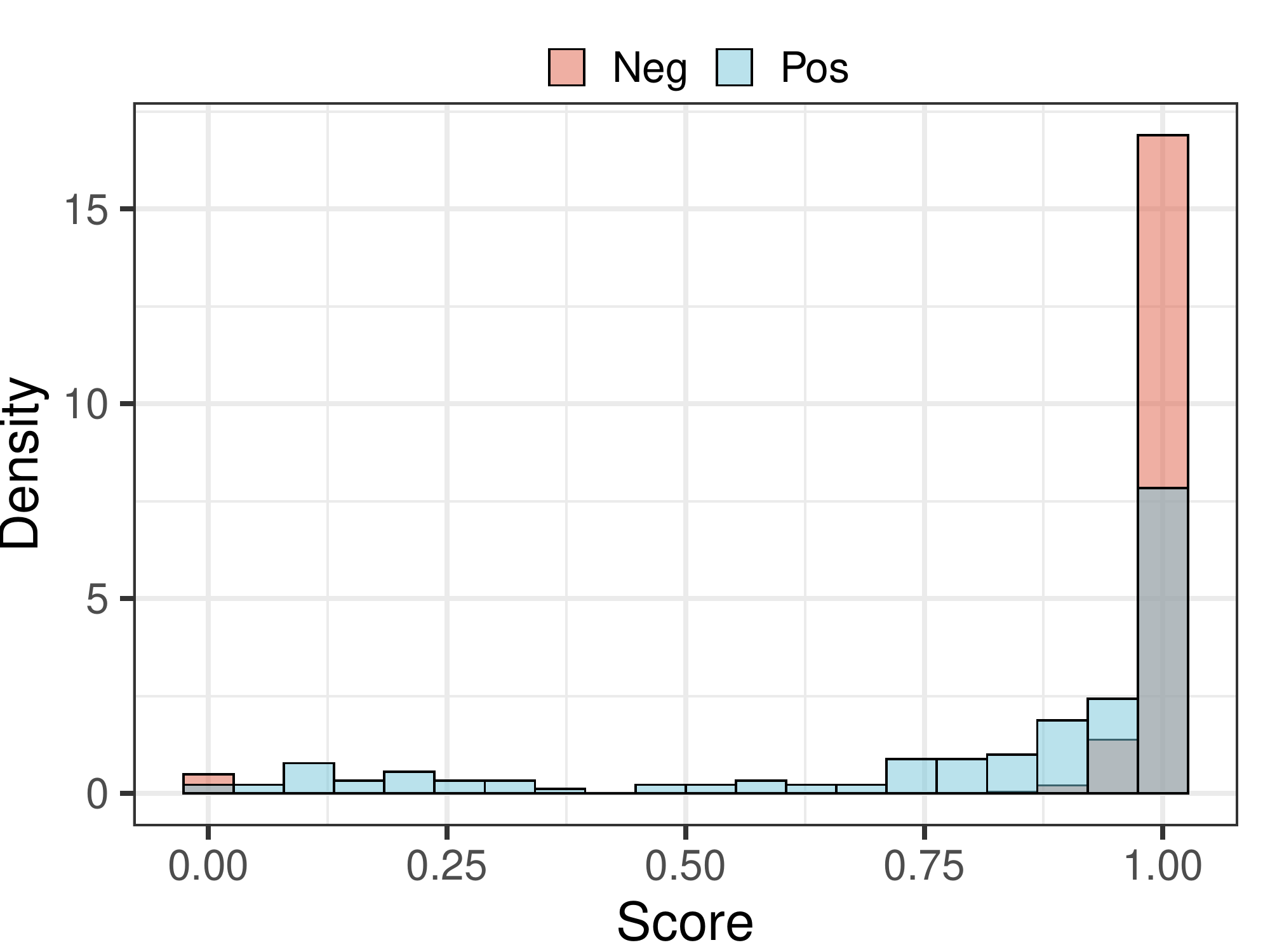}}
        \subfigure[PGD-20]{
        \includegraphics[width=0.24\textwidth]{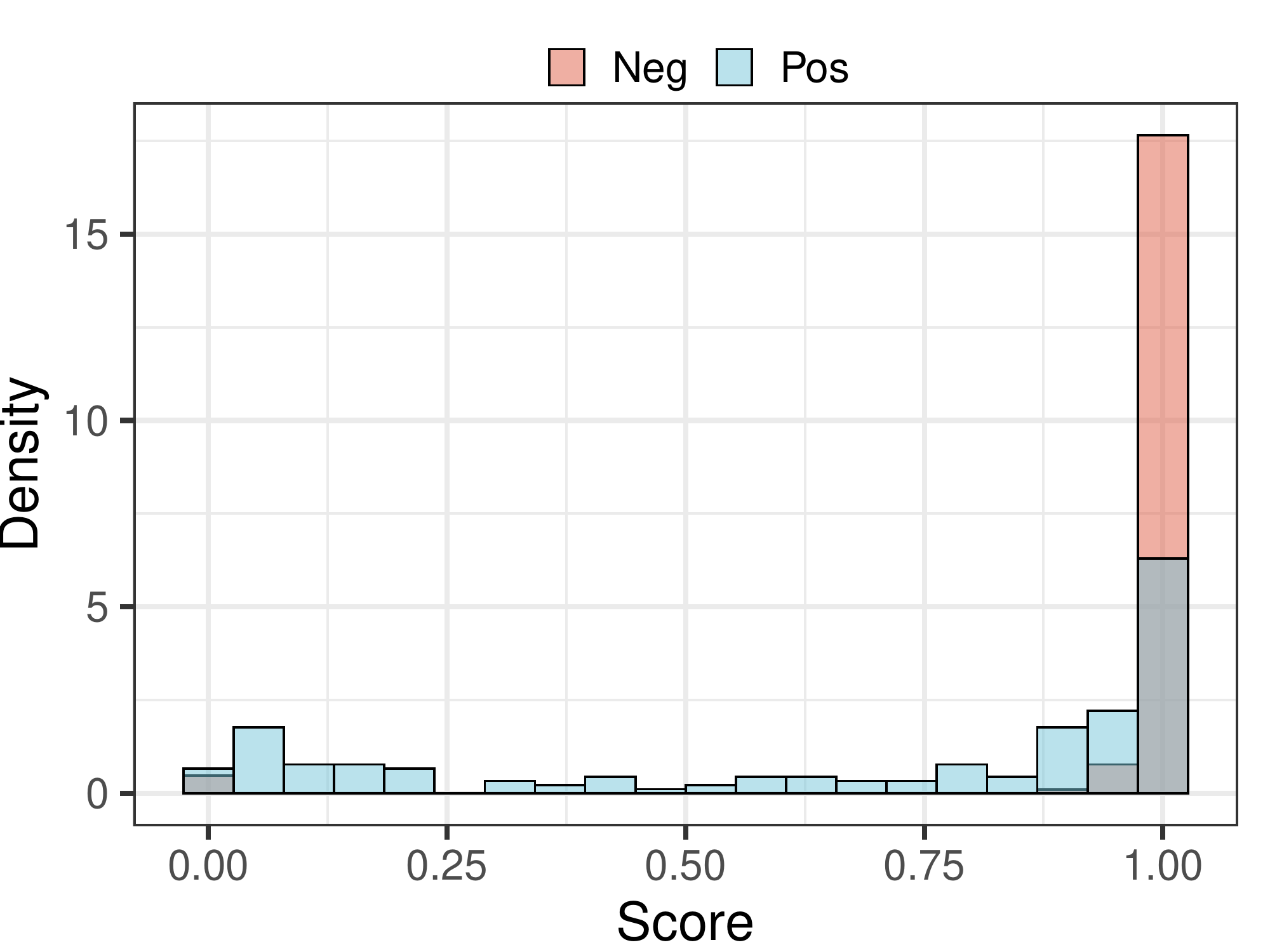}}
        
        \subfigure[Clean]{
        \includegraphics[width=0.24\textwidth]{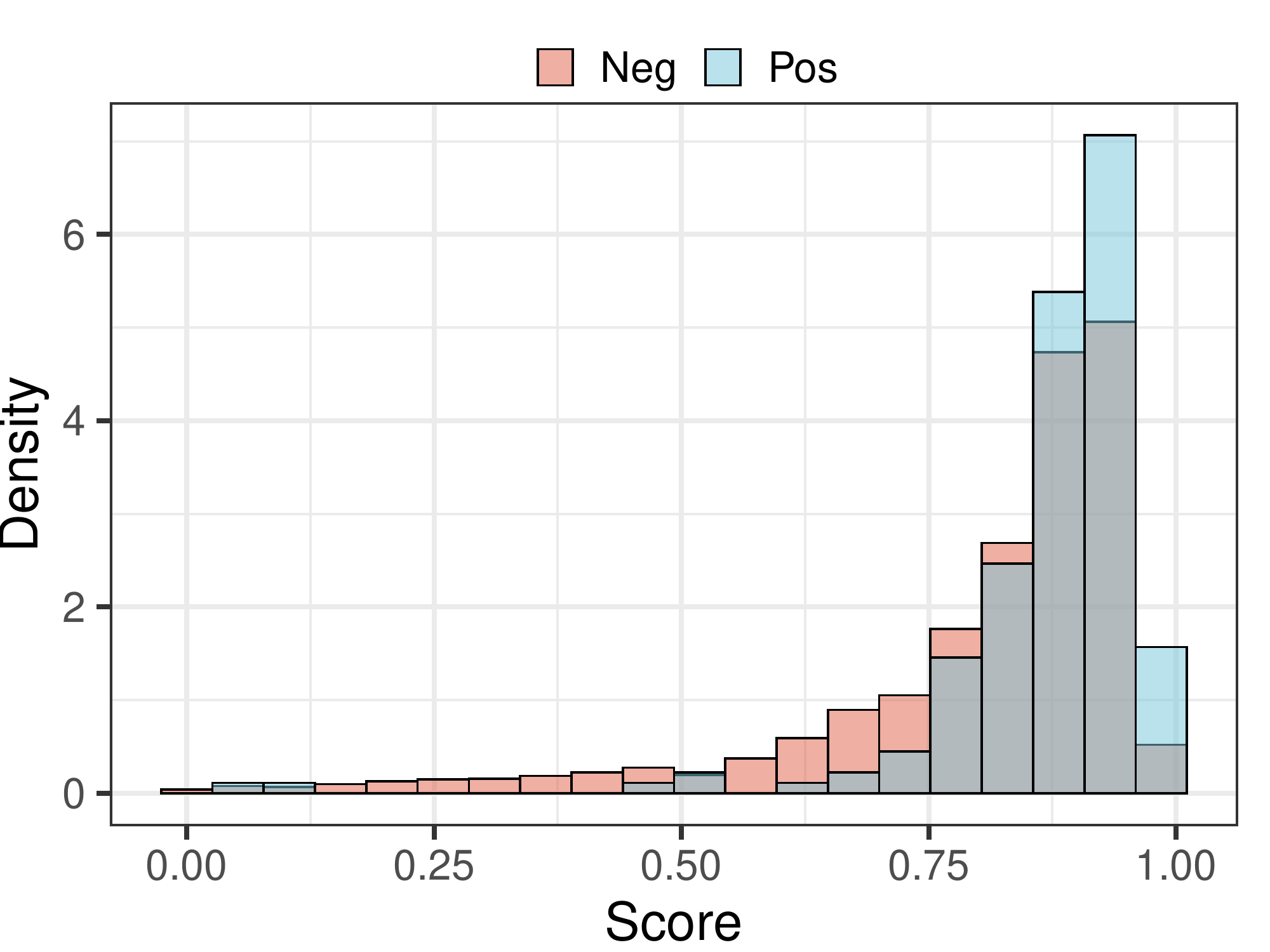}}
        \subfigure[FSGM]{
        \includegraphics[width=0.24\textwidth]{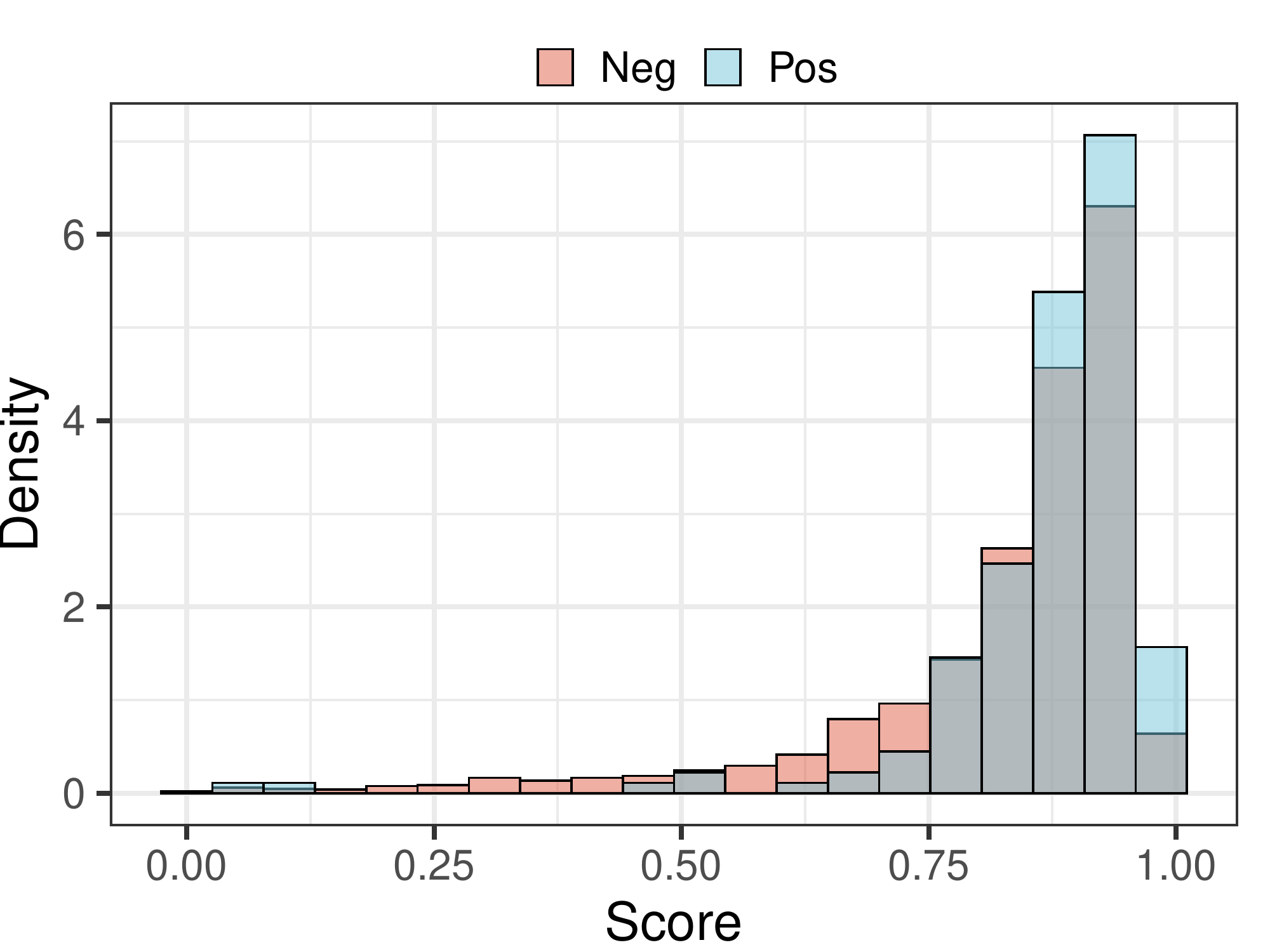}}
        \subfigure[PGD-10]{
        \includegraphics[width=0.24\textwidth]{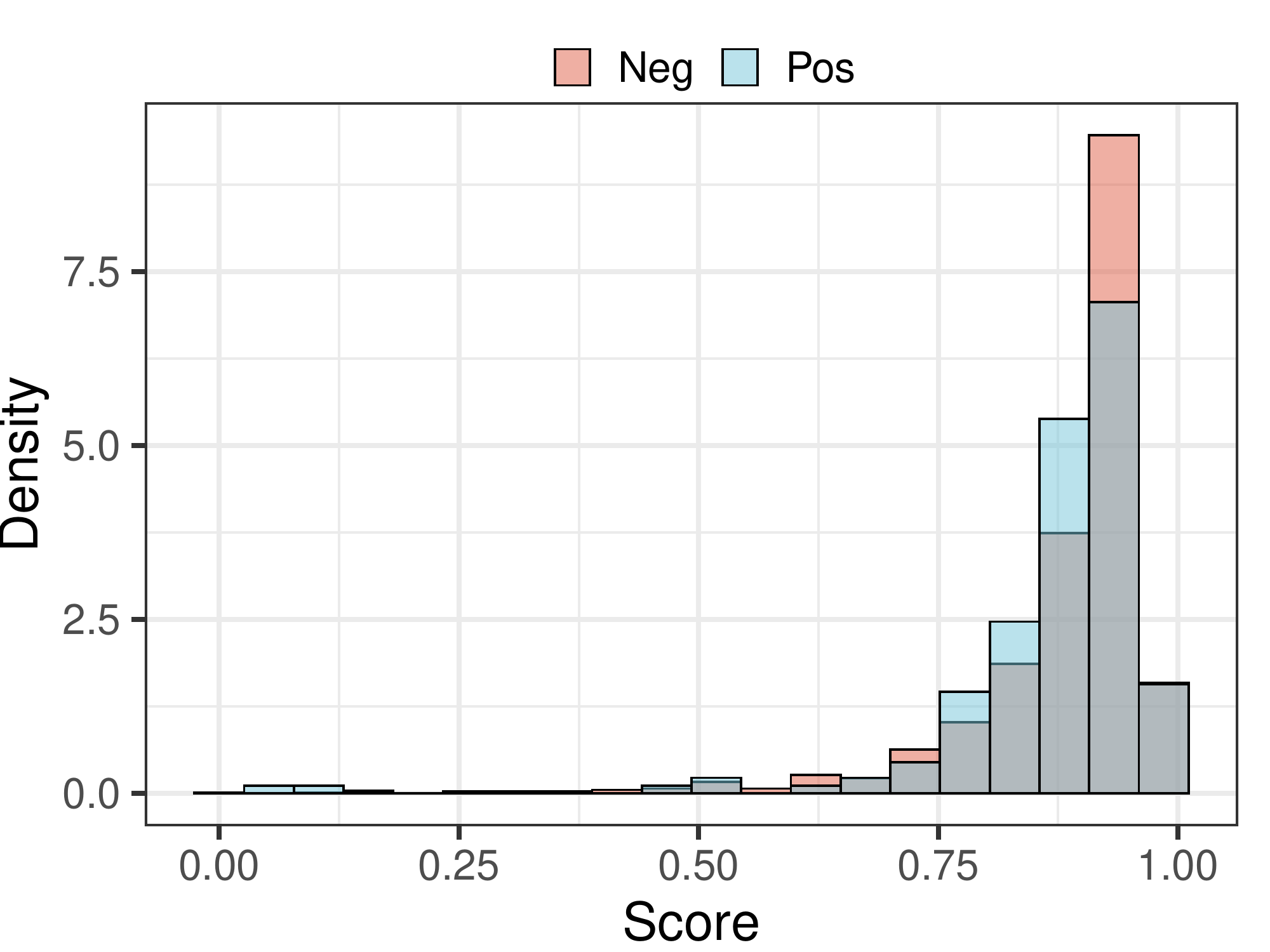}}
        \subfigure[PGD-20]{
        \includegraphics[width=0.24\textwidth]{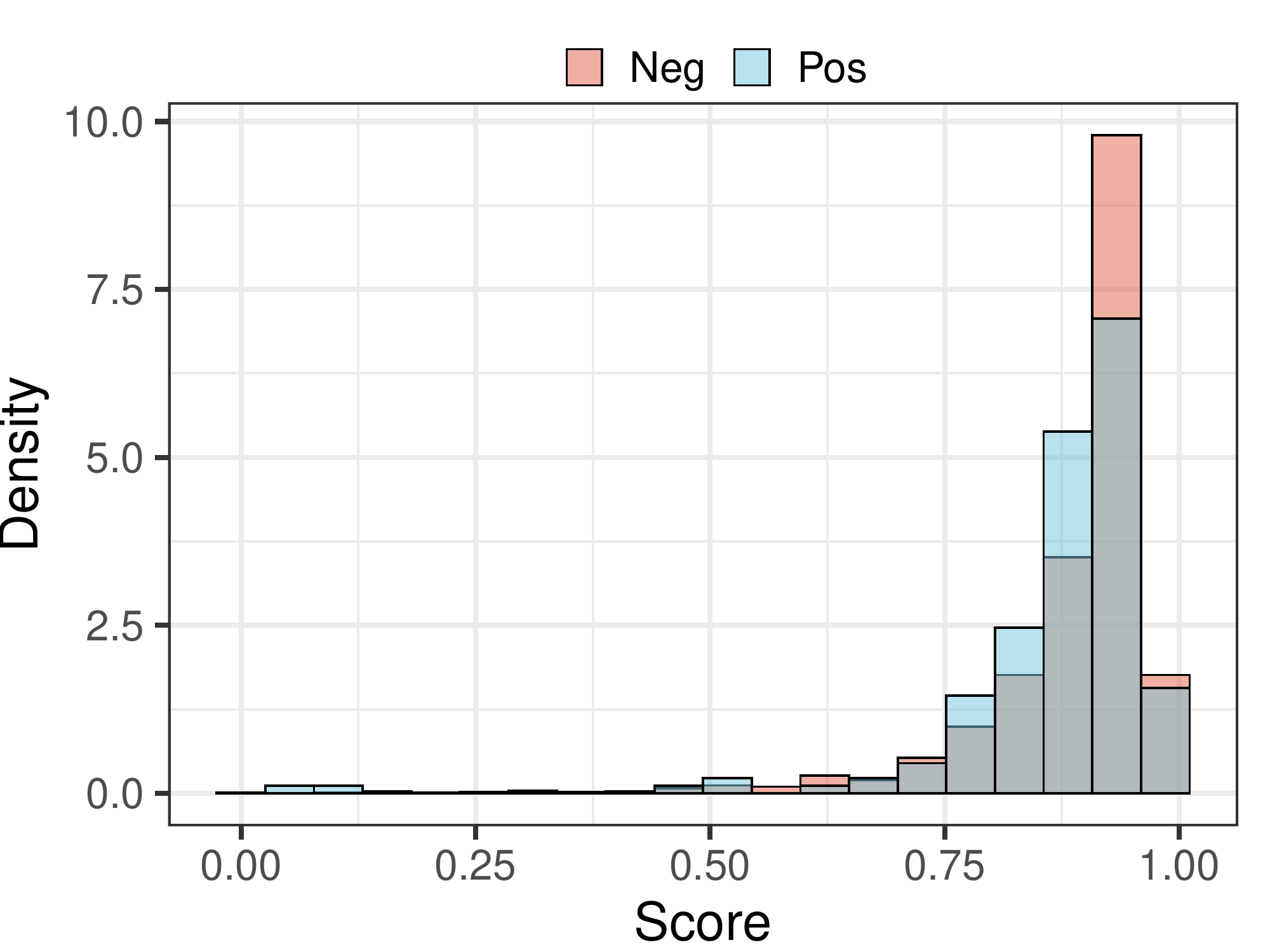}}
        
        \caption{Distribution of positive and negative example scores of AdAUC on CIFAR-100-LT dataset. The first row represents the score distribution against different attacks under Natural Training, and the second row represents the score distribution under Adversarial Training.}
        
        \label{Fig.Distribution.CIFAR100.auc}
    \end{figure*}

\end{document}